\documentclass{article}

\usepackage{microtype}
\usepackage{graphicx}
\usepackage{subfigure}
\usepackage{booktabs} %

\usepackage[colorlinks=true]{hyperref}

\usepackage[noend]{algorithmic}

\usepackage[accepted]{icml2024}

\usepackage{amsmath}
\usepackage{amssymb}
\usepackage{mathtools}
\usepackage{amsthm}
\usepackage{bm}
\usepackage{dsfont}
\usepackage{multicol}

\usepackage[capitalize,noabbrev]{cleveref}

\DeclareMathOperator{\EE}{\mathbb{E}}
\DeclareMathOperator*{\argmax}{\mathop{\mathrm{argmax}}}

\newcommand{\instancespace}{\mathcal{X}}
\newcommand{\labelspace}{\mathcal{Y}}

\newcommand{\datadistribution}{\mathds{P}}
\newcommand{\confusionmatrixspace}{\mathcal{C}}

\newcommand{\vecx}{\bm{x}}
\newcommand{\opty}{y^{\star}}
\newcommand{\optvecy}{\bm{y}^{\star}}
\newcommand{\vecy}{\bm{y}}

\newcommand{\pred}{\widehat{y}}
\newcommand{\vecpred}{\widehat{\vecy}}
\newcommand{\vecp}{\bm{p}}
\newcommand{\avgy}{\overline{\bm{y}}}

\newcommand{\matC}{\bm{C}}
\newcommand{\tensorC}{\matC}

\newcommand{\veceta}{\bm{\eta}}
\newcommand{\empeta}{\widehat{\eta}}
\newcommand{\empveceta}{\widehat{\veceta}}

\newcommand{\vecv}{\bm{v}}

\newcommand{\vech}{\bm{h}}

\newcommand{\utility}{\psi}
\newcommand{\regret}{R}

\newcommand{\algo}[1]{\text{#1}}
\newcommand{\OMMA}{\algo{OMMA}\xspace}
\newcommand{\OMMAeta}{$\algo{OMMA}(\empeta)$}

\newcommand{\truepos}{\text{tp}}
\newcommand{\falsepos}{\text{fp}}
\newcommand{\falseneg}{\text{fn}}
\newcommand{\trueneg}{\text{tn}}

\newcommand{\datasettable}[1]{\textbf{\textsc{#1}}}

\usepackage{csquotes}
\usepackage{pgfplotstable}
\usepackage{color}
\usepackage{colortbl}
\usepackage{comment}
\usepackage{wrapfig}
\usepackage{thmtools}
\usepackage{thm-restate}
\usepackage{float}
\usepackage{nicefrac}       %
\usepackage{microtype}      %
\usepackage[verbose]{placeins}
\usepackage{enumitem}
\usepackage{xspace}
\usepackage{float}
\usepackage{xcolor}

\theoremstyle{plain}
\newtheorem{theorem}{Theorem}[section]

\newtheorem{lemma}[theorem]{Lemma}

\theoremstyle{definition}

\newtheorem{assumption}[theorem]{Assumption}
\theoremstyle{remark}
\newtheorem{remark}[theorem]{Remark}

\usepackage{todonotes}

\icmltitlerunning{A General Online Algorithm for Optimizing Complex Performance Metrics}

\begin{document}

\twocolumn[
\icmltitle{A General Online Algorithm for Optimizing Complex Performance Metrics}

\begin{icmlauthorlist}
\icmlauthor{Wojciech Kotłowski}{put}
\icmlauthor{Marek Wydmuch}{put}
\icmlauthor{Erik Schultheis}{aalto}
\icmlauthor{Rohit Babbar}{aalto,bath}
\icmlauthor{Krzysztof Dembczy{\'n}ski}{put,yahoo}
\end{icmlauthorlist}

\icmlaffiliation{put}{Poznan University of Technology, Poznan, Poland}
\icmlaffiliation{aalto}{Aalto University, Helsinki, Finland}
\icmlaffiliation{bath}{University of Bath, Bath, UK}
\icmlaffiliation{yahoo}{Yahoo Research, New York, United States}

\icmlcorrespondingauthor{Wojciech Kotłowski}{wkotlowski@cs.put.poznan.pl}
\icmlcorrespondingauthor{Marek Wydmuch}{mwydmuch@cs.put.poznan.pl}

\icmlkeywords{Machine Learning, ICML}

\vskip 0.3in
]

\printAffiliationsAndNotice{\icmlEqualContribution} %

\begin{abstract}
We consider sequential maximization of performance metrics that are general functions of a confusion matrix
of a classifier (such as precision, F-measure, or G-mean). 
Such metrics are, in general, non-decomposable over individual instances, making their optimization very challenging.
While they have been extensively studied
under different frameworks in the batch setting, their analysis in the online learning regime is very limited, 
with only a few distinguished exceptions. 
In this paper, we introduce and analyze a general online algorithm that can be used in a straightforward way with a variety of complex performance metrics in binary, multi-class, and multi-label classification problems.
The algorithm's update and prediction rules are appealingly simple and computationally efficient without the need to store any past data.
We show the algorithm attains $\mathcal{O}(\frac{\ln n}{n})$ regret for concave and smooth metrics and verify the efficiency of the proposed algorithm in empirical studies.
\end{abstract}

\section{Introduction}

Many modern applications of machine learning involve optimization of complex performance metrics 
that, unlike misclassification error, do not decompose into expectation over instance-wise quantities.
Examples of such measures include $F$-measure~\citep{lewis1995evaluating},
the area under the ROC curve (AUC)~\citep{DrummondHolte_2005},
geometric~\citep{DrummondHolte_2005,WangYao_2012,Menon_etal2013,cao2019learning} and harmonic mean~\citep{kennedy2009learning},
and Matthews coefficient~\citep{baldi2000assessing}.

Complex performance metrics have been studied in binary%
~\citep{Ye_et_al_2012,Koyejo_et_al_2014,Busa-Fekete_et_al_2015,Dembczynski_et_al_2017, singh2022optimal}, 
multi-class~\citep{Narasimhan_et_al_2015,Narasimhan_et_al_2022}, 
and multi-label classification~\citep{Waegeman_et_al_2014,Koyejo_et_al_2015,Kotlowski_Dembczynski_2017}
under different theoretical frameworks. 
The population utility (PU) framework focuses on building a model
being close to the optimal one on the population level~\citep{Ye_et_al_2012}. 
The expected test utility (ETU) framework concerns the optimization of the expected utility computed
over a given (known) test set~\citep{lewis1995evaluating, Jansche_2007, Ye_et_al_2012}.
These two frameworks are different in the formulation of the objectives 
(more precisely, they differ in the order in which the expectation and the metric is computed),
but they turn out to be asymptotically equivalent~\citep{Dembczynski_et_al_2017}. 
Still, most of the existing works concern a batch setting.

In this paper, we consider an online setting for optimization of complex performance metrics.
In our framework, an algorithm observes a sequence of instances, one at a time, which are drawn \emph{i.i.d.} from some distribution. 
At each iteration, after receiving the input vector, 
the algorithm makes a prediction and observes
the true label (or vector of labels). After observing the entire sequence of data, 
the algorithm is evaluated by means of a performance metric computed from
its empirical confusion matrix over the data.
We evaluate an algorithm in our setting by means of a \emph{regret},
which is the expected (over the data sequence) difference between the algorithm's performance and that of the optimal classifier.
The goal is to design no-regret algorithms, that is, algorithms that guarantee regret converges to zero as the sequence length grows.

Contrary to most of the online learning work, we do not assume that
predictions of the algorithm are obtained by means of some parametric
function of the input vector (e.g., linear or generalized linear models).
Instead, we assume that the algorithm has access to a \emph{conditional
probability estimator} (CPE), which returns an
estimate of the true label conditional distribution at the input vector,
which the algorithm will use to issue its prediction. 
Such a setup would be trivial for decomposable performance
metrics, where the optimal decision is fully determined by the label conditional
probabilities. For non-decomposable performance metrics, however, the optimal
decision on different instances are no longer independent, yet the online
setting requires the classifier to commit to a prediction as soon as it gets the
instance, without any chance to revise decisions that turn out to be suboptimal in retrospect.

In this paper, we introduce and analyze a general online algorithm called Online Metric Maximization
Algorithm (\OMMA{}), which
can be used in a straightforward way with a variety of complex performance metrics 
in binary, multi-class, and multi-label classification problems, 
with or without budgeted predictions (i.e., requiring exactly $k$ predictions per instance).
The algorithm's prediction rule is appealingly simple: at any given
trial, the algorithm issues the prediction that maximizes the expected
value (with respect to the current, not yet seen label, obtained by means of the CPE) of a \emph{linearized} version of the metric.
Thus, \OMMA{} is effectively a cost-sensitive classification
rule with costs determined by the gradient of the utility at the current confusion matrix, and
computing its prediction boils down to thresholding or sorting linear functions of the conditional probabilities.
The algorithm is computationally efficient and 
does not need to store any past data or predictions beyond the entries of the confusion matrix, update of
which can be done online in a straightforward manner.
The algorithm attains 
$\mathcal{O}(\smash{\frac{\ln n}{n}})$ regret for concave and smooth metrics (up to the estimation error of its CPE),
and $\mathcal{O}(\smash{\sqrt{\smash[b]{\frac{\ln n}{n}}}})$ regret without smoothness.

We also verify the efficiency of the proposed algorithm in empirical studies, 
where we evaluate it on a wide range of multi-label benchmarks and performance measures.
Our results show that the proposed algorithms achieve attractive performance in comparison to similar algorithms that can only be applied to a limited set of measures or require more computations and memory.

Our contribution can be summarized as follows:
\begin{itemize}
\item We formulate an online learning framework in which the classifier
itself is not parametric but based on a potentially learnable label probability estimator,
\item We provide a simple algorithm with a constant memory footprint that can optimize
general non-decomposable performance metrics in an online manner across
a wide range of learning tasks, including binary, multi-class, multi-label,
and budgeted-at-$k$ predictions,
\item We prove regret bounds for this algorithm, showing that for concave and smooth
metrics, we converge to the population-optimal at a rate of $\mathcal{O}(\frac{\ln n}{n})$.
\item We show that in data sets with large variance, it can be beneficial for the
decision process to \emph{ignore} the available label feedback, and instead rely entirely on
the biased, but much more stable, CPE instead.
\end{itemize}

\paragraph{Related work.} 
Despite an extensive and long line of research on complex performance metrics in the batch setting,
the analysis of these metrics in the online setting is quite limited, 
with only a few distinguished exceptions. Below we briefly describe the main results obtained so far.

The online problem has been tackled by
optimizing a surrogate loss/reward function convex with respect
to model parameters, which are updated
in an online fashion. \citet{kar2014online} proposed an online gradient descent method for structural surrogates 
\citep{Joachims_2005}.
\citet{conf/icml/NarasimhanK015} considered optimizing metrics being functions of true positive and true negative 
rate, which are then replaced by convex surrogate rewards.
Such approaches may have high memory cost and do not converge to an optimal model for a given metric, 
but to a model optimal for the structural surrogate, and these two do not necessarily coincide
(within the considered parametric class of functions).

The setup closest to ours has been considered in the context of optimizing the $F$-measure,
where it is known that the optimal classifier is obtained by thresholding label conditional probabilities.
In particular, \citet{Busa-Fekete_et_al_2015} have introduced an online algorithm 
to simultaneously train a CPE model and tune the threshold. 
They have shown that the algorithm is asymptotically consistent, but the convergence rate has not been established. 
Interestingly, this algorithm can be seen as a special case of our general algorithm.
In a follow-up paper, \citet{Liu_et_al_2018} have designed an algorithm in which the threshold is updated using stochastic gradient descent of a strongly convex function, which reflects a specific property of the optimal threshold. 
The algorithm has a convergence rate of $\mathcal{O}(\frac{\ln n}{\sqrt{n}})$.

For a wider class of functions, one can directly exploit the connection to cost-sensitive classification.
\citet{Yan_et_al_2017} has considered simultaneously learning multiple classifiers with different cost vectors
together with an online algorithm that learns to select the best one.
An alternative approach is to devise a method that finds the optimal costs in the space of confusion matrices,
for example, by adapting the Frank-Wolfe algorithm previously used in batch learning for smooth and concave metrics%
~\citep{Narasimhan_et_al_2015,Narasimhan_et_al_2022,Schultheis_etal_ICLR2024}.
Such an approach can be then na\"ively applied to the online setting by running it in exponentially growing time intervals 
to amortize the cost of rerunning the algorithm on the entire batch. %
We use such an approach as a baseline in our experimental studies.

Let us notice that our contribution does not coincide with any of the online/stochastic Frank-Wolfe methods known in the literature~\citep{HazanKale2012,Reddi2016StochasticFM}, as these concern minimization of an objective that is either a finite sum or an expectation, and thus linearly decomposes between the trials. In our case, the objective is a global, non-linear function of all past predictions and labels.

Online algorithms and regret minimization have been extensively studied in 
the online convex optimization (OCO) framework \citep{CBLu06:book,Haz16,SS12}. Our setup, however, substantially differs from
OCO. On the one hand, we aim at optimizing a more difficult, non-linear function of the entire sequence of prediction,
and we need stochastic (i.i.d.) assumption on the data sequence, as otherwise, no algorithm can guarantee vanishing regret (Appendix \ref{sec:adversarial}).
On the other hand, our algorithm is equipped with the CPE, which gives an approximate distribution of labels at the moment of prediction.

\section{Problem Setup}

Let $\vecx \in \instancespace$ be the input instance
and $\vecy \in \labelspace \subseteq \{0,1\}^m$ the output label vector, jointly distributed according to
$(\vecx, \vecy) \sim \datadistribution$. We let $\datadistribution_{\!\instancespace}$ denote the marginal distribution $\vecx$, and
$\veceta(\vecx)$ the label conditional probability, $\veceta(\vecx) = \EE_{\vecy | \vecx}[\vecy]$.
In \emph{multi-class classification}, the label vector $\vecy$ is a one-hot encoding of one out of $m$ classes, 
$\labelspace = \{\vecy \in \{0,1\}^m \colon \sum_j y_j = 1\}$, 
while $\veceta(\vecx) = (\eta_1(\vecx),\ldots, \eta_m(\vecx))$ is the label conditional distribution
with $\eta_j(\vecx) = \datadistribution(y_j=1|\vecx)$ denoting the conditional probability of class $j$.
Whereas in \emph{multi-label classification} $\vecy \in \{0,1\}^m$ 
denotes a vector of relevant labels and $\veceta(\vecx) = (\eta_1(\vecx),\ldots, \eta_m(\vecx))$ is vector
of \emph{label marginal conditional probabilities},
$\eta_j(\vecx) = \datadistribution(y_j=1|\vecx)$. Note that in the multi-label case,
$\veceta(\vecx)$ \emph{is not} the conditional distribution $\datadistribution(\vecy|\vecx)$
(which is a \emph{joint} distribution over $\{0,1\}^m$), 
but in the context of this paper, it will always suffice to only operate on its marginals; %
we will thus sometimes use $\vecy \sim \veceta(\vecx)$ to denote a multi-label vector drawn from $\datadistribution(\vecy|\vecx)$
with marginals $\veceta(\vecx)$ if no other property of the distribution matters,
and we will call $\veceta$ the conditional distribution to simplify the presentation.

\begin{figure}[b!]
\noindent\rule{\linewidth}{0.5pt}
\vspace{-16pt}
\begin{algorithmic}[0]
\FOR{$t = 1, \ldots, n$} 
   \STATE Observe input instance $\vecx_t$ drawn from $\datadistribution_{\!\instancespace}$
   \STATE Receive conditional probability estimate $\empveceta_t(\vecx_t)$ 
   \STATE Predict label $\vecpred_t$ based on $\empveceta_t(\vecx_t)$ 
   \STATE Receive true label $\vecy_t$ drawn from $\datadistribution(\cdot|\vecx_t)$
\ENDFOR
\STATE Evaluate based on $\psi(\matC(\vecy^n, \vecpred^n))$
\end{algorithmic}
\vspace{-8pt}
\noindent\rule{\linewidth}{0.5pt}

\caption{The online protocol.}
\label{fig:online-protocol}
\end{figure}

We consider an online (stochastic) setting in which the algorithm observes a sequence of instances $(\vecx_1,\vecy_1), (\vecx_2,\vecy_2),\ldots, (\vecx_n,\vecy_n)$, drawn \emph{i.i.d.}\@ from $\datadistribution$. At each iteration $t=1,2,\ldots,n$, after receiving input $\vecx_t$ drawn from
$\datadistribution_{\!\instancespace}$,
the algorithm makes a prediction $\vecpred_t$, and next observes
the true output $\vecy_t$ drawn from $\datadistribution(\vecy | \vecx_t)$.
We assume that the algorithm has access to a \emph{conditional probability estimator} (CPE) $\empveceta_t$, which at trial $t$
returns an estimate $\empveceta_t(\vecx_t)$ of the true conditional
distribution $\veceta(\vecx_t)$, which the algorithm
will use for issuing its prediction $\vecpred_t$.
We put the time subscript in $\empveceta_t$ as we allow the CPE to change over time
(e.g., it can be produced by an external online learner run on the same sequence of instances);
we assume, however, that $\empveceta_t$ can only depend on the observed data $\vecx_{t}$,
$(\vecx_{1},\vecy_{1}),\ldots,(\vecx_{t-1},\vecy_{t-1})$ (i.e., does not depend on $\vecy_t$ and all the future instance).
The CPE could be, for instance, a neural network or boosted decision trees trained (possibly online)
by minimizing some proper scoring loss. In this work, we do not focus on the way the CPE is learned, and
all our bounds on the performance of the algorithm will depend on the
\emph{estimation error} of $\empveceta_t$ with respect to true conditional distribution $\veceta$. 
We outline the online protocol in~\cref{fig:online-protocol}.\looseness=-1

Given a sequence of labels\footnote{\label{foot:notation}Generally, by $\vecv^t$ we denote a sequence $(\vecv_1,\ldots,\vecv_t)$; $f(\vecv^t)$ denotes $(f(\vecv_1),\ldots,f(\vecv_t))$.} $\vecy^n = (\vecy_1,\ldots,\vecy_n)$ and a sequence of
algorithm's predictions $\vecpred^n = (\vecpred_1,\ldots,\vecpred_n)$,
we let $\matC(\vecy^n, \vecpred^n)$ denote
the (empirical) \emph{confusion matrix}. In multi-class classification, it is an $m \times m$ matrix $\matC$ with entries defined as $C_{j\ell} = \frac{1}{n} \sum_{t=1}^n y_{tj} \widehat{y}_{t\ell}$, where $y_{tj}$ is the $j$-th entry of vector $\vecy_t$. Thus,
$C_{j\ell}$ contains the fraction of times an instance from class $j$ was predicted as being in class $\ell$. Note that for binary classification ($m=2$), we get the usual true-positive ($C_{11}$), false-negative ($C_{10}$), false-positive ($C_{01}$), true-negative ($C_{00}$) entries, where we index the entries
from 0 to stick to a more convenient notation.
In turn, in multi-label classification, the confusion matrix is an $m \times 2 \times 2$ tensor $\tensorC$, where
$C_{juv} = \frac{1}{n} \sum_{t=1}^n \bm{1}\{y_{tj} = u\} \bm{1}\{\widehat{y}_{tj} = v\}$,\footnote{$\bm{1}(S)$ is the indicator function equal one when $S$ holds.} where $j \in \{1,\ldots,m\}$ and $u,v \in \{0,1\}$. In other words, $\matC$ is a sequence of $m$ binary confusion matrices, separately for each label. In a given problem, we let $\confusionmatrixspace$ denote the set of all achievable confusion matrices, that is
the set of confusion matrices that can be formed from allowed label and prediction vectors $\vecy^t$, $\vecpred^t$ for any $t$.
Note that, independent of the considered problem, the confusion matrix can always be written as
$\matC(\vecy^t, \vecpred^t) = \frac{1}{t} \sum_{i=1}^t \matC(\vecy_i, \vecpred_i)$ (with $\matC(\vecy_i, \vecpred_i)$ being confusion matrices computed out of a single
label $\vecy_i$ and prediction $\vecpred_i$), which gives a simple online update:
\begin{equation}
\matC(\vecy^t, \vecpred^t)
= \frac{t-1}{t} \matC(\vecy^{t-1}, \vecpred^{t-1}) + \frac{1}{t} \matC(\vecy_t, \vecpred_t) \,.
\label{eq:confusion_matrix}
\end{equation}

\begin{table*}[t]
\caption{Examples of binary and multiclass confusion matrix measures. Binary measures are expressed in terms of true-positives ($\truepos = C_{11}$),
false-negatives ($\falseneg = C_{10}$), false-positives ($\falsepos = C_{01}$), and true-negatives ($\trueneg = C_{00}$).
}
\label{tab:binary-metrics}
\begin{center}
\resizebox{\linewidth}{!}{
\begin{tabular}{lcc|lcc}
\toprule
Metric & $\utility(\matC^{2 \times 2})$ & $\utility(\matC^{m \times m})$ & Metric & $\utility(\matC^{2 \times 2})$ & $\utility(\matC^{m \times m})$\\[5pt]
\midrule
Accuracy & \small{$\truepos + \trueneg$} & \small{$\sum_{i=1}^{m} C_{ii}$} & G-mean & $\sqrt{\frac{\truepos \cdot \trueneg} {(\truepos+\falseneg)(\trueneg+\falsepos)}}$ & $\left(\prod_{j=1}^m \frac{C_{jj}}{\sum_{i=1}^m C_{ji}} \right)^{1/m}$ \\[5pt]
Balanced Acc. & $\frac{\truepos}{2(\truepos + \falseneg)} + \frac{\trueneg}{2(\trueneg + \falsepos)}$ & $\sum_{i = 1}^m \frac{C_{ii}}{m \sum_{j=1}^m C_{ij}}$ & H-mean & $2 \left(\frac{\truepos + \falseneg}{\truepos} + \frac{\trueneg + \falsepos}{\trueneg}\right)^{\!-1}$ & $m \left(\sum_{j=1}^m \frac{\sum_{i=1}^m C_{ji}}{C_{jj}} \right)^{-1}$ \\[5pt]
Recall & $\frac{\truepos}{\truepos + \falseneg}$ & \small{micro- or macro-avg.} & Q-mean & $1 \!-\! \sqrt{\frac{1}{2}\left(\left(\frac{\falseneg}{\truepos + \falseneg}\right)^2 \!+\! \left(\frac{\falsepos}{\trueneg + \falsepos}\right)^2\right)}$ & $1 \!-\! \sqrt{\frac{1}{m}\sum_{j = 1}^m \left(1 \!-\! \frac{C_{jj}}{\sum_{i=1}^m C_{ji}}\right)^2} $ \\[5pt]
Precision & $\frac{\truepos}{\truepos + \falsepos}$ & \small{micro- or macro-avg.} & Jaccard & $\frac{\truepos}{\truepos + \falsepos + \falseneg}$ & \small{micro- or macro-avg.} \\[5pt]
F$_{\beta}$-measure & $\frac{(1+\beta^2) \truepos}{(1+\beta^2)\truepos+\beta^2\falseneg + \falsepos}$ & \small{micro- or macro-avg.} & Matthews coeff. & $\frac{\truepos \cdot \trueneg - \falsepos \cdot \falseneg}{\sqrt{(\truepos+\falsepos)(\truepos+\falseneg)(\trueneg+\falsepos)(\trueneg+\falseneg)}}$
&  \small{micro- or macro-avg.} \\[5pt]
\bottomrule
\end{tabular}
}
\end{center}
\end{table*}

In this work, we focus on online maximization of performance metrics
that do not decompose into a sum over instances but are general functions of the confusion matrix of the algorithm.
Specifically, after observing the entire sequence of $n$ instances, the algorithm
is evaluated by means of a \emph{utility metric} $\utility = \utility(\matC(\vecy^n, \vecpred^n))$. For binary classification, examples of such measures include the $F$-measure, geometric and harmonic mean, area under the ROC curve, recall, precision, etc. We present their definitions in \cref{tab:binary-metrics}.
Most of these metrics have multi-class (some also presented in~\cref{tab:binary-metrics}) as well as multi-label extensions in the micro- and macro-averaged variants. 
By rewriting the multi-class $\matC^{m \times m}$ confusion matrix into a multi-label confusion tensor $\tensorC^{m \times 2 \times 2}$,\looseness=-1
\begin{equation*}
    \tensorC^{m \times 2 \times 2}(\matC^{m \times m})\!:=\!\left[ \begin{pmatrix}
    C_{jj} & \!\sum\limits_{i \neq j} C_{i j} \\
    \sum\limits_{i \neq j} C_{ji} & \!\sum\limits_{i \neq j, \ell \neq j} C_{i\ell}
    \end{pmatrix} \right]_{\mathrlap{j=1}}^{\mathrlap{m}} \quad\,\,,
\end{equation*}
we can define the micro- and macro-averaged metrics for both multi-class and multi-label classification as:
\begin{equation*}
\begin{aligned}
    \text{Micro-}\utility(\tensorC^{m \times 2 \times 2}) &:= \utility\bigg(\frac{1}{m} \sum_{j=1}^{m} \matC_j\bigg) \,, \\
    \text{Macro-}\utility(\tensorC^{m \times 2 \times 2}) &:= \frac{1}{m} \sum_{j=1}^{m} \utility\left(\matC_j\right) \,.
\end{aligned}
\end{equation*}

All the metrics presented above can be considered under the standard as well as budgeted at $k$ variant~\citep{Schultheis_etal_NeurIPS2023, Schultheis_etal_ICLR2024}, where the classifier is required to predict exactly $k$ classes per instance ($\vecpred \in \widehat \labelspace_k = \{\vecpred \in \{0, 1\}^m : \sum_j \pred_j = k\}$).
The main difficulty in maximizing $\utility$ comes from the 
fact that the algorithm is only evaluated at the end of the sequence by a (possibly complex) function of all the labels and predictions,
while each prediction $\vecpred_t$ must be made immediately upon observing $\vecx_t$, and cannot be changed in the future.

\paragraph{Regret.} Since without imposing assumptions on the data distribution $\datadistribution$ one cannot meaningfully bound the $\psi$-accuracy of the algorithm in the absolute terms, the goal is to compare the algorithm's performance \emph{relative} to that of the optimal predictor. Note, however, that defining the optimal predictions $\optvecy{}^n$ simply as those maximizing the value of the utility, 
$\optvecy{}^n = \argmax_{\vecpred^n} \psi(\matC(\vecy^n,\vecpred^n))$, leads (for essentially all reasonable utilities)
to a trivial solution $\optvecy{}^n = \vecy^n$, which gives the maximum possible value $\psi$, and is thus not achievable by
any algorithm, no matter how large $n$ is. Thus, we proceed differently, defining a \emph{classifier} $\vech \colon \instancespace \to \labelspace$
to be a function from the inputs to the outputs, and the optimal predictions $\optvecy{}^n = \vech^{\star}(\vecx^n)$
to be those generated by a classifier $\vech^{\star}$ which maximizes
\emph{the expected (with respect to the data sequence) value of the utility}:\footnote{Recall that
$\vech(\vecx^n)$ stands for $(\vech(\vecx_1),\ldots,\vech(\vecx_n))$.}
\begin{equation}
\vech^{\star} \coloneqq \argmax_{\vech} \EE_{(\vecx, \vecy)^n} \left[ \psi(\matC(\vecy^n,\vech(\vecx^n))) \right],
\label{eq:optimal_predictions}
\end{equation}
where the maximization is with respect to all (measurable) classifiers. We define the (expected) 
\emph{regret} of the algorithm as the difference between its expected performance in terms of $\psi$,
and the performance of the optimal classifier $\vech^{\star}$:
\[
\regret_n \coloneqq \EE\left[\psi(\matC(\vecy^n,\vech^{\star}(\vecx^n)))\right]  - \EE\left[\psi(\matC(\vecy^n,\vecpred^n))\right].
\]
Our goal is to design no-regret algorithms, that is, algorithms that guarantee $\regret_n \to 0$
as $n \to \infty$. In the next section, we propose a computationally- and memory-efficient algorithm,
which guarantees (under certain assumptions on the utility) $\regret_n = \mathcal{O}\left(\frac{\ln n}{n}\right)$
up to the estimation error of CPE.

Note that when
$n \to \infty$, $\vech^{\star}$ coincides with the $\psi$-optimal population-level classifier. 
Approximating
$\vech^{\star}$ in the batch setting is a task which was tackled
in a series of past works spanning binary classification
\citep{Ye_et_al_2012,Menon_etal2013,Koyejo_et_al_2014,Narasimhan_et_al_2014}, 
multi-class classification 
\citep{Narasimhan_et_al_2015,Narasimhan_et_al_2022},
and multi-label classification
\citep{Waegeman_et_al_2014,Koyejo_et_al_2015,Kotlowski_Dembczynski_2017,Schultheis_etal_ICLR2024},
and boils down to, under certain mild assumptions on the data distribution,
finding the optimal threshold on the conditional
probability $\veceta(\vecx)$ (for binary classification), or finding
the optimal ensemble of cost-sensitive classifiers with an iterative convex optimization algorithm,
such as the Frank-Wolfe method or gradient descent
(for multi-class and multi-label classification).\footnote{In the past papers,
the optimal classifier is allowed to be \emph{randomized}, that is to take values in the convex hull of
$\labelspace$. All of our analysis effortlessly extends to such a setup, but we do not consider it
for the clarity of the presentation. Moreover, our algorithm never needs to randomize its predictions.}
In this light, the results of our paper can be interpreted as proposing a simple and efficient iterative
algorithm, which does only a single pass over the data while achieving performance that is very close
to that of the aforementioned batch optimization methods.

\begin{remark}[Adversarial sequence of inputs] One might consider a setup in which the input instances $\vecx^n$ are not drawn i.i.d.\@ from a fixed
distribution, but rather form an arbitrary (possibly even adversarial) sequence, and the goal is to analyze the worst case with respect to
$\vecx^n$ (note that the labels remain stochastic, being drawn from $\veceta(\vecx)$).
Unfortunately, this setting turns out to be too difficult in the worst case: in \cref{sec:adversarial}, we show that there exists
a conditional probability function $\veceta$ (which we even make known to the algorithm), and a sequence of inputs $\vecx^n$, 
on which no algorithm can have a vanishing regret. Thus, a stochastic
data generation mechanism seems crucial for the existence of a no-regret learner.
\end{remark}

\section{The algorithm}
\label{sec:algorithm}

\begin{algorithm}[t]
\caption{Online Measure Maximization Algorithm}
\label{alg:OMMA}
\begin{small}
\begin{algorithmic}[0]
\STATE \textbf{Initialization:} confusion matrix $\matC_0$ 
\FOR{$t = 1, \ldots, n$}
    \STATE Receive input $\vecx_t$ and probability estimate $\empveceta_t(\vecx_t)$
    \STATE Predict $\vecpred_t = \argmax_{\vecpred} \!\nabla\! \psi(\matC_{t-1}) \!\cdot \!
    \left(\EE_{\vecy_t \sim \empveceta_t(\vecx_t)} \!\left[ \matC(\vecy_t,\vecpred)  \right]\right)$
    \STATE Receive label $\vecy_t$ and update $\matC_t = \frac{t-1}{t} \matC_{t-1} + \frac{1}{t}\matC(\vecy_t,\vecpred_t)$
\ENDFOR
\end{algorithmic}
\end{small}
\end{algorithm}

In this section, we introduce our algorithm, called \emph{Online Metric Maximization Algorithm (\OMMA{})}. 
To define the algorithm, we need to make the following assumption:

\begin{assumption}
The utility $\psi(\matC)$ is differentiable in $\matC$.
\end{assumption}
This property is shared by all performance metrics commonly 
used in practice.\footnote{In case of non-differentiable and concave metrics, one can still use a supergradient in place of gradient.}
Let $\nabla \psi(\matC)$ be the matrix of derivatives of $\psi$ with respect to $\matC$, which is of the same shape as $\matC$ ($m \times m$ for multi-class, and $m \times 2 \times 2$ for multi-label case). 
To simplify notation, we let $\matC_t \coloneqq \matC(\vecy^t,\vecpred^t)$ denote the confusion matrix of the algorithm after $t$ trials.
At trial $t$, given past labels $\vecy^{t-1}$, past predictions $\vecpred^{t-1}$, new input $\vecx_t$ and
the conditional probability estimate $\empveceta_t(\vecx_t)$, the \OMMA{}
algorithm predicts $\vecpred_t$ which maximizes the expected (with respect to $\empveceta_t$) \emph{linearized} (at $\matC_{t-1}$) version of the utility:%
\footnote{The algorithm is well-defined, as $\matC(\vecy_t,\vecpred)$ is \emph{linear} in
labels and thus its expectation only depends on marginals $\empveceta_t(\vecx_t)$.}
\begin{equation}
\vecpred_t = \argmax_{\vecpred} \nabla \psi(\matC_{t-1}) \cdot \left(\EE_{\vecy_t \sim \empveceta_t(\vecx_t)}  \left[ \matC(\vecy_t,\vecpred)  \right]\right) \,,
\label{eq:algorithm}
\end{equation}
where the maximization is over the allowed predictions $\vecpred$ (e.g., one-hot vectors in multi-class classification), and the operator ``$\cdot$'' 
denotes the matrix/tensor dot product (which is a standard vector dot product between vectorized matrix/tensor arguments).
Note that the last element in \eqref{eq:algorithm} is a confusion matrix computed from \emph{only a single observation $\vecy_t$ and a single prediction $\vecpred$}.
The algorithm does not need to store any past data or predictions beyond the entries of the confusion matrix, as the update of matrix
can be done online using \eqref{eq:confusion_matrix}. Due to the linearity of \eqref{eq:algorithm}, \OMMA{} is effectively a \emph{cost-sensitive classification}
rule with costs determined by the gradient of the utility at the current confusion matrix.
In what follows, we will show that formula \eqref{eq:algorithm}
leads to a very simple rule, boiling down to thresholding or sorting linear functions of $\empveceta_t(\vecx_t)$.\looseness=-1

\paragraph{Motivation.} \OMMA{} can be derived from considering a greedy method,
which, given past labels $\vecy^{t-1}$
and predictions $\vecpred^{t-1}$, chooses its prediction
$\vecpred_t$ in order to maximize the expected utility \emph{including} iteration $t$ (where the expectation is with respect to label $\vecy_t$, which is
unknown at the moment of making prediction), that is, to maximize 
$F(\vecpred_t) = \EE_{\vecy_t \sim \empveceta_t(\vecx_t)} \left[ \psi(\matC(\vecy^t,\vecpred^t)) \right]$. By using a Taylor-expansion of $F(\vecpred_t)$ with respect to $\vecy_t$,
it turns out that maximizing $F(\vecpred_t)$ and maximizing \eqref{eq:algorithm} are equivalent up to $\mathcal{O}(1/t^2)$ (see \cref{app:greedy} for
details). At the same time,
our algorithmic update \eqref{eq:algorithm} is based on a linear objective and, therefore, more straightforward to calculate, as well as easier to analyze.
In the experiment section, we show that both updates indeed behave very similarly to each other, being essentially indistinguishable for larger values
of $n$.

\paragraph{Example: binary classification.} To simplify the presentation,
we switch from one-hot vector notation to scalars and 
denote by $y_t \in \{0,1\}$ a label at time $t$ ($y_t=0$ corresponds to $\vecy_t=(1,0)$, 
and $y_t=1$ to $\vecy_t=(0,1)$); similarly $\pred_t \in \{0,1\}$ denotes the prediction
and $\eta(\vecx_t) = P(y_t =1 | \vecx_t)$.
The confusion matrix $\matC_{t-1} = \matC(\vecy^{t-1},\vecpred^{t-1})$ 
consists of four entries 
\begin{equation}
C_{t-1,j\ell} = \sum_{\mathclap{i \le t-1}} \bm{1}\{y_i=j\}\bm{1}\{\pred_i=\ell\},\quad j,\ell \in \{0,1\}\,.
\end{equation}
Let us abbreviate $\frac{\partial}{\partial_{C_{\!j\ell}}} \psi(\matC_{t-1})$ as $\nabla_{\!\!j\ell}$,
and $\empeta_t(\vecx_t)$ as $\empeta_t$. 
Since 
\begin{equation*}
    \EE_{y_t \sim \empeta_t}[ \matC(y_t,\pred) ] = \begin{bmatrix}(1-\empeta_t)(1-\pred) & (1-\empeta_t) \pred \\ \empeta_t (1-\pred) & \empeta_t \pred \end{bmatrix} \,,
\end{equation*}
equation \eqref{eq:algorithm} boils down to maximizing a cost-sensitive classification accuracy:
\begin{equation}
\!\!\nabla_{\!\!00} (1-\empeta_t)(\!1-\pred) + \nabla_{\!\!01\!}(1-\empeta_t) \pred + \nabla_{\!\!10} \empeta_t (1-\pred) + \nabla_{\!\!11} \empeta_t \pred,
\label{eq:binary_classification_prediction_rule}
\end{equation}
which can also be rewritten, up to terms independent of $\pred$, as $\pred (\alpha \empeta_t - \beta)$, with
\[
\alpha = \nabla_{11} + \nabla_{00} - \nabla_{01} - \nabla_{10}, \quad
\beta = \nabla_{00} - \nabla_{01},
\]
a cost-sensitive prediction rule \citep{Elkan_2001,JMLR:v18:15-226}. 
Since all utilities used in practice are non-decreasing with true positives and true negatives ($\nabla_{11},\nabla_{00} \ge 0$),
and non-increasing with false negatives and false positives ($\nabla_{01}, \nabla_{10} \le 0$), we get $a \ge 0$,
and thus maximizing $\pred (\alpha \empeta_t - \beta)$ boils down to choosing $\pred_t=1$ whenever the conditional probability
$\empeta_t(\vecx_t)$ exceeds a threshold $\nicefrac{\beta}{\alpha}$. It is well-known that, under mild assumptions on the input distribution
$\datadistribution_{\!\instancespace}$ and the metric $\psi$, thresholding $\eta(\vecx)$ is the optimal classification rule for maximizing
the population-level version of $\psi$ \citep{Koyejo_et_al_2014,Narasimhan_et_al_2014}.
Thus, \OMMA{} mimics the optimal classification rule, with a threshold
computed based on the empirical data.

\paragraph{Example: multi-label classification (with a budget).} Here $\vecy_t$ and $\vecpred_t$ are label and prediction
vectors of length $m$, and $\veceta(\vecx_t)$ is the vector of marginal label probabilities. $\matC_{t-1}$
is a sequence of binary confusion matrices $\matC^j_{t-1}$, $j=1,\ldots,m$. 
Abbreviating $\frac{\partial}{\partial C^j_{uv}} \psi(\matC_{t-1})$ as $\nabla^j_{uv}$,
and $\empveceta_t(\vecx_t)$ as $\empveceta_t$ allows us to write down the objective
\eqref{eq:algorithm} as a direct extension of the one from binary classification \eqref{eq:binary_classification_prediction_rule}, 
summed over $m$ labels:\looseness=-1
\begin{equation}
\sum_{j=1}^m \pred_j (\alpha_j \empeta_{tj} - \beta_j)
\label{eq:multi-label_classification_prediction_rule}
\end{equation}
with $\alpha_j = \nabla^j_{11} + \nabla^j_{00} - \nabla^j_{01} - \nabla^j_{10}$, $\beta_j = \nabla^j_{00} - \nabla^j_{01}$.
Maximizing \eqref{eq:multi-label_classification_prediction_rule} amounts to setting $\pred_{tj} = 1$
whenever $\alpha_j \empeta_{tj} - \beta_j \ge 0$, or $\empeta_{tj} \ge \nicefrac{\beta_j}{\alpha_j}$.
In the budgeted-at-$k$ variant, where the algorithm must predict 
\emph{exactly} $k$ labels each time, the prediction amounts to sorting the labels
by means of $\alpha_j \empeta_{tj} - \beta_j$ in descending order, and setting the top $k$
of them to $1$. Interestingly, this is closely related to the optimal population-level prediction rule
for the budgeted setting \citep{Schultheis_etal_ICLR2024}.

\section{Theoretical analysis}

In order to prove a bound on the regret of \OMMA{} algorithm, we need to make additional assumptions on the utility $\psi$.

\begin{assumption}
\label{as:utility}
$\psi(\matC)$ is differentiable, concave, $L$-Lipschitz, and $M$-smooth
in $\matC$, that is for any $\matC_1,\matC_2 \in \confusionmatrixspace$,
$\psi(\matC_1) \le \psi(\matC_2) + \nabla \psi(\matC_2)^\top (\matC_1 - \matC_2)$,
$|\psi(\matC_1)-\psi(\matC_2)\| \le L \|\matC_1 - \matC_2\|$,
and
$\|\nabla \psi(\matC_1) - \nabla \psi(\matC_2)\| \le M \|\matC_1 - \matC_2\|$, where
$\|\cdot\|$ denotes the entrywise $L_2$-norm.
\end{assumption}

\begin{remark}
Our algorithm generally requires smoothness of the objective to converge. 
The assumption can be waived by running it on a \emph{smoothed} version of the metric. This, however,
comes at the price of a slower convergence rate $\mathcal{O}(\sqrt{\ln n / n})$.
As non-smooth metrics are not commonly employed in machine learning, we relegate the
discussion to \cref{app:non-smooth}
\end{remark}

\begin{theorem}
\label{thm:regret}
Let \cref{as:utility} hold. Then the \OMMA{} algorithm has its regret bounded by:
\[
\regret_n \le 
\frac{M a (1+ \ln n)}{n} + 
\frac{2 b L}{n} \sum_{t=1}^n \EE\left[ \|\veceta(\vecx_t) - \empveceta_t(\vecx_t)\| \right]
\]
where $a=1, b=1$ for multi-class, and
$a=m, b=\sqrt{2}$ for multi-label classification.
\end{theorem}
The proof is given in \cref{app:regret_bounds}. The bound in \cref{thm:regret}
consists of two parts: (1) the first term of order $\mathcal{O}(\frac{\ln n}{n})$
can be interpreted as the regret of \OMMA{} had it been equipped with the true conditional probability
$\veceta$ (that is, the estimator is exact); (2) the estimation error of $\empveceta_t$ averaged over
trials $t=1,\ldots,n$. If the estimation error converges 
to zero with $n \to \infty$, \OMMA{} becomes a no-regret learning algorithm.
\begin{remark}
\label{rem:smoothing}
In the proof of \cref{thm:regret},
the Lipschitzness and the smoothness properties 
are invoked along the parameter path of the algorithm, $\{\matC_t\}_{t=1}^n$,
in order to control the progress in optimizing the utility.
These properties might not necessarily hold
globally (for every confusion matrix) 
for utilities given in \cref{tab:binary-metrics},
for instance when the observed labels lead to severe class imbalance.
However, adding a small constant to the denominator
in the definition of a utility stabilizes its values and ensures
its global Lipschitzness and smoothness. This is also the
approach we take in the experiments to keep our algorithm
stable over the initial part of the data sequence.

An alternative approach is to 
use the fact that most utilities in \cref{tab:binary-metrics}
are Lipschitz and smooth when label frequencies in the confusion
matrix are bounded away from zero. Taking into account that these
properties are applied along the path of the algorithm, 
and using concentration inequalities on the label frequencies,
we show in \cref{app:alternative_proof} 
that as long as the probabilities of labels
$\datadistribution (y_j = 1)$ are bounded away from zero,
the regret of $\OMMA{}$ converges at a rate of $\mathcal{O}(\sqrt{\ln n / n})$ with high probability. 
Motivated by this fact, we also propose to use a regularization technique that, additionally to using a small constant in the denominator of some metrics, also adds the small value $\lambda$ to the initial entries of the confusion matrix.
With the updates of the confusion matrix, this value diminishes, being $\frac{\lambda}{t}$ at iteration $t$ of the algorithm. This simple technique turns out to be helpful for some metrics.

\end{remark}
\section{Alternative variants of the algorithm}

\begin{algorithm}[ht]
\caption{\OMMAeta{}}
\label{alg:OMMA_eta}
\begin{small}
\begin{algorithmic}[0]
\STATE \textbf{Initialization:} confusion matrix $\matC_0$ 
\FOR{$t = 1, \ldots, n$}
    \STATE Receive input $\vecx_t$ and probability estimate $\empveceta_t := \empveceta_t(\vecx_t)$
    \STATE Predict $\vecpred_t = \argmax_{\vecpred} \!\nabla \psi(\matC_{t-1}) \!\cdot \!
     \matC(\empveceta_t,\vecpred)$
    \STATE Update $\matC_t = \frac{t-1}{t} \matC_{t-1} + \frac{1}{t}\matC(\empveceta_t,\vecpred_t)$
\ENDFOR
\end{algorithmic}
\end{small}
\end{algorithm}

\paragraph{Internal semi-empirical confusion matrix.} In the experiment, we also use an alternative version of the $\OMMA{}$ algorithm called
\OMMAeta{}, outlined in \cref{alg:OMMA_eta}.
Note that in the algorithm's description we use $\matC(\empveceta_t,\vecpred_t) = \EE_{\vecy_t \sim \empveceta_t}\matC(\vecy_t,\vecpred_t)$ 
which follows from the fact that the confusion matrix is linear in labels and that $\EE_{\vecy_t \sim \empveceta_t}[\vecy_t]
= \empveceta_t$. The only difference between \OMMA{} and its modification is that \OMMAeta{} updates its running
confusion matrix $\matC_t$ by means of $\empveceta_t$ instead of the true label $\vecy_t$.
In fact, \OMMAeta{} does not use the true labels at all, fully trusting its CPE.
We remark that $\matC_t$ does \emph{not} correspond to any empirical confusion matrix (as it is not based on the labels),
so it should rather be treated as an internal parameter of \OMMAeta{}.
In \cref{app:OMMA_eta} we show that if the algorithm's probability estimator is
\emph{exact}, that is $\empveceta_t \equiv \veceta$ for all $t$,
\OMMAeta{} achieves (under the same assumptions as before) a regret bound $\regret_n \le \frac{Ma(2+\ln n)}{n}$, which is very similar
to that of the original \OMMA{} algorithm in \cref{thm:regret}. It is, however, unclear whether the modified algorithm
converges for a non-exact CPE. Still, it turns out that \OMMAeta{} performs surprisingly well in the experiments.

\paragraph{Sparse variant for a large number of labels.} The introduced \OMMA{} algorithm requires calculating the derivative %
with respect to every entry of the confusion matrix at each step, resulting in the total complexity of $\mathcal{O}(nm)$ over the entire sequence of data (assuming constant time for the gradient calculation of a single entry). 
Even this can be fairly expensive in case of a large number of labels (e.g., in recommendation or extreme classification); in these multi-label problems, however, 
the number of positive labels per sample is much smaller than the total number of label, $\|\vecy \|_1 \ll m$, and often most of the conditional probabilities in $\veceta(\vecx)$ are 0 or very close to 0.
Many recommenders and extreme classifiers are naturally designed to predict only top-$k'$ entries with the highest values of $\eta(\vecx)$.
We can then leverage the sparsity of top-$k'$ assuming $\eta(\vecx) = 0$ for the rest of the labels and calculate the gradient with respect to entries of the confusion matrix that correspond to non-zero values of $\veceta(\vecx)$, resulting in a total complexity of $\mathcal{O}(nk')$ on the entire data sequence. With reasonably selected $k'$, according to~\cref{thm:regret}, we should only slightly increase the regret.

\begin{table*}[ht]
    \caption{Results of the different online algorithms on \emph{multi-label} problems, averaged over 5 runs, reported as \%. 
    In this table we report the final performance obtained on the sequence of $n$ samples. 
    The best result on each metric is in \textbf{bold}, the second best is in \textit{italic}. We additionally report basic statistics of the benchmarks: number of labels $m$ and instances in the test sequence $n$. $\times$ -- means that the algorithm does not support the optimization of that metric.}
    \label{tab:main-results-multi-label}
    \vspace{8pt}
    \small
    \centering
\resizebox{\linewidth}{!}{
\setlength\tabcolsep{4 pt}
\begin{tabular}{l|c|cccccc|c|cccccc}
\toprule
    Method & Micro & \multicolumn{6}{c|}{Macro} & Micro & \multicolumn{6}{c}{Macro} \\
    & F1 & F1 & F1$@3$ & Rec.$@3$ & Pr.$@3$ & G-mean & H-mean & F1 & F1 & F1$@3$ & Rec.$@3$ & Pr.$@3$ & G-mean & H-mean \\
\midrule
& \multicolumn{7}{c|}{\datasettable{YouTube} ($m = 46, n = 7926$)} & \multicolumn{7}{c}{\datasettable{Eurlex-LexGlue} ($m = 100, n = 5000$)} \\ 
 \midrule
    Top-$k$ / $\hat \eta \!>\!0.5$ & 31.20 & 22.74 & 30.99 & 42.13 & 26.39 & 32.82 & 24.46 & 70.99 & 52.43 & 46.35 & 36.67 & 74.70 & 62.33 & 55.95 \\
    \midrule
    OFO & 43.71 & 36.15 & $\times$ & $\times$ & $\times$ & $\times$ & $\times$ & 73.23 & 58.93 & $\times$ & $\times$ & $\times$ & $\times$ & $\times$ \\
    Greedy & $\times$ & 36.32 & 34.72 & 45.84 & \textit{67.18} & \textbf{77.98} & \textbf{77.93} & $\times$ & 59.83 & 54.19 & 52.67 & \textit{88.21} & 89.74 & 89.73 \\
    Online-FW & 43.67 & 36.00 & 34.38 & 45.83 & 38.92 & \textit{77.96} & \textit{77.91} & \textbf{73.68} & 59.53 & 54.18 & 52.68 & 58.76 & 89.69 & 89.75 \\
    Online-FW$(\hat \eta)$ & 43.69 & \textbf{36.47} & \textbf{35.43} & \textbf{45.89} & 50.12 & 77.92 & \textbf{77.93} & 73.22 & 59.78 & \textit{54.34} & \textbf{53.87} & 51.37 & \textit{89.91} & \textit{89.84} \\
    \midrule
    \OMMA{} & \textbf{43.73} & \textit{36.34} & 34.81 & \textit{45.85} & \textbf{67.74} & \textbf{77.98} & \textbf{77.93} & \textit{73.29} & \textbf{59.85} & 54.15 & 52.67 & \textbf{88.41} & 89.74 & 89.73 \\
    \OMMAeta{} & \textit{43.72} & \textbf{36.47} & \textit{35.38} & \textbf{45.89} & 65.49 & \textbf{77.98} & \textbf{77.93} & 73.22 & \textit{59.84} & \textbf{54.37} & \textit{53.85} & 84.99 & \textbf{89.92} & \textbf{89.85} \\
\midrule
& \multicolumn{7}{c|}{\datasettable{Mediamill} ($m = 101, n = 12914$)} & \multicolumn{7}{c}{\datasettable{Flickr} ($m = 195, n = 24154$)} \\ 
 \midrule
    Top-$k$ / $\hat \eta \!>\!0.5$ & 52.45 & 4.06 & 4.43 & 4.35 & 7.42 & 4.62 & 3.65 & 29.46 & 18.27 & 26.39 & 38.96 & 21.38 & 27.08 & 20.02 \\
    \midrule
    OFO & \textit{56.99} & 12.36 & $\times$ & $\times$ & $\times$ & $\times$ & $\times$ & \textbf{41.05} & 30.46 & $\times$ & $\times$ & $\times$ & $\times$ & $\times$ \\
    Greedy & $\times$ & 12.43 & 10.29 & \textit{9.41} & \textit{22.54} & \textit{65.68} & \textbf{66.48} & $\times$ & 30.90 & \textit{30.42} & \textbf{46.41} & \textit{57.55} & \textit{83.39} & \textbf{83.37} \\
    Online-FW & 56.98 & 12.22 & 10.18 & 9.16 & 12.70 & 64.73 & \textit{65.68} & \textbf{41.05} & 30.60 & 29.58 & \textit{46.38} & 28.66 & 83.37 & \textit{83.28} \\
    Online-FW$(\hat \eta)$ & \textit{56.99} & \textit{14.33} & \textbf{11.97} & \textbf{16.43} & 17.35 & 65.37 & 65.56 & \textit{41.02} & \textbf{31.17} & 29.65 & 46.28 & 25.73 & 83.37 & 83.20 \\
    \midrule
    \OMMA{} & \textbf{57.00} & 12.39 & 10.25 & 9.39 & \textbf{22.85} & 65.67 & \textbf{66.48} & 41.01 & 30.90 & 30.39 & \textbf{46.41} & \textbf{58.35} & \textit{83.39} & \textbf{83.37} \\
    \OMMAeta{} & \textit{56.99} & \textbf{14.34} & \textit{11.87} & \textbf{16.43} & 18.31 & \textbf{65.93} & 65.35 & \textit{41.02} & \textit{31.15} & \textbf{30.55} & 46.33 & 55.56 & \textbf{83.41} & 83.23 \\
\midrule
& \multicolumn{7}{c|}{\datasettable{RCV1X} ($m = 2456, n = 155962$)} & \multicolumn{7}{c}{\datasettable{AmazonCat} ($m = 13330, n = 306784$)} \\ 
 \midrule
    Top-$k$ / $\hat \eta \!>\!0.5$ & 68.57 & 11.29 & 5.34 & 4.59 & 13.24 & 16.01 & \textit{12.36} & 67.77 & 28.76 & 14.98 & 11.18 & 30.98 & \textit{33.93} & \textit{30.03} \\
    \midrule
    OFO & \textbf{69.83} & 20.26 & $\times$ & $\times$ & $\times$ & $\times$ & $\times$ & \textit{70.38} & 39.60 & $\times$ & $\times$ & $\times$ & $\times$ & $\times$ \\
    Greedy & $\times$ & \textbf{20.80} & \textit{16.01} & 21.20 & \textbf{30.91} & \textbf{69.07} & \textbf{67.04} & $\times$ & 44.20 & 43.64 & 57.81 & \textbf{54.00} & \textbf{80.86} & \textbf{80.04} \\
    Online-FW & \textbf{69.83} & 19.82 & 15.33 & 21.09 & 19.88 & \textbf{69.07} & \textbf{67.04} & \textbf{70.61} & 42.42 & 40.20 & 57.74 & 40.10 & \textbf{80.86} & \textbf{80.04} \\
    Online-FW$(\hat \eta)$ & \textit{69.79} & 20.40 & 15.55 & \textit{22.21} & 22.17 & \textit{69.06} & \textbf{67.04} & 70.34 & \textit{47.32} & \textit{44.63} & \textit{58.87} & 49.68 & \textbf{80.86} & \textbf{80.04} \\
    \midrule
    \OMMA{} & 69.77 & 20.57 & 15.87 & 21.08 & \textit{30.39} & \textbf{69.07} & \textbf{67.04} & 69.90 & 43.04 & 42.14 & 57.80 & \textit{52.04} & \textbf{80.86} & \textbf{80.04} \\
    \OMMAeta{} & 69.71 & \textit{20.71} & \textbf{16.07} & \textbf{22.23} & 30.38 & \textit{69.06} & \textbf{67.04} & 70.02 & \textbf{47.67} & \textbf{45.06} & \textbf{58.89} & 51.52 & \textbf{80.86} & \textbf{80.04} \\
\bottomrule
\end{tabular}
    }
\end{table*}

\begin{table}[ht]
    \caption{Results of the online algorithms on \emph{multi-class} problems, averaged over 5 runs, reported as \%. 
    In this table we report the final performance obtained on the sequence of $n$ samples. 
    }
    \label{tab:main-results-multi-class}
    \vspace{8pt}
    \small
    \centering

\resizebox{\linewidth}{!}{
\setlength\tabcolsep{4 pt}
\begin{tabular}{l|cccc|ccc}
\toprule
    Method & \multicolumn{4}{c|}{Macro} & \multicolumn{3}{c}{Multi-class means} \\
    & F1 & F1$@3$ & Rec.$@3$ & Pr.$@3$ & G- & H- & Q-\\
\midrule
& \multicolumn{7}{c}{\datasettable{News20} ($m = 20, n = 7532$)} \\
 \midrule
    Top-$k$ & \textbf{83.37} & 49.23 & 94.78 & 33.95 & 82.51 & 81.66 & 80.12  \\
    \midrule
    Greedy & 82.44 & 72.69 & \textit{94.89} & \textbf{83.89} & $\times$ & $\times$ & $\times$  \\
    On.-FW & 82.88 & 54.10 & 94.88 & 19.90 & 82.68 & \textit{82.35} & 80.95  \\
    On.-FW$(\hat \eta)$ & 82.55 & 56.06 & \textbf{95.01} & 15.58 & \textbf{82.78} & \textbf{82.57} & \textit{80.97}  \\
    \midrule
    \OMMA{}& 82.11 & \textit{72.84} & \textit{94.89} & \textit{83.88} & 82.72 & 82.08 & 80.95  \\
    \OMMAeta{} & \textit{83.07} & \textbf{73.41} & \textbf{95.01} & 83.30 & \textit{82.77} & 82.18 & \textbf{81.00}  \\
\midrule
& \multicolumn{7}{c}{\datasettable{Ledgar-LexGlue} ($m = 100, n = 10000$)} \\ 
 \midrule
    Top-$k$ & 79.06 & 51.80 & 92.04 & 38.88 & 0.00 & 0.00 & 69.21 \\
    \midrule
    Greedy & \textit{79.30} & 78.08 & 93.26 & \textit{91.17} & $\times$ & $\times$ & $\times$  \\
    On.-FW & 79.22 & 70.94 & 93.30 & 54.52 & 62.31 & 75.81 & \textbf{74.59} \\
    On.-FW$(\hat \eta)$ & 79.22 & 73.85 & \textit{93.38} & 49.63 & \textit{78.02} & \textbf{77.58} & 74.50  \\
    \midrule
    \OMMA{}& 79.28 & \textit{78.10} & 93.26 & \textbf{91.41} & 77.48 & 74.62 & \textbf{74.59}  \\
    \OMMAeta{} & \textbf{79.34} & \textbf{78.22} & \textbf{93.39} & 90.10 & \textbf{78.03} & \textit{76.08} & \textit{74.53}  \\
\midrule
& \multicolumn{7}{c}{\datasettable{Caltech-256} ($m = 256, n = 14890$)} \\ 
 \midrule
    Top-$k$ & 79.45 & 46.82 & 89.85 & 32.58 & 77.32 & 75.69 & 74.53  \\
    \midrule
    Greedy & \textit{79.59} & \textit{79.02} & 90.20 & 96.61 & $\times$ & $\times$ & $\times$  \\
    On.-FW & 79.15 & 70.29 & 90.20 & 63.22 & 78.31 & \textit{77.99} & 76.00  \\
    On.-FW$(\hat \eta)$ & 79.29 & 72.97 & \textit{90.34} & 65.42 & \textbf{78.41} & \textbf{78.08} & \textit{76.07}  \\
    \midrule
    \OMMA{}& 79.54 & 78.96 & 90.20 & \textit{96.77} & 78.33 & 77.12 & \textbf{76.15}  \\
    \OMMAeta{} & \textbf{79.66} & \textbf{79.11} & \textbf{90.35} & \textbf{96.96} & \textit{78.36} & 77.01 & 75.99  \\
\bottomrule
\end{tabular}
    }
\end{table}

\section{Empirical study}
\label{sec:experiments}

To demonstrate the practicality and generality of the introduced \OMMA{} algorithm, 
we test it on a wide range of multi-label and multi-class benchmark datasets that differ substantially in the number of labels, ranging from tens to a few thousands, and in the imbalance of the label distribution \footnote{
Code to reproduce the experiments: \\\url{https://github.com/mwydmuch/xCOLUMNs}
}.
For multi-class experiments, we use News20~\citep{lang1995newsweeder}, Ledgar-LexGlue~\citep{chalkidis-etal-2022-lexglue} with tf-idf features, Caltech-256~\citep{Caltech256} with features obtained using VGG16~\citep{Simonyan2014DCNN} trained on ImageNet, and for multi-label experiments, we use
YouTube, Flickr~\citep{Tang_and_Liu_2009} with DeepWalk features~\citep{Perozzi_et_al_2014}, Eurlex-LexGlue~\citep{chalkidis2021lexglue}, Mediamill~\citep{Snoek_et_al_2006}, RCV1X~\citep{lewis2004rcv1}, and AmazonCat~\citep{mcauley2013hidden,Bhatia_et_al_2016} with tf-idf features.
We conduct two types of experiments:
\begin{enumerate}
    \item With fixed conditional probability estimator (CPE) -- we train the CPE using $L_2$-regularized logistic loss %
    on the provided training sets and use the obtained CPE to predict all $\empveceta$ for the test set, which are then used in 
    the online algorithms on the test set using the protocol from~\cref{fig:online-protocol}.
    For benchmarks without default train and test sets, we split them randomly in proportion 70/30.
    \item With online CPE -- instead of training CPE on a separate training set, it incrementally updated on observed instances in the sequence. In this case, we concatenate train and test sets to create one long sequence used in the online protocol. We report the results of this experiment in \cref{app:online-cpe-results}.
\end{enumerate}

Each experiment is repeated five times, each time randomly shuffling the sequence. We report the mean results over all runs.
In~\cref{app:experiments-setup}, we include more details regarding the experimental setup.%

We compare \OMMA{} with the following algorithms:
\begin{itemize}%
    \item Top-$k(\empveceta(\vecx))$ -- a classifier that selects $k$ labels with the highest values in $\empveceta(\vecx)$. 
    It is used as a baseline in multi-class and multi-label problems with the budget $k$ constrain. For multi-class problems we use $k=1$ if the budget is not specified.
    \item $\empeta(\vecx) > 0.5$ -- a classifier with a constant threshold that predicts a label as positive if its conditional probability is greater than 0.5. It is used as a baseline in multi-label problems without the budget $k$ constrain.
    \item OFO -- a consistent online algorithm for the (micro- and macro-averaged) F-measure in the multi-label setting~\citep{Busa-Fekete_et_al_2015, Jasinska_et_al_2016}.
    \item Greedy -- an algorithm that given past labels and predictions chooses its next prediction in order to maximize the expected utility. It might be seen as a close approximation of the \OMMA{} algorithm. It has been motivated by a similar algorithm recently introduced for batch multi-label classification with predictions budgeted at $k$~\citep{Schultheis_etal_NeurIPS2023}. An efficient implementation is possible for metrics that linearly decompose over labels (e.g., macro-averaged metrics).
    \item Online-FW -- an adaptation of the Frank-Wolfe algorithm, used earlier for batch multi-class~\citep{Narasimhan_et_al_2015, Narasimhan_et_al_2022} and (budgeted at $k$) multi-label problems~\citep{Schultheis_etal_ICLR2024}, to the online setting. It re-runs the batch algorithm on all instances observed so far (which need to be stored in memory) in exponentially growing intervals of $10 \times 1.1^i$ instances with $i \in \{0, 1, 2, 3, \dots\}$, which we found to provide frequent enough updates for achieving good predictive performance. 
    Since this algorithm, similar to \OMMA{}, can be run for all metrics considered in the study, we treat it as the main baseline for our algorithm. 
    For a fair comparison, we also introduce an alternative variant that uses $\empveceta$ instead of true labels $\vecy$ to estimate its internal confusion matrix. We denote this algorithm as Online-FW$(\empeta)$.
    Note that despite calling this method ``online,'' it requires storing all previously seen instances in memory and thus does not fully adhere to the online paradigm.
    \item Offline-FW -- the batch variant of the Frank-Wolfe algorithm used earlier in~\citep{Narasimhan_et_al_2015, Schultheis_etal_ICLR2024}, trained using the same training set that was used to obtain CPE.
    It then uses the same probability estimates as online methods for inference on the test sequence without further updates to the classifier. The comparison with this algorithm is reported in \cref{app:extended-results}.
\end{itemize}

As discussed in \cref{rem:smoothing}, we add both small constant $\epsilon=\text{1e-9}$ to the denominators in all metrics, as well as the regularization value $\lambda \in \{0, \text{1e-6}, \text{1e-3}, 0.1, 1\}$ to the entries of the confusion matrix used in all online algorithms (Greedy, OFO, Online-FW, and \OMMA{}).

We use the following popular metrics for evaluating the algorithms: %
Micro and Macro-averaged F1, which are widely used to evaluate classifiers in many domains such as information retrieval; 
budgeted at $k = 3$ Macro F1, Macro Recall, and Macro Precision being well-suited to recommendation systems and extreme classification; 
and Macro G-mean and H-mean for multi-label problems and Multi-class G-mean, H-mean and Q-mean for multi-class problems, which are frequently used in class imbalance problems. 
Note that \OMMA{} and Online-FW can target these metrics directly, while the scope of other algorithms is limited. 

We present the results of the experiment with fixed CPE in~\cref{tab:main-results-multi-label,tab:main-results-multi-class}.
We report the mean performance on the entire sequence of $n$ instances for the best value of $\lambda$. Additionally, we present incremental performance on~\cref{fig:main-plots} and the effect of using different $\lambda$ values on~\cref{fig:main-reg-plots} for the Flickr dataset.
The introduced \OMMA{} algorithm matches the performance of the online Frank-Wolfe algorithm for most of the measures, 
performing much better on Macro-Precision@3, which is not Lipschitz. 
Concurrently, \OMMA{} uses less time and memory as it does not require storing all previously seen instances. 
We also observe that \OMMA{}, as an approximation of Greedy, matches its performance on metrics supported by this method. 
This is additionally confirmed by the plots where we can observe that the performance of these algorithms is very close to each other at each iteration $t$. 
Surprisingly, \OMMAeta{} often performs better on macro-averaged metrics, especially on benchmarks with a large number of labels, 
where many of them have only a small number of positive samples. In these cases, the variance reduction effect of using $\empveceta$ seems to compensate the estimation error of $\veceta$. It is also much less sensitive to the selection of $\lambda$, as small probability values that are added to many entries of the confusion matrix give the same effect as the regularization value $\lambda$.
Extended results with additional algorithms, standard deviations, running times, and plots for the rest of the datasets can be found in \cref{app:extended-results}.

\begin{figure}[t!]
    \centering
    \includegraphics[width=0.235\textwidth]{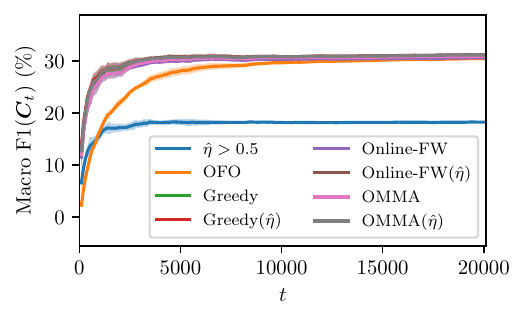}
    \includegraphics[width=0.235\textwidth]{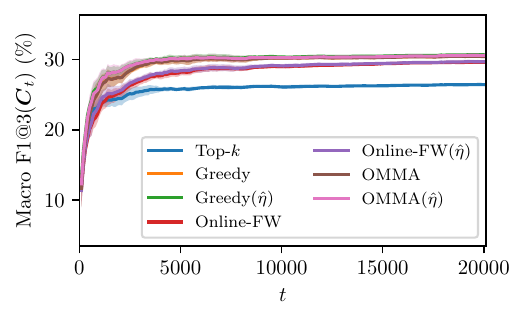}
    \vspace{-8pt}
    \caption{Comparison of the incremental performance of the online algorithms on the Flickr dataset. Averaged over 5 runs, the opaque fill indicates the standard deviation at given iteration $t$.}
    \label{fig:main-plots}
\end{figure}
\begin{figure}[t!]
    \centering
    \includegraphics[width=0.235\textwidth]{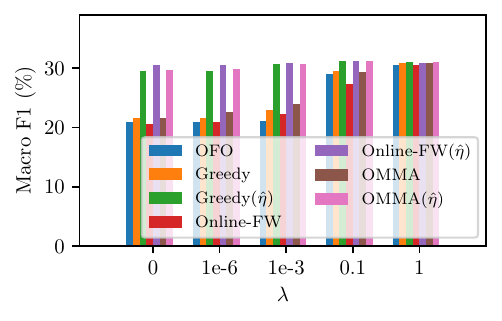}
    \includegraphics[width=0.235\textwidth]{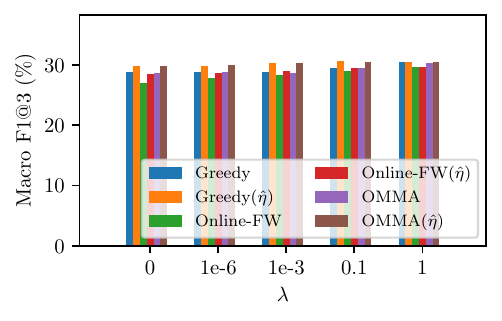}
    \vspace{-8pt}
    \caption{Impact of $\lambda$ on the results of the online algorithms on the Flickr dataset. Averaged over 5 runs.}
    \label{fig:main-reg-plots}
\end{figure}
\section{Discussion}
\label{sec:discussion}

In our framework, the algorithm aims to maximize the empirical value of the utility of the observed data sequence, 
$\psi(\matC(\vecy^n, \vecpred^n))$.
An interesting direction of research is an alternative online framework, in which the goal is to sequentially
learn a \emph{classifier} $\widehat{\vech}_n \colon \instancespace \to \labelspace$, which is then
evaluated on the entire population by means of $\psi(\matC(\widehat{\vech}_n))$, where
$\matC(\widehat{\vech}_n) = \EE_{(\vecx,\vecy)} [ \matC(\vecy, \widehat{\vech}_n(\vecx)) ]$
is the population confusion matrix of $\widehat{\vech}_n$.

In the analysis of the algorithm,
we assumed the concavity of the metric. On one hand,
\emph{some} assumption of this type seems essential for no-regret learning; for instance,
in multi-class macro-averaged $F$-measure learning, the utility belongs to function class, optimizing which
is, in general, NP-hard \citep{Narasimhan_eta_al_ICDM2016}. %
On the other hand, our concavity assumption excludes an important class of linear-fractional functions (such as binary or
micro-averaged $F$-measure), for which a specialized version of our algorithm is known to converge, although without
providing a convergence rate \citep{Busa-Fekete_et_al_2015}.
The analysis of our method for linear-fractional utilities is an interesting direction for future research, especially
given that our algorithm works very well on these metrics in the experiments.

\section*{Impact Statement}

This paper presents work whose goal is to advance the field of Machine Learning. 
There are many potential societal consequences of our work, 
none which we feel must be specifically highlighted here.

\section*{Acknowledgments}

This research has been supported by Poznan University of Technology via grant no. 0311/SBAD/0758 and by Academy of Finland via grants no. 347707 and 348215.

\bibliography{lit}
\bibliographystyle{icml2024}

\newpage
\appendix
\onecolumn

\section{Motivation for the \OMMA{} algorithm}
\label{app:greedy}

In this section, we motivate the \OMMA update \eqref{eq:algorithm} as a $O(1/t^2)$ approximation
of a greedy optimization method defined below. As this section is only meant to be a motivation form the algorithm,
we keep the derivation somewhat informal.
Fix a trial $t$, and let $\vecy^{t-1}$ and $\vecpred^{t-1}$
be the past predictions and labels, respectively.
Given $\vecx_t$, let $\empveceta_t(\vecx_t)$ be the conditional probability
estimate, which we concisely denote as $\empveceta_t$. The greedy method
is defined as returning prediction which \emph{maximizes the $\empveceta_t$-expected utility up to (and including)
trial $t$} (the expectation is with respect to label $\vecy_t$, which is
unknown at the moment of making prediction). That is, the greedy algorithm maximizes:
\[
F(\vecpred_t) = \EE_{\vecy_t \sim \veceta_t} \left[ \psi(\matC(\vecy^t,\vecpred^t)) \right]
\]
with respect to $\vecpred_t$. Denote $\matC_t = \matC(\vecy^t,\vecpred^t)$, $\matC_{t-1} = \matC(\vecy^{t-1},\vecpred^{t-1})$,
and $\bm\Delta_t = \matC(\vecy_t, \vecpred_t)$. Note that by \eqref{eq:confusion_matrix} we have
\[
\matC_t = \frac{t-1}{t} \matC_{t-1} + \frac{1}{t} \bm\Delta_t = \matC_{t-1} + \frac{1}{t} (\bm\Delta_t - \matC_{t-1}).
\]
Taylor-expanding $\psi(\matC_t)$ around $\matC_{t-1}$, and leaving only the first-order terms gives
\begin{align*}
\psi(\matC_t) &= \psi(\matC_{t-1}) + \frac{1}{t} \nabla \psi(\matC_{t-1}) \cdot (\bm\Delta_t - \matC_{t-1})
+ O \left(\frac{1}{t^2}\right) \\
&= \mathrm{const}(\bm \Delta_{t}) + \frac{1}{t} \nabla \psi(\matC_{t-1}) \cdot \bm\Delta_t + O \left(\frac{1}{t^2}\right).
\end{align*}
Taking expectation with respect to $\vecy_t$ on both sides gives:
\[
F(\vecpred_t) = \mathrm{const}(\vecpred_t) + \EE_{\vecy_t \sim \veceta_t} \left[\nabla \psi(\matC_{t-1}) \cdot \matC(\vecy_t, \vecpred_t) \right]
+ O \left(\frac{1}{t^2}\right),
\]
maximizing of which is equivalent (up to $O(1/t^2)$ to \OMMA{} update \eqref{eq:algorithm}.

\section{Regret bounds}
\label{sec:regret_bounds}

\subsection{Proof of Theorem \ref{thm:regret}}
\label{app:regret_bounds}

Let $\vech^{\star}$ be the optimal classifier. Let
\begin{equation}
\matC^{\star} = \EE_{(\vecx, \vecy)}\left[ \matC(\vecy,\vech^{\star}(\vecx)) \right],
\label{eq:C_star}
\end{equation}
denote its \emph{population confusion matrix}, which is the expected value of confusion matrix entries
on an instance $(\vecx, \vecy)$ randomly drawn from $\datadistribution$.
Due to concavity of $\psi$, we have:
\[
\EE_{(\vecx, \vecy)^n} \left[ \psi(\matC(\vecy^n,\vech^{\star}(\vecx^n))) \right]
\le \psi\left(\EE_{(\vecx, \vecy)^n}\left[ \matC(\vecy^n,\vech^{\star}(\vecx^n)) \right] \right)
= \psi\left(\matC^{\star} \right),
\]
thus the regret can be upper-bounded by:
\begin{equation}
\regret_n \le \psi(\matC^{\star})  - \EE\left[\psi(\matC_n)\right].
\label{eq:regret_upperbound_with_population_optimal_cm}
\end{equation}
where we used $\matC_n = \matC(\vecy^n, \vecpred^n)$. Comparing the algorithm with the population version of the optimal
confusion matrix, rather than the expected empirical one, will turn out to be useful in making an important argument later in the proof.

We now switch to the analysis of the \OMMA{} update \eqref{eq:algorithm}. One problematic issue there is that
the algorithms makes use of $\empveceta_t$, which might not coincide with the true
conditional $\veceta$, potentially leading to different predictions in both cases. We will, however, show now
that the discrepancy in algorithm's prediction due to $\empveceta_t$ can be bounded by its estimation error
$\|\empveceta_t(\vecx_t) - \veceta(\vecx_t)\|$, and is thus well controlled. To this end, we write
down the \OMMA{} update \eqref{eq:algorithm} as:
\begin{equation}
\vecpred_t = \argmax_{\vecpred} \widehat{F}(\vecpred),
\qquad \text{where~~} \widehat{F}(\vecpred) = \nabla \psi(\matC_{t-1}) \cdot \left(\EE_{\vecy_t \sim \textcolor{black!40!red}{\empveceta_t(\vecx_t)}}  \left[ \matC(\vecy_t,\vecpred)  \right]\right).
\label{eq:algorithm_update_in_the_proof}
\end{equation}
Let $\widetilde{\vecy}_t$ be a hypothetical prediction made as in \eqref{eq:algorithm_update_in_the_proof}, but with
\textcolor{black!40!red}{$\empveceta_t(\vecx_t)$} replaced by the true conditional probability \textcolor{black!50!green}{$\veceta(\vecx_t)$}, that is
\[
\widetilde{\vecy}_t = \argmax_{\vecpred} \widetilde{F}(\vecpred)
\qquad \text{where~~} \widetilde{F}(\vecpred) = \nabla \psi(\matC_{t-1}) \cdot \left(\EE_{\vecy_t \sim \textcolor{black!50!green}{\veceta(\vecx_t)}}  \left[ \matC(\vecy_t,\vecpred)  \right]\right).
\]
We will bound:
\begin{align*}
\widetilde{F}(\widetilde{\vecy}_t) - \widetilde{F}(\vecpred_t)
&= \widetilde{F}(\widetilde{\vecy}_t) - \widehat{F}(\widetilde{\vecy}_t) + \overbrace{\widehat{F}(\widetilde{\vecy}_t) - \widehat{F}(\vecpred_t)}^{\le 0 \text{~~from \eqref{eq:algorithm_update_in_the_proof}}} + \widehat{F}(\vecpred_t) - \widetilde{F}(\vecpred_t) \\
&\le 2 \max_{\vecpred} |\widetilde{F}(\vecpred) - \widehat{F}(\vecpred)| \\
&\le 2  \|\nabla \psi(\matC_{t-1})\| \max_{\vecpred} \Big\|\underbrace{\EE_{\vecy_t \sim \veceta(\vecx_t)}  \left[ \matC(\vecy_t,\vecpred)\right]
- \EE_{\vecy_t \sim \empveceta_t(\vecx_t)}  \left[ \matC(\vecy_t,\vecpred) \right]}_{\bm\Delta_t} \Big\|,
\end{align*}
where the last inequality is from Cauchy-Schwarz inequality. Due to Lipschitzness of $\psi$,
$\|\nabla \psi(\matC_{t-1})\| \le L$. For multi-class classification, $\bm\Delta_t$ is a matrix with entries equal to $\bm\Delta_{t,j\ell} = (\eta_j(\vecx_t) - \empeta_{tj}(\vecx_t))\pred_{\ell}$. Due to one-hot constraint $\sum_{\ell} \pred_{\ell}=1$, we get that for any $\vecpred$,
$\|\bm \Delta_t\|^2 = \sum_{j=1}^m |\eta_j(\vecx_t) - \empeta_{tj}(\vecx_t)|^2 = \|\veceta(\vecx_t) - \empveceta_t(\vecx_t)\|^2$.
For multi-label classification $\bm\Delta_t$ is a sequence of $m$ $2 \times 2$ matrices, with $j$-th matrix having entries
\[
\begin{pmatrix}
(\empeta_{tj}(\vecx_t) - \eta_j(\vecx_t))(1-\pred_j) & (\empeta_{tj}(\vecx_t) - \eta_j(\vecx_t))\pred_j \\
(\eta_j(\vecx_t) - \empeta_{tj}(\vecx_t))(1-\pred_j) & (\eta_j(\vecx_t) - \empeta_{tj}(\vecx_t))\pred_j
\end{pmatrix},
\]
so that $\|\bm \Delta_t\|^2 = 2 \sum_{j=1}^m |\eta_j(\vecx_t) - \empeta_{tj}(\vecx_t)|^2 =
2 \|\veceta(\vecx_t) - \empveceta_t(\vecx_t)\|^2$. Using $b=1$ for multi-class, and $b=\sqrt{2}$ for multi-label setting, we finally get:
\[
\widetilde{F}(\widetilde{\vecy}_t) - \widetilde{F}(\vecpred_t) \le 2 b L \|\veceta(\vecx_t) - \empveceta_t(\vecx_t)\|,
\]
or using the definition of $\widetilde{F}$ and rearranging,
\begin{equation}
\nabla \psi(\matC_{t-1}) \cdot \left(\EE_{\vecy_t \sim \veceta_t(\vecx_t)}  \left[ \matC(\vecy_t,\vecpred_t)  \right]\right)
\ge \max_{\vecpred} \left\{\nabla \psi(\matC_{t-1}) \cdot \left(\EE_{\vecy_t \sim \veceta_t(\vecx_t)}  \left[ \matC(\vecy_t,\vecpred)  \right]\right)\right\}
- 2 b L \|\veceta(\vecx_t) - \empveceta_t(\vecx_t)\|.
\label{eq:widetildeF}
\end{equation}
Thus, even thought the algorithm uses the estimator $\empveceta_t$ in place of $\veceta$,
the prediction it picks (left-hand side) is suboptimal with respect to the choice it should have made had it known true $\veceta$ (first term on the right-hand side)
by at most the estimation error of $\empveceta_t$ (second term on the right-hand side).
We will use this result later in the proof.

We now proceed with the main part of the proof. From the concavity of the utility,
\[
\psi(\matC^{\star}) \le \psi(\matC_{t-1}) + \nabla \psi(\matC_{t-1}) \cdot (\matC^{\star} - \matC_{t-1}),
\]
or, after rearranging,
\begin{equation}
\nabla \psi(\matC_{t-1}) \cdot (\matC_{t-1}-\matC^{\star}) \le \psi(\matC_{t-1}) - \psi(\matC^{\star}).
\label{eq:from_concavity}
\end{equation}
It follows from the concavity of $\psi$ and its smoothness that (see, e.g., \citet{bubeck2014convex}) for any $\matC_1, \matC_2 \in \mathcal{C}$,
\[
\psi(\matC_1) \ge \psi(\matC_2) + \nabla \psi(\matC_2) \cdot (\matC_1 - \matC_2) - \frac{M}{2} \|\matC_1 - \matC_2\|^2.
\]
Substituting $\matC_1 = \matC_t$ and $\matC_2 = \matC_{t-1}$, we get:
\begin{align*}
\psi(\matC_t) &\ge \psi(\matC_{t-1}) + \nabla \psi(\matC_{t-1}) \cdot (\matC_t - \matC_{t-1}) - \frac{M}{2} \|\matC_t - \matC_{t-1}\|_1^2 \\
&= \psi(\matC_{t-1}) + \frac{1}{t} \nabla \psi(\matC_{t-1}) \cdot (\matC(\vecy_t, \vecpred_t) - \matC_{t-1}) - 
\frac{M}{2t^2} \|\matC(\vecy_t, \vecpred_t) - \matC_{t-1}\|^2
\end{align*}
where we used \eqref{eq:confusion_matrix} to get:
\[
\matC_t = \frac{t-1}{t} \matC_{t-1} + \frac{1}{t} \matC(\vecy_t, \vecpred_t) = \matC_{t-1} + \frac{1}{t} (\matC(\vecy_t, \vecpred_t) - \matC_{t-1}).
\]
We upper-bound 
\[
\|\matC(\vecy_t, \vecpred_t) - \matC_{t-1}\|^2
\le \max_{\matC_1,\matC_2 \in \confusionmatrixspace} \|\matC_1 - \matC_2\|^2.
\]
We have
\[
\|\matC_1 - \matC_2\|^2 = \|\matC_1\|^2  - 2 \matC_1 \cdot \matC_2 + \|\matC_2\|^2
\le \|\matC_1\|^2 + \|\matC_2\|^2 \le \|\matC_1\|_1 + \|\matC_2\|_1,
\]
where the first inequality follows from the fact that confusion matrix has nonnegative entries,
while the second inequality follows from the fact that all entries are at most 1 (so that the sum of squares of the entries
is not larger than the sum of their absolute values). For multi-class classification $\|\matC\|_1 = 1$ for any $\matC \in \confusionmatrixspace$, while for multi-label
classification, $\|\matC\|_1 = m$ for any $\matC \in \confusionmatrixspace$. Thus, using $a = 1$ for multi-class, and $a = m$
for multi-label classification, we have:
\[
\psi(\matC_t) \ge \psi(\matC_{t-1}) + \frac{1}{t} \nabla \psi(\matC_{t-1}) \cdot (\matC(\vecy_t, \vecpred_t) - \matC_{t-1}) - 
\frac{M a}{t^2}.
\]
Multiplying both sides of the inequality by $(-1)$ followed by adding $\psi(\matC^{\star})$ gives:
\begin{align*}
\psi(\matC^{\star}) - \psi(\matC_t)
&\le \psi(\matC^{\star}) - \psi(\matC_{t-1})
+ \frac{1}{t} \nabla \psi(\matC_{t-1}) \cdot (\matC_{t-1} - \matC(\vecy_t, \vecpred_t)) + 
\frac{M a}{t^2} \\
&= \psi(\matC^{\star}) - \psi(\matC_{t-1})
+ \frac{1}{t} \nabla \psi(\matC_{t-1}) \cdot (\matC_{t-1} - \matC^{\star})
+ \frac{1}{t} \nabla \psi(\matC_{t-1}) \cdot (\matC^{\star} - \matC(\vecy_t, \vecpred_t)) + 
\frac{M a}{t^2} \\
&\le \frac{t-1}{t} \left(\psi(\matC^{\star}) - \psi(\matC_{t-1})\right)
+ \frac{1}{t} \nabla \psi(\matC_{t-1}) \cdot (\matC^{\star} - \matC(\vecy_t, \vecpred_t)) + 
\frac{M a}{t^2} 
\end{align*}
where in the last inequality we used \eqref{eq:from_concavity}. Multiplying both sides by $t$ and denoting
$A_t = t(\psi(\matC^{\star}) - \psi(\matC_t))$, we get
\[
A_t \le A_{t-1} + \nabla \psi(\matC_{t-1}) \cdot (\matC^{\star} - \matC(\vecy_t, \vecpred_t)) + 
\frac{M a}{t}.
\]
Now, let $\EE_{(\vecx_t,\vecy_t)}[\cdot]$ denote the 
expectation with respect to $(\vecx_t,\vecy_t)$
(formally: conditioned on all the past $(\vecx^{t-1},\vecy^{t-1})$). Since $\matC_{t-1}$ does not depend on $(\vecx_t,\vecy_t)$,
$\EE_{(\vecx_t,\vecy_t)}[\matC_{t-1}] = \matC_{t-1}$ (and similarly for $A_{t-1}$). Applying this expectation to both sides of the inequality above
gives:
\begin{equation}
\EE_{(\vecx_t,\vecy_t)}[A_t] \le A_{t-1} + \EE_{(\vecx_t,\vecy_t)} \left[\nabla \psi(\matC_{t-1}) \cdot (\matC^{\star} - \matC(\vecy_t, \vecpred_t))\right]
+ \frac{M a}{t}.
\label{eq:regret_bound_intermediate_1}
\end{equation}
We proceed with further bounding the middle term on the right-hand side of \eqref{eq:regret_bound_intermediate_1}. Using the rule of conditional
expectation, $\EE_{(\vecx_t,\vecy_t)}[\cdot] = \EE_{\vecx_t}[\EE_{\vecy_t \sim \veceta(\vecx_t)}[\cdot]]$,
\begin{align*}
\EE_{(\vecx_t,\vecy_t)} \left[\nabla \psi(\matC_{t-1}) \cdot \matC(\vecy_t, \vecpred_t)\right]
&= \EE_{\vecx_t} \left[\nabla \psi(\matC_{t-1}) \cdot \EE_{\vecy_t \sim \veceta(\vecx_t)}\left[\matC(\vecy_t, \vecpred_t)\right]\right] \\
\text{(from \eqref{eq:widetildeF})}\quad &\ge
\EE_{\vecx_t} \left[\max_{\vecpred} \left\{\nabla \psi(\matC_{t-1}) \cdot \EE_{\vecy_t \sim \veceta(\vecx_t)}\left[\matC(\vecy_t, \vecpred)\right] \right\}\right] - 2 b L \EE_{\vecx_t} \left[ \|\veceta(\vecx_t) - \empveceta_t(\vecx_t)\| \right],
\end{align*}
or, after rearranging,
\begin{equation}
\EE_{\vecx_t} \left[\max_{\vecpred} \left\{\nabla \psi(\matC_{t-1}) \cdot \EE_{\vecy_t \sim \veceta(\vecx_t)}\left[\matC(\vecy_t, \vecpred)\right] \right\}\right] \le \EE_{(\vecx_t,\vecy_t)} \left[\nabla \psi(\matC_{t-1}) \cdot \matC(\vecy_t, \vecpred_t)\right] + 
2 b L \EE_{\vecx_t} \left[ \|\veceta(\vecx_t) - \empveceta_t(\vecx_t)\| \right]
\label{eq:regret_bound_intermediate_2}
\end{equation}
Using the fact that neither $\matC_{t-1}$ nor $\matC^{\star}$ depend on $(\vecx_t,\vecy_t)$ (here we employ the fact that $\matC^{\star}$
is the population optimal matrix, which does not depend on any data), we rewrite:
\begin{align*}
\EE_{(\vecx_t,\vecy_t)} \left[\nabla \psi(\matC_{t-1}) \cdot \matC^{\star} \right] &=
\nabla \psi(\matC_{t-1}) \cdot \matC^{\star} \\
(\text{from \eqref{eq:C_star}})\quad &=
\nabla \psi(\matC_{t-1}) \cdot \EE_{(\vecx,\vecy)}[\matC(\vecy,\vech^{\star}(\vecx)] \\
(\text{$(\vecx,\vecy)$ has the same distr. as $(\vecx_t,\vecy_t)$})\quad &=
\nabla \psi(\matC_{t-1}) \cdot \EE_{(\vecx_t,\vecy_t)}[\matC(\vecy_t,\vech^{\star}(\vecx_t)] \\
(\text{$\matC_{t-1}$ does not depend on $(\vecx_t,\vecy_t)$})\quad
&= \EE_{\vecx_t} \left[\nabla \psi(\matC_{t-1}) \cdot \EE_{\vecy_t \sim \veceta(\vecx_t)}\left[ \matC(\vecy_t,\vech^{\star}(\vecx_t) \right] \right] \\
&\le \EE_{\vecx_t} \left[\max_{\vecpred} \left\{\nabla \psi(\matC_{t-1}) \cdot \EE_{\vecy_t \sim \veceta(\vecx_t)}\left[\matC(\vecy_t, \vecpred)\right] \right\}\right] \\
(\text{from \eqref{eq:regret_bound_intermediate_2}})\quad &\le
\EE_{(\vecx_t,\vecy_t)} \left[\nabla \psi(\matC_{t-1}) \cdot \matC(\vecy_t, \vecpred_t)\right]
+ 2 b L \EE_{\vecx_t} \left[ \|\veceta(\vecx_t) - \empveceta_t(\vecx_t)\| \right].
\end{align*}
Thus, the middle term on the right-hand side of \eqref{eq:regret_bound_intermediate_1} can be bounded by
\[
\EE_{(\vecx_t,\vecy_t)} \left[\nabla \psi(\matC_{t-1}) \cdot (\matC^{\star} - \matC(\vecy_t, \vecpred_t))\right]
\le 2 b L \EE_{\vecx_t} \left[ \|\veceta(\vecx_t) - \empveceta_t(\vecx_t)\| \right],
\]
which, after substituting to \eqref{eq:regret_bound_intermediate_1} and taking the expectation on both sides with respect
to all the past $(\vecx^{t-1},\vecy^{t-1})$,
\[
\EE[A_t] \le \EE[A_{t-1}] + 2 b L \EE\left[ \|\veceta(\vecx_t) - \empveceta_t(\vecx_t)\| \right]
+ \frac{M a}{t}.
\]
Applying the above to $t=n$, and recursively expanding,
\begin{align}
\EE[A_n] &\le \EE[A_{n-1}] + 2 b L \EE\left[ \|\veceta(\vecx_n) - \empveceta_n(\vecx_n)\| \right]
+ \frac{M a}{n} \nonumber \\
&\le \EE[A_{n-2}] + 2 b L \EE\left[ \|\veceta(\vecx_{n-1}) - \empveceta_{n-1}(\vecx_{n-1})\| \right]
+ 2 b L \EE\left[ \|\veceta(\vecx_n) - \empveceta_n(\vecx_n)\| \right]
+ \frac{M a}{n} + \frac{M a}{n-1} \nonumber \\
&\le \ldots \le \EE[A_0] + 2 b L \sum_{t=1}^n \EE\left[ \|\veceta(\vecx_t) - \empveceta_t(\vecx_t)\| \right]
+ M a\underbrace{\sum_{t=1}^n \frac{1}{t}}_{H_n} \nonumber \\
&\le \EE[A_0] + 2 b L \sum_{t=1}^n \EE\left[ \|\veceta(\vecx_t) - \empveceta_t(\vecx_t)\| \right]
+ M a (1 + \ln n),
\label{eq:regret_bound_intermediate_3}
\end{align}
where we bounded the harmonic number $H_n=\sum_{t=1}^n t^{-1}$ by $1 + \ln n$.
Recall that $A_t = t(\psi(\matC^{\star}) - \psi(\matC_t))$, which gives $A_0=0$ (our recursion
\eqref{eq:regret_bound_intermediate_1} was also valid for $t=1$). 
After dividing both sides of \eqref{eq:regret_bound_intermediate_3} by $n$ and using \eqref{eq:regret_upperbound_with_population_optimal_cm}
we get
\[
\regret_n \le \frac{M a (1+ \ln n)}{n} + \frac{2 b L}{n} \sum_{t=1}^n \EE\left[ \|\veceta(\vecx_t) - \empveceta_t(\vecx_t)\| \right],
\]
which was to be shown. \qed

\subsection{Non-smooth utilities}
\label{app:non-smooth}

While all commonly used utilities are smooth functions of confusion matrix, 
it is possible to modify the \OMMA{} algorithm so that it achieves 
vanishing regret even for a non-smooth utility $\psi$, 
but at a price of a slower convergence rate $O(\sqrt{(\ln n) / n})$. 
For the sake of clarity, here we give a regret bound for the case when the time horizon $n$ is known ahead of time.

The modification of the algorithm is based on optimizing a \emph{smoothed} version of $\psi$.

A collection of convex functions $\{\psi_{\gamma} \colon \gamma > 0 \}$ is called a $\gamma$-smoothing of a convex and closed function $\psi$ with parameters 
$(\alpha, \beta)$ if $0 \le \psi(\matC) - \psi_{\gamma}(\matC) \le \alpha \gamma$ for all $\matC \in \confusionmatrixspace$, and $\psi_{\gamma}$ is $\frac{\beta}{\gamma}$-smooth \citep{BeckTeboulle2012}. That is,
$\gamma$ dictates a trade-off between the smoothness of $\psi_{\gamma}$ and its deviation from the original function $\psi$. 

Running the \OMMA{} algorithm
with $\psi$ replaced by its $\frac{\beta}{\gamma}$-smoothed (concave) $\psi_{\gamma}$ results in the regret bound from Theorem \ref{thm:regret}:
\[
\regret^{\gamma}_n \le 
\frac{\beta a (1+ \ln n)}{\gamma n} + \frac{2 b L_{\gamma}}{n} \sum_{t=1}^n \EE\left[ \|\veceta(\vecx_t) - \empveceta_t(\vecx_t)\| \right]
\]
where $\regret^{\gamma}_n = \EE\left[\psi_{\gamma}(\matC(\vecy^n,\vech^{\star}(\vecx^n)))\right]  - \EE\left[\psi_{\gamma}(\matC(\vecy^n,\vecpred^n))\right]$
is the regret defined in terms of $\psi_{\gamma}$ rather than the original function $\psi$, while $L_{\gamma}$ is the Lipschitz constant
of $\psi_{\gamma}$.
Using the property of $\gamma$-smoothing
$0 \le \psi(\matC) - \psi_{\gamma}(\matC) \le \alpha \gamma$
twice to $\matC=\matC(\vecy^n,\vecpred^n)$ and $\matC=\matC(\vecy^n,\vech^{\star}(\vecx^n))$ we relate the regret with respect to the original
utility with that of its smoothed version:
\[
\regret_n = \EE\left[\psi(\matC(\vecy^n,\vech^{\star}(\vecx^n)))\right]  - \EE\left[\psi(\matC(\vecy^n,\vecpred^n))\right]
\le \EE\left[\psi_{\gamma}(\matC(\vecy^n,\vech^{\star}(\vecx^n)))\right]  - \EE\left[\psi_{\gamma}(\matC(\vecy^n,\vecpred^n))\right]
+ \alpha \gamma = \regret^{\gamma}_n + \alpha \gamma.
\]
Thus the resulting bound on the $\psi$-regret of the algorithm is:
\[
\regret_n \le
\frac{\beta a (1+ \ln n)}{\gamma n} + \frac{2 b L_{\gamma}}{n} \sum_{t=1}^n \EE\left[ \|\veceta(\vecx_t) - \empveceta_t(\vecx_t)\| \right]
+ \alpha \gamma.
\]
The optimal choice of $\gamma$ (i.e., the one minimizing the right-hand side of the bound) is
$\gamma = \sqrt{\frac{a \beta (1+\ln n)}{\alpha n}}$, which results in:
\begin{equation}
\regret_n \le 2 \sqrt{\frac{a \alpha \beta (1+\ln n)}{n}}
+ \frac{2 b L_{\gamma}}{n} \sum_{t=2}^n \EE\left[ \|\veceta(\vecx_t) - \empveceta_t(\vecx_t)\| \right].
\label{eq:regret_bound_non_smooth}
\end{equation}
An example of $\gamma$-smoothing is the \emph{infimal convolution smoothing} with Moreau-Yoshida regularization \citep{Beck2017}:
\[
\psi_{\gamma}(\matC) = \max_{\matC'} \left\{ \psi(\matC') - \frac{1}{2\gamma} \|\matC - \matC'\|^2\right\},
\]
which has parameters $\alpha = \frac{L^2}{2}$, $\beta = 1$, and $L_{\gamma} = L$, where $L$ is the Lipschitz constant of $\psi$. Using it
in the algorithm, the bound \eqref{eq:regret_bound_non_smooth} becomes:
\[
\regret_n \le L \sqrt{\frac{2a (1+\ln n)}{n}}
+ \frac{2 b L}{n} \sum_{t=2}^n \EE\left[ \|\veceta(\vecx_t) - \empveceta_t(\vecx_t)\| \right].
\]

\section{The adversarial setting is too difficult}
\label{sec:adversarial}

Consider the following variation of the problem. At trials $t=1,\ldots,n$:
\begin{itemize}
    \item The algorithm receives input instance $\vecx_t$ together with its conditional label probability vector
    $\veceta_t \equiv \veceta(\vecx_t)$
    \item The algorithm issues prediction $\vecpred_t$
    \item The label $\vecy_t$ is drawn from $\veceta_t$ and revealed to the algorithm
\end{itemize}
Note that we allow the algorithm to know \emph{exact} conditional probabilities (rather than only their estimate $\empveceta_t$), but
we do not make any assumptions about the sequence of inputs $\vecx^n$ (in particular, we do not assume they are generated i.i.d.).
At the end of the sequence, the algorithm evaluated by the expected utility, $\EE_{\vecy^n \sim \veceta^n} \psi(\vecy^n, \vecpred^n)$,
and is compared with the set of predictions maximizing the expected utility:
\begin{equation}
\optvecy{}^n = \argmax_{\vecpred^n} \EE_{\vecy^n \sim \veceta^n} \left[ \psi(\vecy^n, \vecpred^n) \right].
\label{eq:optimal_adversary}
\end{equation}
The regret is given by:
\[
R_n = \EE\left[ \psi(\vecy^n, \optvecy{}^n) \right] - \EE \left[ \psi(\vecy^n, \vecpred^n)\right],
\]
where we dropped the subscript in expectation as clear from the context.
We will show that in this setting, any algorithm has a regret which is at least $\frac{1}{36} - o(1)$.
To this end, take a binary classification setting ($m=2$). As in the example in Section \ref{sec:algorithm},
to simplify the presentation we turn the one-hot vector notation into scalars, and repeat the translation here
for the sake of convenience.
Let $y_t \in \{0,1\}$ denote a label at time $t$ (where $y_t=0$ corresponds to
$\vecy_t=(1,0)$ in the one-hot notation, while $y_t=1$ to $\vecy_t=(0,1)$). Similarly,
let $\pred_t \in \{0,1\}$ denote the prediction.
The conditional probability is $\eta_t = P(y_t =1 | \vecx_t)$
(which corresponds to $\eta_1(\vecx_t)$ in the vector notation).
The confusion matrix $\matC_t := \matC(\vecy^t,\vecpred^t)$
consists of four entries $C_{t-1,j\ell} = \sum_{i \le t} \bm{1}\{y_i=j\}\bm{1}\{\pred_i=\ell\}$,
$j,\ell \in \{0,1\}$ (true negatives, false positives, false negatives, true positives).

We will consider the following utility:
\[
\psi(\matC) = \min\{C_{00}, C_{11}\},
\]
which is concave and 1-Lipschitz. While $\psi$ is non-smooth, we can show a similar bound
with a smoothed version of it, however, it would require much more involved analysis.
Given predictions $\pred^n$, 
the expected value (with respect to $y^n$)
of $n C_{n,11}$ can be calculated as:
\[
n \EE[C_{n,11}] = (n-1) \EE[C_{n-1,11}] + 
\EE[\pred_t \EE_{y_t}[y_t]] = (n-1) \EE[C_{n-1,11}] + \EE[\pred_t\eta_t] = \ldots = \EE \left[\sum_t \pred_t \eta_t\right],
\]
where we used the fact that $\pred_t$ does not depend on $y_t$. A similar method can be applied to $n C_{n,00}$. Thus, we get:
\begin{equation}
\EE[C_{n,11}] = \EE \left[\frac{1}{n}\sum_{t=1}^n \pred_t \eta_t\right], \qquad
\EE[C_{n,00}] = \EE \left[\frac{1}{n}\sum_{t=1}^n (1-\pred_t) (1-\eta_t) \right].
\label{eq:expected_conf_matrix_elements}
\end{equation}
Now, as $\psi$ is concave, we have 
\begin{equation}
\EE\left[\psi(\matC_n)\right] \le \psi\left(\EE\left[\matC_n\right]\right),
\label{eq:concavity_adversarial}
\end{equation}
and the regret can be lower-bounded by:
\[
R_n \ge \EE\left[ \psi(y^n, \opty{}^n) \right] - \psi\left(\EE\left[\matC_n\right]\right).
\]
We now turn into defining a sequence of outcomes.
Since the only property of inputs $\vecx_t$ used by the algorithm (and appearing in the regret)
is the conditional probability $\eta_t$, we will directly operate on sequences of $\eta_t$'s. Let $n=2m$ be even,
and suppose the algorithm observes for the first half of the game $\eta_1=\eta_2 = \ldots = \eta_m = \frac{2}{3}$.
In the second half of the game, the algorithm will observe one of the two sequences:
\begin{enumerate}
    \item either $\eta_{m+1} = \eta_{m+2} = \ldots = \eta_n = \frac{2}{3}$ (sequence $\bm{1}$),
    \item or $\eta_{m+1} = \eta_{m+2} = \ldots = \eta_n = \frac{1}{3}$ (sequence $\bm{2}$).
\end{enumerate}
We will exploit the fact that the algorithm up to iteration $m$ does not know which sequence it should prepare for, while the optimal predictions
for both sequences are substantially different.
Let $S_1 = \EE\left[m^{-1} \sum_{t=1}^m \pred_t\right]$ and $S_2 = \EE\left[m^{-1} \sum_{t=m+1}^n \pred_t\right]$ be the expected average
values of the predictions of the algorithm in the first half and second half of the sequence, respectively. Using \eqref{eq:expected_conf_matrix_elements},
and the fact the first $m$ elements of the sequence are all equal to $\eta_1$, while the last $m$ elements are all equal to $\eta_n$, we have
\[
\EE[C_{n,11}] = \frac{1}{n} m S_1 \eta_1 + m S_2 \eta_n = \frac{S_1\eta_1 + S_2\eta_n}{2},
\]
and similarly $\EE[C_{n,10}] = \frac{(1-S_1)(1-\eta_1) + (1-S_2) (1-\eta_n)}{2}$. Using \eqref{eq:concavity_adversarial}, 
the performance of the algorithm for
sequence $\bm{1}$ ($\eta_1 = \eta_n = \frac{2}{3}$) is at most:
\begin{equation}
\text{(sequence $\bm{1}$)}\quad\EE\left[\psi(\matC_n)\right] \le
\min\left\{\frac{1}{3}(S_1 + S_2), \frac{1}{6}(2-S_1-S_2)\right\},
\label{eq:first_sequence_perf_alg_1}
\end{equation}
while for sequence $\bm{2}$ ($\eta_1 = \frac{2}{3}, \eta_n = \frac{1}{3}$) it is at most:
\begin{equation}
\text{(sequence $\bm{2}$)}\quad\EE\left[\psi(\matC_n)\right] \le
 \min\left\{\frac{1}{6}(2 S_1 + S_2), \frac{1}{6}(3-S_1- 2 S_2)\right\}
\label{eq:second_sequence_perf_alg_1}
\end{equation}
Note that we must have $0 \le S_1, S_2 \le 1$ as these corresponds to expected averages of predictions of the algorithm.
Furthermore $S_2$ can be optimally chosen by the algorithm \emph{separately} for case $\bm{1}$ and $\bm{2}$, as it corresponds
to the second half of predictions, 
when the algorithm already knows whether the sequence is of type $\bm{1}$ or $\bm{2}$. 
In both cases, the optimal choice (maximizing the right-hand sides of \eqref{eq:first_sequence_perf_alg_1}-\eqref{eq:second_sequence_perf_alg_1}) is to
make both terms under min equal. For sequence $\bm{2}$ it is always possible:
\[
\frac{1}{6}(2 S_1 + S_2) =  \frac{1}{6}(3-S_1- 2 S_2) \quad \Rightarrow \quad S_2 = 1-S_1,
\]
which leads to a bound:
\begin{equation}
\text{(sequence $\bm{2}$)}\quad\EE\left[\psi(\matC_n)\right] \le \frac{1}{6}(1+S_1)
\label{eq:second_sequence_perf_alg_2}
\end{equation}
For sequence $\bm{1}$ we get:
\[
\frac{1}{3}(S_1 + S_2) = \frac{1}{6}(2-S_1-S_2) \quad \Rightarrow \quad S_2 = \frac{2}{3} - S_1,
\]
which is only possible when $S_1 \le \frac{2}{3}$. Thus, we set $S_2 = \max\left\{0, \frac{2}{3} - S_1\right\}$, which gives
\begin{equation}
\text{(sequence $\bm{1}$)}\quad\EE\left[\psi(\matC_n)\right] \le
\min\left\{\frac{2}{9}, \frac{2-S_1}{6}\right\},
\label{eq:first_sequence_perf_alg_2}
\end{equation}

We will lower bound the value of the optimal utility by simply picking a particular set of predictions $\opty{}^n$ rather than
using \eqref{eq:optimal_adversary}. Assume $n$ is divisible by $3$. For sequence $\bm{1}$, set $\opty_t=1$ in the first one-third of
the sequence, while $\opty_t = 0$ in the last two-thirds of the sequence. Let $Y_1 = \sum_{t=1}^{n/3} y_i$ and $Y_2 = \sum_{t=n/3}^n y_i$.
We get $C_{11}^{\star} := C_{11}(y^n, \opty{}^n) = Y_1/n$ and $C_{00}^{\star} = Y_2/n$.
Note that $Y_1 \sim \mathrm{Binomial}(n/3, 2/3)$, $Y_2 \sim \mathrm{Binomial}(2n/3, 1/3)$, both having the same expectation
$\EE[Y_1]=\EE[Y_2]=\frac{2n}{9}$. Using the fact that $Y_i - \EE[Y_i]$, $i=1,2$, are subgaussian (as being binomial), we have
for any $s \in \mathbb{R}$:
\[
\EE\left[e^{s (Y_1 - \EE[Y_1])}\right] \le e^{\frac{n s^2}{24}}, \qquad
\EE\left[e^{s (Y_2 - \EE[Y_2])}\right] \le e^{\frac{n s^2}{12}},
\]
which we can plug into a standard derivation of a bound for the minimum of subgaussian variables. First note that:
\[
\EE\left[\min\{Y_1,Y_2\}\right] = \frac{2n}{9} - \EE\left[\max\{\mathbb{E}[Y_1] - Y_1, \mathbb{E}[Y_1] - Y_2\} \right].
\]
To bound the maximum, we have for any $s > 0$:
\begin{align*}
\EE\left[\max\{\mathbb{E}[Y_1] - Y_1, \mathbb{E}[Y_1] - Y_2\} \right]
&= \frac{1}{s} \EE\left[\ln e^{s \max\{\mathbb{E}[Y_1] - Y_1, \mathbb{E}[Y_1] - Y_2\}} \right] \\
\text{(Jensen's ineq.)} \quad &\le \frac{1}{s} \ln\EE\left[ e^{s \max\{\mathbb{E}[Y_1] - Y_1, \mathbb{E}[Y_1] - Y_2\}} \right] 
\le \frac{1}{s} \ln\EE\left[ e^{s (\mathbb{E}[Y_1] - Y_1)} + e^{s (\mathbb{E}[Y_1] - Y_1)} \right] \\
&\le \frac{1}{s} \ln \left(e^{\frac{n s^2}{24}} + e^{\frac{n s^2}{12}}\right)
\le \frac{1}{s} \ln \left(2e^{\frac{n s^2}{12}}\right) = \frac{\ln2}{s} + \frac{ns}{12} \\
\left(s=\sqrt{\frac{12 \ln2 }{n}}\right) \quad &=2 \sqrt{\frac{\ln 2}{12 n}} \le \frac{\sqrt{n}}{2}.
\end{align*}
Thus, $\EE\left[\min\{Y_1,Y_2\}\right] \ge \frac{2n}{9} - \frac{\sqrt{n}}{2}$, and 
$\EE\left[\psi(\matC^{\star})\right] \ge \frac{2}{9} - \frac{1}{2\sqrt{n}}$.

For sequence $\bm{2}$ we set $\opty_t=1$ in the first half of the sequence and $\opty_t=0$ in the second half of the sequence.
We get $C_{11}^{\star} := Y_1/n$ and $C_{00}^{\star} = Y_2/n$, where $Y_1 \sim \mathrm{Binomial}(n/2, 2/3)$, $Y_2 \sim \mathrm{Binomial}(n/2, 2/3)$
with $\EE[Y_1]=\EE[Y_2]=\frac{n}{3}$. Proceeding similarly as in the previous case, we get
$\EE\left[\psi(\matC^{\star})\right] \ge \frac{1}{3} - \frac{1}{2\sqrt{n}}$.

Using \eqref{eq:first_sequence_perf_alg_2} and \eqref{eq:second_sequence_perf_alg_2}, the regret of the algorithm can
be lower-bounded for the sequence $\bm{1}$ and $\bm{2}$, respectively, as:
\begin{align*}
\regret_{\bm{1}} &\ge \frac{2}{9} - \min\left\{\frac{2}{9}, \frac{2-S_1}{6}\right\} - \frac{1}{2\sqrt{n}}
= \max\left\{0, \frac{S_1}{6} - \frac{1}{9} \right\} - \frac{1}{2\sqrt{n}}, \\
\regret_{\bm{2}} &\ge \frac{1}{3} - \frac{1}{6}(1+S_1) - \frac{1}{2\sqrt{n}}
= \frac{1-S_1}{6} - \frac{1}{2\sqrt{n}}.
\end{align*}
Now, the crucial point is that the algorithm \emph{cannot} tune $S_1$ separately for both cases, as it does not know
which of the two sequences will occur. Since the adversary will select a sequence which incurs more regret to the algorithm,
the best choice of $S_1$ is to make the regret (bound) equal in both cases. Clearly, the algorithm will choose $S_1 \ge \frac{2}{3}$,
as any smaller value does not change the regret bound for sequence $\bm{1}$ (it is already zero), while it only increases the regret bound
for sequence $\bm{2}$. For $S_1 \ge \frac{2}{3}$, we can equalize the bounds in both cases by setting:
\[
\frac{S_1}{6} - \frac{1}{9} = \frac{1-S_1}{6} \quad \Rightarrow \quad S_1 = \frac{5}{6},
\]
which gives the worse-case regret at least $\frac{1}{36} - \frac{1}{2\sqrt{n}}$. Thus, no matter how large $n$ is, the regret
does not decrease to zero.

\section{\OMMAeta{} algorithm}
\label{app:OMMA_eta}

In this section, we prove a regret bound for \OMMAeta{} algorithm defined
as Algorithm \ref{alg:OMMA_eta}, in the case when the CPE is exact, that
is $\empveceta_t \equiv \veceta$ for all $t$. 
The bound is of the same order as that of \OMMA{} in Theorem \ref{thm:regret}
under the same assumptions. We assume familiarity of the reader
with the proof of Theorem \ref{thm:regret} in Section \ref{app:regret_bounds}.

The algorithm updates its parameter matrix according to
\[
\matC_t = \frac{t-1}{t} \matC_{t-1} + \frac{1}{t}\matC(\veceta_t,\vecpred_t),
\]
with prediction give by
\begin{equation}
\vecpred_t = \argmax_{\vecpred} \nabla \psi(\matC_{t-1}) \cdot 
     \matC(\veceta_t,\vecpred),
\label{eq:OMMA_eta_prediction}
\end{equation}
where $\matC(\veceta_t,\vecpred_t) = \EE_{\vecy_t \sim \veceta_t} \left[\matC(\vecy_t, \vecpred_t)\right]$,
and we abbreviate $\veceta_t = \veceta_t(\vecx_t)$. Note that neither algorithm's predictions $\vecpred_t$
nor its parameters $\matC_t$ depend on the labels.
We remind that $\matC_n = \matC(\textcolor{black!40!red}{\veceta^n}, \vecpred^n)$ \emph{does not correspond}
to the confusion matrix of the algorithm $\matC(\textcolor{black!40!red}{\vecy^n}, \vecpred^n)$, as it does not contain labels.
We can however interpret $\matC_n$ as the \emph{expected algorithm's confusion matrix} $\matC(\vecy^n, \vecpred^n)$
with labels $\vecy_t$ drawn from $\datadistribution(\cdot | \vecx_t)$, as:
\[
\matC_n = \sum_{t=1}^n \matC(\veceta_t,\vecpred_t)
= \sum_{t=1}^n \EE_{\vecy_t | \vecx_t} [\matC(\vecy_t,\vecpred_t)]
= \EE_{\vecy^n | \vecx^n} \left[\sum_{t=1}^n \matC(\vecy_t,\vecpred_t) \right]
= \EE_{\vecy^n | \vecx^n} \left[ \matC(\vecy^n, \vecpred^n) \right].
\]
Using the smoothness of $\psi$, for any $\vecy^n$,
we have \citep{bubeck2014convex}
\[
\psi(\matC(\vecy^n, \vecpred^n)) \ge \psi(\matC_n) + \nabla \psi(\matC_n) \cdot (\matC(\vecy^n, \vecpred^n) - \matC_n)
- \frac{M}{2} \|\matC(\vecy^n, \vecpred^n) - \psi(\matC_n)\|^2,
\]
By taking expectation over $\vecy^n | \vecx^n$ on both sides and using the fact proven above that
$\matC_n = \EE_{\vecy^n | \vecx^n} \left[ \matC(\vecy^n, \vecpred^n) \right]$, the gradient
term on the right-hand side disappears, while the quadratic term becomes the total 
\emph{variance} of $\matC(\vecy^n, \vecpred^n)$:
\begin{equation}
\EE_{\vecy^n | \vecx^n} [\psi(\matC(\vecy^n, \vecpred^n))]
\ge \psi(\matC_n) - \frac{L}{2} \mathrm{Var}_{\vecy^n | \vecx^n}\left(\matC(\vecy^n, \vecpred^n)\right),
\label{eq:OMMA_eta_proof_eq_1}
\end{equation}
where we define $\mathrm{Var}(\matC) = \EE[\|\matC - \EE[\matC]\|^2]$ to be the total variance of all entries of $\matC$. Since
$\matC(\vecy^n, \vecpred^n) = \frac{1}{n} \sum_{t=1}^n \matC(\vecy_t, \vecpred_t)$, with each $\vecy_t$
independent of each other we have, from the scaling property of variance and its additivity for independent random variables,
\[
\mathrm{Var}_{\vecy^n | \vecx^n}\left(\matC(\vecy^n, \vecpred^n)\right)
= \frac{1}{n^2} \sum_{t=1}^n \mathrm{Var}_{\vecy^t | \vecx^t} \left(\matC(\vecy_t, \vecpred_t)\right)
\le \frac{2 a}{n},
\]
where we bounded each term in the sum by $\max_{\matC_1, \matC_2 \in \confusionmatrixspace} \|\matC_1 - \matC_2\|^2 \le 2 a$,
where $a = 1$ for multi-class classification and $a=m$ for multi-label classification (see the proof of Theorem \ref{thm:regret} in Appendix
\ref{app:regret_bounds}). Using the bound on the variance in \eqref{eq:OMMA_eta_proof_eq_1} and taking expectation with respect 
to $\vecx^n$ on both sides gives
\begin{equation}
\EE[\psi(\matC(\vecy^n, \vecpred^n))] \ge \EE_{\vecx^n}[\psi(\matC_n)] - \frac{Ma}{n}
\label{eq:OMMA_eta_proof_eq_2}
\end{equation}
Let $\vech^{\star}$ be the optimal classifier and let 
\begin{equation}
\matC^{\star} = \EE_{(\vecx, \vecy)}\left[ \matC(\vecy,\vech^{\star}(\vecx)) \right]
= \EE_{\vecx} \left[\matC(\veceta(\vecx), \vech^{\star}(\vecx)) \right]
\label{eq:C_star_OMMA_eta}
\end{equation}
be its population confusion matrix. From the concavity of $\psi$, 
$\EE \left[ \psi(\matC(\vecy^n,\vech^{\star}(\vecx^n))) \right]
\le \psi\left(\EE\left[ \matC(\vecy^n,\vech^{\star}(\vecx^n)) \right] \right)
= \psi\left(\matC^{\star} \right)$, which, together with \eqref{eq:OMMA_eta_proof_eq_2}
gives an upper bound on the regret:
\begin{equation}
\regret_n
= \EE \left[ \psi(\matC(\vecy^n,\vech^{\star}(\vecx^n))) \right] - \EE[\psi(\matC(\vecy^n, \vecpred^n))]
\le \underbrace{\psi(\matC^{\star}) - \EE_{\vecx^n}[\psi(\matC_n)]}_{=: Q_n} + \frac{Ma}{n}.
\label{eq:regret_bound_OMMA_eta_Q}
\end{equation}
The rest of the proof derives a bound on $Q_n$.
From the concavity of the utility,
\[
\psi(\matC^{\star}) \le \psi(\matC_{t-1}) + \nabla \psi(\matC_{t-1}) \cdot (\matC^{\star} - \matC_{t-1}),
\]
or, after rearranging,
\begin{equation}
\nabla \psi(\matC_{t-1}) \cdot (\matC_{t-1}-\matC^{\star}) \le \psi(\matC_{t-1}) - \psi(\matC^{\star}).
\label{eq:from_concavity_OMMA_eta}
\end{equation}
Using the smoothness of $\psi$,
\begin{align}
\psi(\matC_t) &\ge \psi(\matC_{t-1}) + \nabla \psi(\matC_{t-1}) \cdot (\matC_t - \matC_{t-1}) - \frac{M}{2} \|\matC_t - \matC_{t-1}\|_1^2 \nonumber \\
&= \psi(\matC_{t-1}) + \frac{1}{t} \nabla \psi(\matC_{t-1}) \cdot (\matC(\veceta_t, \vecpred_t) - \matC_{t-1}) - 
\frac{M}{2t^2} \|\matC(\veceta_t, \vecpred_t) - \matC_{t-1}\|^2 \nonumber \\
&\ge \psi(\matC_{t-1}) + \frac{1}{t} \nabla \psi(\matC_{t-1}) \cdot (\matC(\veceta_t, \vecpred_t) - \matC_{t-1}) - \frac{Ma}{2t^2},
\label{eq:OMMA_eta_proof_eq_3}
\end{align}
where in the first equality we used
\[
\matC_t = \frac{t-1}{t} \matC_{t-1} + \frac{1}{t} \matC(\veceta_t, \vecpred_t) = \matC_{t-1} + \frac{1}{t} (\matC(\veceta_t, \vecpred_t) - \matC_{t-1}),
\]
while in the second inequality we bounded $\|\matC(\veceta_t, \vecpred_t) - \matC_{t-1}\|^2 \le \max_{\matC_1, \matC_2 \in \confusionmatrixspace} \|\matC_1 - \matC_2\|^2 \le 2 a$.
Multiplying both sides of \eqref{eq:OMMA_eta_proof_eq_3} by $(-1)$ and adding $\psi(\matC^{\star})$ gives:
\begin{align*}
\psi(\matC^{\star}) - \psi(\matC_t)
&\le \psi(\matC^{\star}) - \psi(\matC_{t-1})
+ \frac{1}{t} \nabla \psi(\matC_{t-1}) \cdot (\matC_{t-1} - \matC(\veceta_t, \vecpred_t)) + 
\frac{M a}{t^2} \\
&= \psi(\matC^{\star}) - \psi(\matC_{t-1})
+ \frac{1}{t} \nabla \psi(\matC_{t-1}) \cdot (\matC_{t-1} - \matC^{\star})
+ \frac{1}{t} \nabla \psi(\matC_{t-1}) \cdot (\matC^{\star} - \matC(\veceta_t, \vecpred_t)) + 
\frac{M a}{t^2} \\
&\le \frac{t-1}{t} \left(\psi(\matC^{\star}) - \psi(\matC_{t-1})\right)
+ \frac{1}{t} \nabla \psi(\matC_{t-1}) \cdot (\matC^{\star} - \matC(\veceta_t, \vecpred_t)) + 
\frac{M a}{t^2} 
\end{align*}
where in the last inequality we used \eqref{eq:from_concavity_OMMA_eta}. Multiplying both sides by $t$ and denoting
$A_t = t(\psi(\matC^{\star}) - \psi(\matC_t))$, we get
\[
A_t \le A_{t-1} + \nabla \psi(\matC_{t-1}) \cdot (\matC^{\star} - \matC(\veceta_t, \vecpred_t)) + 
\frac{M a}{t}.
\]
Taking expectation with respect to $\vecx_t$ (conditioned on $\vecx^{t-1}$) on both sides and using the fact
than $\matC_{t-1}, \matC^{\star}, A_{t-1}$ do not depend on $\vecx_t$,
\begin{equation}
\EE_{\vecx_t}[A_t] \le A_{t-1} + \underbrace{\nabla \psi(\matC_{t-1}) \cdot \matC^{\star} - \nabla \psi(\matC_{t-1}) \cdot \EE_{\vecx_t} \left[\matC(\veceta_t, \vecpred_t))\right]}_{B_t}
+ \frac{M a}{t}.
\label{eq:regret_bound_intermediate_1_OMMA_eta}
\end{equation}
We show that $B_t$ is nonpositive.
Using \eqref{eq:C_star_OMMA_eta},
\begin{align*}
\nabla \psi(\matC_{t-1}) \cdot \matC^{\star} &=
\nabla \psi(\matC_{t-1}) \cdot \EE_{\vecx} \left[\matC(\veceta(\vecx), \vech^{\star}(\vecx)) \right] \\
(\text{$\vecx$ has the same distr. as $\vecx_t$})\quad
&= \nabla \psi(\matC_{t-1}) \cdot \EE_{\vecx_t} \left[\matC(\veceta_t, \vech^{\star}(\vecx_t)) \right] \\
&=\EE_{\vecx_t} \left[ \nabla \psi(\matC_{t-1}) \cdot \matC(\veceta_t, \vech^{\star}(\vecx_t)) \right] \\
(\text{from \eqref{eq:OMMA_eta_prediction}})\quad
&\le \EE_{\vecx_t} \left[ \nabla \psi(\matC_{t-1}) \cdot \matC(\veceta_t, \vecpred_t) \right] \\
&= \nabla \psi(\matC_{t-1}) \cdot \EE_{\vecx_t} \left[\matC(\veceta_t, \vecpred_t))\right],
\end{align*}
which implies $B_t \le 0$. Using this fact and taking the expectation on both sides of \eqref{eq:regret_bound_intermediate_1_OMMA_eta} with respect
to all the past $\vecx^{t-1}$ gives:
\[
\EE[A_t] \le \EE[A_{t-1}] + \frac{M a}{t}.
\]
Taking $t=n$, expanding the recursion down to $A_0=0$, and bounding $\sum_{t=1}^n \frac{1}{t} \le 1 + \ln n$ as in the Appendix \ref{app:regret_bounds}
gives $\EE[A_n] \le Ma (1 + \ln n)$. Since $\EE[A_n] = n Q_n$ as defined in \eqref{eq:regret_bound_OMMA_eta_Q},
we have $Q_n \le \frac{Ma(1+ \ln n)}{n}$ and \eqref{eq:regret_bound_OMMA_eta_Q} gives:
\[
\regret_n \le \frac{Ma(2+ \ln n)}{n}.
\]

\section{An alternative regret bound for \OMMA{} with $\vecp$-dependent Lipschitzness and smoothness constants}
\label{app:alternative_proof}

As mentioned in Section \ref{sec:regret_bounds}, some of the utilities in Table \ref{tab:binary-metrics}, or their
micro- and macro-extensions in multi-label classification, are not globally Lipschitz and smooth with fixed constants $L$ and $M$
over all admissible confusion matrices. One can, however, make them smooth and Lipschitz in a slightly narrower sense,
if one allows these constants to depend on the
\emph{label frequencies} in the confusion matrix, that is the averaged/expected rate of positive labels per class, which leads to a so called
$\vecp$-Lipschitnzess property \cite{Dembczynski_etal_ICML2017}, as well as $\vecp$-smoothness.

To make it formal, define a label frequency vector $\vecp(\matC) \in [0,1]^m$ for \emph{multi-class} confusion matrix $\matC \in [0,1]^{m \times m}$
as $p_k(\matC) = \frac{1}{m} \sum_{\ell=1}^m C_{k\ell}$, while
for \emph{multi-label} confusion matrix $\matC \in [0,1]^{m \times 2 \times 2}$ as $p_k(\matC) = C_{k,11} + C_{k,10}$. Note if $\matC$ is computed over some data distribution $\datadistribution$, 
$p_k(\matC) = \datadistribution(y_k=1)$ is a probability
of class $k$ in the multi-class case, or a marginal probability of label $k$ in the multi-label case. 
For empirical confusion matrix $\matC = \matC(\vecy^n, \vecpred^n)$, $\vecp(\matC)$ is a vector of average counts,
$\vecp(\matC) = \avgy_n := \frac{1}{n} \sum_{i=1}^n \vecy^i$. 
Note that the label frequency vector is independent on the predictions and only depends on the distribution/sequence
of labels.

We say utility $\psi$ is $\vecp$-Lipschitz if for any $\matC,\matC'$, $|\psi(\matC) - \psi(\matC')| \le L(\vecp(\matC)) \|\matC - \matC'\|$,
that is, the Lipschitz constant may depend on $\vecp(\matC)$. Similarly, $\psi$ is $\vecp$-smooth if
$\|\nabla \psi(\matC) - \nabla \psi(\matC')| \le M(\vecp(\matC)) \|\matC - \matC'\|$. Generally these functions
$L(\vecp)$ and $M(\vecp)$ will degrade (increase) when
any of coordinates of $\vecp$ gets close to zero. Therefore, we define:
\[
L_{\gamma} = \max_{\vecp \colon \forall k, p_k \ge \gamma} L(\vecp),
\]
as the largest value of $L(\vecp)$ when all elements of $\vecp$ are separated from $0$ by at least $\gamma$; $M_{\gamma}$ is defined analogously.
In what follows, we assume that $L_{\gamma}$ and $M_{\gamma}$ are well-controlled when $\gamma$ is separated from $0$.

Given a data distribution $\datadistribution$, let $\vecp^{\star} = \EE\left[\vecy\right]$ denote a vector of label (marginal) probabilities,
and we let $\gamma^{\star} = \min_{k} p^{\star}_k$.

\begin{theorem}
Let $\psi$ be concave, $\vecp$-Lipschitz and $\vecp$-smooth. Assume $\psi(\matC) \in [0,1]$ for
all $\matC \in \confusionmatrixspace$. The regret of \OMMA{} algorithm is bounded by:
\[
\regret_n \le O\left(\sqrt{\frac{\log n}{n}}\right) + 2bL_{\frac{\gamma^{\star}}{4}} e_n
\]
where $e_n = \frac{1}{n} \sum_{t=1}^n \EE[\|\veceta(\vecx_t) - \empveceta_t(\vecx_t)\|]$ is the average
estimation error of the CPE,
while the constants under $O(\cdot)$ depend on $m$, $L_{\frac{\gamma^{\star}}{4}}$ and $M_{\frac{\gamma^{\star}}{4}}$.
\label{thm:alternative_regret_bound}
\end{theorem}

\begin{proof}
We assume the reader is familiar with the proof of Theorem \ref{thm:regret} in Appendix \ref{app:regret_bounds}, to which we will
refer at multiple occasions. We start with recalling that the regret can be upper-bounded by
$\regret_n \le \psi(\matC^{\star})  - \EE\left[\psi(\matC_n)\right]$, where $\matC^{\star} = \EE_{(\vecx, \vecy)}\left[ \matC(\vecy,\vech^{\star}(\vecx)) \right]$ is the population confusion matrix of the optimal classifier $\vech^{\star}$.

We say that a sequence of labels $\vecy^n$ is $(n_0,\beta)$-\emph{nice} if for all $t \ge n_0$, the running averages are all at least $\beta$,
that is $\overline{y}_{tk} \ge \beta$ for all $t \ge n_0$ and all $k=1,\ldots,m$.

\begin{lemma}
Let $\vecy^n$ be drawn i.i.d. from $\datadistribution$. Then $\vecy^n$ is $(n_0,\frac{\gamma^{\star}}{4})$-nice
with probability at least $1 - m n_0 e^{-\gamma^{\star}{}^2 n_0 / 2}$.
\label{lem:nice}
\end{lemma}
\begin{proof}(of Lemma \ref{lem:nice})
\newcommand{\nJ}{n_{\!J}}
Split the sequence into $J$ intervals of length at least $\nJ = \lfloor \frac{n}{J} \rfloor$ each (for the sake of readability,
assume $n$ is divisible by $J$). Let
$$\bm{s}_j = \frac{1}{\nJ} \sum_{t=1 + (j-1) \nJ}^{j \nJ} \vecy_t$$ 
be the vector of average
label counts within the $j$-th interval, $j \in \{1,\ldots,J\}$. Note that $\EE[\bm{s}_j] = \vecp^{\star}$.
Using Hoeffding's inequality \citep{hoeffding1963probability,CBLu06:book} and a union bound
over coordinates/labels $k=1,\ldots,m$,
we have $\min_k \{s_{j,k} - p^{\star}_k\} \ge -\beta$ with probability at least $1-m e^{-2 \beta^2 n_J}$. 
Next, taking union bound over
the intervals, $\min_{j,k} \{s_{j,k} - p^{\star}_k\} \ge -\beta$ with probability at least $1-\delta$ where $\delta = \nJ m e^{-2 \beta^2 \nJ}$. 
Thus, with probability at least $1-\delta$, for every interval $j$ and every label $k$ simultaneously,
\begin{equation}
\sum_{i=1 + (j-1) \nJ}^{j \nJ} y_{ik} \ge \nJ (p^{\star}_k - \beta) \,.
\end{equation}

Take any iteration $t$ in the $j$-th interval with $j \ge 2$,
that is $t = (j-1)\nJ + r$ for $1 \le r \le \nJ$. We have, with probability at least $1-\delta$:
\[
\overline{y}_{tk} = \frac{1}{t} \sum_{i=1}^t y_{ik} \ge \frac{1}{t} \sum_{i=1}^{(j-1)\nJ} y_{ik} \ge \frac{(j-1) \nJ}{t} (p^{\star}_k- \beta) \,.
\]
Plugging in $t = (j-1)\nJ + r$, 
\[
\overline{y}_{tk} \geq \frac{(j-1) n_J}{(j-1)\nJ + r} (p^{\star}_k - \beta)
\ge \frac{(j-1) \nJ}{j\nJ} (p^{\star}_k - \beta) = \frac{(j-1)}{j} (p^{\star}_k - \beta)  \ge \frac{1}{2} (p^{\star}_k - \beta), 
\]
where we used the fact that $j \ge 2$ for the last inequality. Thus, for every $t$ outside the first interval, the running average for each label
is at least $\frac{1}{2} (p^{\star}_k - \beta)$. Taking $\beta = \frac{\gamma^{\star}}{2}$, using the fact
that $p^{\star}_k \ge \gamma^{\star}$, we see that
$\overline{y}_{tk} \ge \frac{\gamma^{\star}}{4}$ for all $t > \nJ \eqqcolon n_0$ and all $k$ with probability at least
$1 - m n_0 e^{-\gamma^{\star}{}^2 n_0 / 2}$.
\end{proof}

We use Lemma \ref{lem:nice} as follows: with probability at least $1-\delta_1$ with $\delta_1 := m n_0 e^{-\gamma^{\star}{}^2 n_0 / 2}$, starting
from $t \ge n_0+1$, $\overline{y}_{tk} \ge \frac{\gamma^{\star}}{4}$ for all $t$ and $k$, and thus
we have $\vecp$-Lipschitzness and $\vecp$-smoothness well-controlled for $\psi$ evaluated at $\matC_{t-1}$, that is
for $t \ge n_0+1$, all constants used in the proof of Theorem \ref{thm:regret} are at most
$L_{\frac{\gamma^{\star}}{4}}$ and $M_{\frac{\gamma^{\star}}{4}}$, which we denote simply by $L$ and $M$.

We will need one more probabilistic bound. In Appendix \ref{app:regret_bounds} we have shown that
\[
\EE_{(\vecx_t,\vecy_t)} \left[\nabla \psi(\matC_{t-1}) \cdot (\matC^{\star} - \matC(\vecy_t, \vecpred_t))\right]
\le 2 b L \EE_{\vecx_t} \left[ \|\veceta(\vecx_t) - \empveceta_t(\vecx_t)\| \right].
\]
Define $V_t = \nabla \psi(\matC_{t-1}) \cdot (\matC^{\star} - \matC(\vecy_t, \vecpred_t))
-  2 b L \EE_{\vecx_t} \left[ \|\veceta(\vecx_t) - \empveceta_t(\vecx_t) \| \right]$ and note that
the above inequality implies $\EE_{(\vecx_t,\vecy_t)}[V_t] \le 0$.
We also have 
\[
|V_t - \EE_{(\vecx_t,\vecy_t)}[V_t]| \le |V_t| + |\EE_{(\vecx_t,\vecy_t)}[V_t]|
\le 2 \|\nabla \psi(\matC_{t-1})\| \max_{\matC, \matC'} \|\matC - \matC'\|
+ 4 bL \max_{\veceta,\veceta'} \|\veceta - \veceta'\|,
\]
where we used Cauchy-Schwarz inequality. We note that $ \max_{\matC, \matC'} \|\matC - \matC'\| \le 2a$
and $\max_{\veceta,\veceta'} \|\veceta - \veceta'\| \le c$ where $c=2$ for multi-class classification
and $c=m$ for multi-label classification. Since $b=1$ for multi-class and $b=\sqrt{2}$ for multi-label classification, denoting
$G_{t-1} = \|\nabla \psi(\matC_{t-1})\|$ we conclude that
$|V_t - \EE_{(\vecx_t,\vecy_t)}[V_t]| \le 4 a G_{t-1} + 8aL$ (note that $a=1$ for multi-class and $a=m$ for multi-label case).
We define $U_t = V_t - \EE_{(\vecx_t,\vecy_t)}[V_t]$ if $G_{t-1} \le L$ and otherwise $U_t = 0$.
It follows that $|U_t| \le 4aL + 8aL \le 12aL$, and $U_t$ forms
a martingale difference sequence (see \citet{CBLu06:book}, Appendix A.1.3). Thus, by Hoeffding-Azuma
inequality, $\frac{1}{n-m_0} \sum_{t=m_0+1}^{n} U_t \le \beta$ with probability
at most $1-\delta_2 = 1 - \exp\left\{-\frac{2 \beta^2 (n-m_0)}{144 a^2 L^2} \right\}$. Using it on a nice sequence (when $G_{t-1} \le L$ for all
$t \ge n_0 + 1$), we conclude (by the union bound) that with probability $1- \delta_1 - \delta_2$,
$$\sum_{t > n_0} V_t \le \sum_{t > n_0} (V_t - \EE_{(\vecx_t,\vecy_t)}[V_t]) 
= \sum_{t > n_0} U_t \le \beta (n - n_0).$$
Having control over $L$ and $M$, and over the term $\sum_{t > n_0} V_t$, we use the entire machinery of the proof of Theorem \ref{thm:regret}
in the Appendix \ref{app:regret_bounds} (with expectation $\EE_{(\vecx_t,\vecy_t)}[\cdot]$ replaced by just proven high-probability bound)
to conclude that with probability at least $1 - \delta_0 - \delta_1$,
\begin{align*}
A_n &\le A_{n_0} + 2 b L \sum_{t=n_0+1}^n\EE\left[ \|\veceta(\vecx_t) - \empveceta_t(\vecx_t)\| \right]
 + \beta(n - n_0) + \sum_{t=n_0+1}^n \frac{M a}{t} \\
&\le A_{n_0} + 2 b L \sum_{t=1}^n\EE\left[ \|\veceta(\vecx_t) - \empveceta_t(\vecx_t)\| \right] + \beta (n - n_0)
+ Ma (1+ \ln n)
\end{align*}
where $A_t = t(\psi(\matC^{\star}) - \psi(\matC_t))$. This means that with probability at least $1 - \delta_0 - \delta_1$,
\[
\psi(\matC^{\star}) - \psi(\matC_n) \le \frac{2n_0}{n} 
+ \frac{\beta (n - n_0)}{n} + \frac{Ma (1+ \ln n)}{n} + \frac{2bL}{n} \sum_{t=1}^n\EE\left[ \|\veceta(\vecx_t) - \empveceta_t(\vecx_t)\| \right],
\]
where we bounded $|\psi(\matC^{\star}) - \psi(\matC_n)| \le 1$ using assumption $\psi(\matC) \in [0,1]$ for all $\matC \in \confusionmatrixspace$.
Let $E$ be the event that the sequence is nice and Hoeffding-Azuma
inequality holds, which happens with probability at least $1 - \delta_0 - \delta_1$. We have
\begin{align*}
\MoveEqLeft
\EE[\psi(\matC^{\star}) - \psi(\matC_n)] =
\EE[\psi(\matC^{\star}) - \psi(\matC_n) | E] \; \datadistribution(E)
+ \underbrace{\EE[\psi(\matC^{\star}) - \psi(\matC_n) | \neg E]}_{\le 1} \; \datadistribution(\neg E) \\
&\le \EE[\psi(\matC^{\star}) - \psi(\matC_n) | E] + \datadistribution(\neg E) \\
&\le \frac{2n_0}{n} 
+ \frac{\beta (n - n_0)}{n} + \frac{Ma (1+ \ln n)}{n} + \frac{2bL}{n} \sum_{t=1}^n\EE\left[ \|\veceta(\vecx_t) - \empveceta_t(\vecx_t)\| \right]
+ \delta_0 + \delta_1 \\
&= \frac{2n_0}{n} 
+ \frac{\beta (n - n_0)}{n}+ \frac{Ma (1+ \ln n)}{n} + \frac{2bL}{n} \sum_{t=1}^n\EE\left[ \|\veceta(\vecx_t) - \empveceta_t(\vecx_t)\| \right] +
m n_0 e^{-\frac{\gamma^{\star}{}^2 n_0}{2}} + e^{-\frac{2 \beta^2 (n-m_0)}{144 a^2 L^2}}.
\end{align*}
Finally, we tune the remaining parameters $\beta$ and $n_0$. We set $n_0 = \sqrt{n}$, and $\beta = \frac{12 aL \sqrt{\ln n}}{\sqrt{2 (n-n_0)}}$
(in order to have $e^{-\frac{2 \beta^2 (n-m_0)}{144 a^2 L^2}} = e^{-\ln n} = \frac{1}{n}$). Bounding $\sqrt{n-n_0} \le \sqrt{n}$,
and denoting the estimation error by $e_n = \frac{1}{n}\sum_{t=1}^n\EE\left[ \|\veceta(\vecx_t) - \empveceta_t(\vecx_t)\| \right]$, we finally get:
\begin{align*}
\EE[\psi(\matC^{\star}) - \psi(\matC_n)]
&\le \frac{2 + 12 a L_{\frac{\gamma^{\star}}{4}} \sqrt{\ln n}}{\sqrt{n}} 
+ \frac{M_{\frac{\gamma^{\star}}{4}} a (1+ \ln n)}{n} + \frac{2bL_{\frac{\gamma^{\star}}{4}}}{n}
+ m \sqrt{n} e^{-\frac{\gamma^{\star}{}^2 \sqrt{n}}{2}} + \frac{1}{n}
+ 2bL_{\frac{\gamma^{\star}}{4}} e_n \\
&= O\left(\sqrt{\frac{\log n}{n}}\right) + 2bL_{\frac{\gamma^{\star}}{4}} e_n,
\end{align*}
which was to be shown.
\end{proof}
\section{Details of experiments setup}
\label{app:experiments-setup}

In multi-label experiments with fixed CPE, we use LIBLINEAR model with $L_2$-regularized logistic loss~\citep{liblinear}. For larger problems (RCV1X and AmazonCat), we made it sparse by truncating all the weights whose absolute value is above the threshold of 0.01 as introduced in \citep{babbar2017dismec} to significantly reduce memory requirements and inference time, then we made the probability estimates $\hat \eta$ sparse for these datasets by only keeping top 100 values for each sample, similarly to~\citep{Schultheis_etal_NeurIPS2023}. In multi-class experiments with both fixed and online CPE, we use PyTorch~\citep{pytorch} linear model with $L_2$-regularized binary cross entropy or cross entropy loss, optimized using ADAM~\citep{Kingma_et_al_2015}. We tuned the learning rate and regularization strength and used the best-obtained model.

We evaluate all the algorithms using 5 random sequences of test sets according to~\cref{fig:online-protocol}. All algorithms are tested on the same 5 sequences for fair comparison. All the methods were implemented in Python. For RCV1X and AmazonCat, we implemented all the algorithms (OFO, Greedy, Frank-Wolfe, and OMMA) to allow them to leverage the sparsity for a more fair computational cost comparison. 

All the experiments were conducted on a workstation with 64 GB of RAM.

\section{Extended results with the fixed CPE}
\label{app:extended-results}

In this section, we include extended results of our empirical experiments. In~\cref{tab:extended-results-multi-label}, we report the results presented in the main paper but also results for the Greedy algorithm that only uses $\empveceta_t$ (Greedy$(\hat \eta)$) as initially proposed by~\citet{Schultheis_etal_NeurIPS2023} and Offline Frank-Wolfe (Offline-FW) algorithm. We additionally report standard deviations of the results as well as mean running times. In~\cref{fig:all-plots}, we present running performance for all datasets and metrics presented in~\cref{tab:extended-results-multi-label}.
In~\cref{fig:all-reg-plots}, we present the effect of selecting different regularization constants $\lambda$ for a selected subset of datasets (without RCV1X and AmazonCat) and all metrics presented in~\cref{tab:extended-results-multi-label}.

In most cases, Greedy$(\hat \eta)$ performs similarly to \OMMAeta{}, as in the case of standard Greedy and \OMMA performing similarly. Offline-FW obtains mixed results, sometimes much better and sometimes much worse than online algorithms. Similarly to Online-FW, it is not performing that well on macro-averaged Precision@$k$.

We note that the running time comparison favors the Online-FW algorithm, as its updates are better-optimized thanks to vectorization and less frequent than OMMA's and other online methods, which update after every step. We expect OMMA to be much faster if implemented more efficiently.

\begin{table*}[ht]
    \caption{Predictive performance and running times of the different online algorithms on \emph{multi-label} problems, averaged over 5 runs and reported with standard deviations ($\pm$)s. The best result on each metric is in \textbf{bold}, the second best is in \textit{italic}. We additionally report basic statistics of the benchmarks: number of classes/labels and instances in the test sequence. $\times$ -- means that the algorithm does not support the optimization of that metric. The inference time of Offline-FW was not measured.}
    \label{tab:extended-results-multi-label}
    \vspace{8pt}
    \small
    \centering

\resizebox{\linewidth}{!}{
\setlength\tabcolsep{3 pt}
\begin{tabular}{l|r@{}lr@{}l|r@{}lr@{}l|r@{}lr@{}l|r@{}lr@{}l|r@{}lr@{}l|r@{}lr@{}l|r@{}lr@{}l}
\toprule
    Method & \multicolumn{4}{c|}{Micro F1} & \multicolumn{4}{c|}{Macro F1} & \multicolumn{4}{c|}{Macro F1$@3$} & \multicolumn{4}{c|}{Macro Recall$@3$} & \multicolumn{4}{c|}{Macro Precision$@3$} & \multicolumn{4}{c|}{Macro G-mean}& \multicolumn{4}{c}{Macro H-mean} \\
    & \multicolumn{2}{c}{(\%)} & \multicolumn{2}{c|}{time (s)} & \multicolumn{2}{c}{(\%)} & \multicolumn{2}{c|}{time (s)} & \multicolumn{2}{c}{(\%)} & \multicolumn{2}{c|}{time (s)} & \multicolumn{2}{c}{(\%)} & \multicolumn{2}{c|}{time (s)} & \multicolumn{2}{c}{(\%)} & \multicolumn{2}{c|}{time (s)} & \multicolumn{2}{c}{(\%)} & \multicolumn{2}{c|}{time (s)} & \multicolumn{2}{c}{(\%)} & \multicolumn{2}{c}{time (s)} \\
\midrule
& \multicolumn{28}{c}{\datasettable{YouTube} ($m = 46, n = 7926$)}  \\ 
 \midrule
    Top-$k$ / $\hat \eta \!>\!0.5$ & 31.20 & \scriptsize $\pm$ 0.00 & 0.13 & \scriptsize $\pm$ 0.01 & 22.74 & \scriptsize $\pm$ 0.00 & 0.13 & \scriptsize $\pm$ 0.01 & 30.99 & \scriptsize $\pm$ 0.00 & 0.08 & \scriptsize $\pm$ 0.00 & 42.13 & \scriptsize $\pm$ 0.00 & 0.09 & \scriptsize $\pm$ 0.02 & 26.39 & \scriptsize $\pm$ 0.00 & 0.08 & \scriptsize $\pm$ 0.00 & 32.82 & \scriptsize $\pm$ 0.00 & 0.13 & \scriptsize $\pm$ 0.01 & 24.46 & \scriptsize $\pm$ 0.00 & 0.13 & \scriptsize $\pm$ 0.01 \\
    \midrule
    OFO & 43.71 & \scriptsize $\pm$ 0.05 & 0.33 & \scriptsize $\pm$ 0.01 & 36.15 & \scriptsize $\pm$ 0.17 & 0.58 & \scriptsize $\pm$ 0.10 & & $\times$ & & $\times$ & & $\times$ & & $\times$ & & $\times$ & & $\times$ & & $\times$ & & $\times$ & & $\times$ & & $\times$ \\
    Greedy & & $\times$ & & $\times$ & 36.32 & \scriptsize $\pm$ 0.15 & 1.35 & \scriptsize $\pm$ 0.08 & 34.72 & \scriptsize $\pm$ 0.17 & 1.27 & \scriptsize $\pm$ 0.02 & 45.84 & \scriptsize $\pm$ 0.09 & 1.22 & \scriptsize $\pm$ 0.02 & \textit{67.18}&\textit{\scriptsize $\pm$ 1.56} & 1.22 & \scriptsize $\pm$ 0.03 & \textit{77.98}&\textit{\scriptsize $\pm$ 0.05} & 1.44 & \scriptsize $\pm$ 0.04 & \textbf{77.93}&\textbf{\scriptsize $\pm$ 0.07} & 1.57 & \scriptsize $\pm$ 0.16 \\
    Greedy$(\hat \eta)$ & & $\times$ & & $\times$ & \textbf{36.47}&\textbf{\scriptsize $\pm$ 0.07} & 1.30 & \scriptsize $\pm$ 0.04 & \textit{35.38}&\textit{\scriptsize $\pm$ 0.05} & 1.27 & \scriptsize $\pm$ 0.02 & 45.84 & \scriptsize $\pm$ 0.09 & 1.23 & \scriptsize $\pm$ 0.04 & \textit{67.18}&\textit{\scriptsize $\pm$ 1.56} & 1.23 & \scriptsize $\pm$ 0.04 & \textbf{77.99}&\textbf{\scriptsize $\pm$ 0.02} & 1.42 & \scriptsize $\pm$ 0.05 & \textit{77.92}&\textit{\scriptsize $\pm$ 0.03} & 1.51 & \scriptsize $\pm$ 0.08 \\
    Online-FW & 43.67 & \scriptsize $\pm$ 0.05 & 6.80 & \scriptsize $\pm$ 0.55 & 36.00 & \scriptsize $\pm$ 0.11 & 10.92 & \scriptsize $\pm$ 0.21 & 34.38 & \scriptsize $\pm$ 0.25 & 10.76 & \scriptsize $\pm$ 0.42 & 45.83 & \scriptsize $\pm$ 0.10 & 5.34 & \scriptsize $\pm$ 0.07 & 38.92 & \scriptsize $\pm$ 1.56 & 6.54 & \scriptsize $\pm$ 0.39 & 77.96 & \scriptsize $\pm$ 0.05 & 6.60 & \scriptsize $\pm$ 0.26 & 77.91 & \scriptsize $\pm$ 0.04 & 11.46 & \scriptsize $\pm$ 0.57 \\
    Online-FW$(\hat \eta)$ & 43.69 & \scriptsize $\pm$ 0.02 & 10.02 & \scriptsize $\pm$ 0.14 & \textbf{36.47}&\textbf{\scriptsize $\pm$ 0.04} & 11.57 & \scriptsize $\pm$ 0.19 & \textbf{35.43}&\textbf{\scriptsize $\pm$ 0.08} & 29.99 & \scriptsize $\pm$ 0.39 & \textbf{45.89}&\textbf{\scriptsize $\pm$ 0.05} & 6.39 & \scriptsize $\pm$ 0.07 & 50.12 & \scriptsize $\pm$ 1.95 & 11.50 & \scriptsize $\pm$ 0.06 & 77.92 & \scriptsize $\pm$ 0.03 & 2.00 & \scriptsize $\pm$ 0.01 & \textbf{77.93}&\textbf{\scriptsize $\pm$ 0.03} & 15.52 & \scriptsize $\pm$ 1.48 \\
    \midrule
    \OMMA{} & \textbf{43.73}&\textbf{\scriptsize $\pm$ 0.04} & 2.78 & \scriptsize $\pm$ 0.18 & \textit{36.34}&\textit{\scriptsize $\pm$ 0.19} & 2.84 & \scriptsize $\pm$ 0.09 & 34.81 & \scriptsize $\pm$ 0.11 & 2.80 & \scriptsize $\pm$ 0.05 & \textit{45.85}&\textit{\scriptsize $\pm$ 0.09} & 2.41 & \scriptsize $\pm$ 0.07 & \textbf{67.74}&\textbf{\scriptsize $\pm$ 1.61} & 2.40 & \scriptsize $\pm$ 0.08 & \textit{77.98}&\textit{\scriptsize $\pm$ 0.05} & 2.92 & \scriptsize $\pm$ 0.13 & \textbf{77.93}&\textbf{\scriptsize $\pm$ 0.07} & 3.15 & \scriptsize $\pm$ 0.17 \\
    \OMMAeta{} & \textit{43.72}&\textit{\scriptsize $\pm$ 0.03} & 2.79 & \scriptsize $\pm$ 0.17 & \textbf{36.47}&\textbf{\scriptsize $\pm$ 0.08} & 2.73 & \scriptsize $\pm$ 0.10 & \textit{35.38}&\textit{\scriptsize $\pm$ 0.06} & 2.72 & \scriptsize $\pm$ 0.05 & \textbf{45.89}&\textbf{\scriptsize $\pm$ 0.04} & 2.38 & \scriptsize $\pm$ 0.06 & 65.49 & \scriptsize $\pm$ 2.14 & 2.38 & \scriptsize $\pm$ 0.07 & \textit{77.98}&\textit{\scriptsize $\pm$ 0.02} & 2.94 & \scriptsize $\pm$ 0.15 & \textbf{77.93}&\textbf{\scriptsize $\pm$ 0.03} & 3.15 & \scriptsize $\pm$ 0.22 \\
    \midrule
    Offline-FW & 43.75 & \scriptsize $\pm$ 0.01 & & $-$ & 36.55 & \scriptsize $\pm$ 0.05 & & $-$ & 35.68 & \scriptsize $\pm$ 0.05 & & $-$ & 44.34 & \scriptsize $\pm$ 0.00 & & $-$ & 51.51 & \scriptsize $\pm$ 5.55 & & $-$ & 78.03 & \scriptsize $\pm$ 0.01 & & $-$ & 77.96 & \scriptsize $\pm$ 0.01 & & $-$ \\
\midrule
& \multicolumn{28}{c}{\datasettable{Eurlex-LexGlue} ($m = 100, n = 5000$)}  \\ 
 \midrule
    Top-$k$ / $\hat \eta \!>\!0.5$ & 70.99 & \scriptsize $\pm$ 0.00 & 0.08 & \scriptsize $\pm$ 0.00 & 52.43 & \scriptsize $\pm$ 0.00 & 0.08 & \scriptsize $\pm$ 0.00 & 46.35 & \scriptsize $\pm$ 0.00 & 0.06 & \scriptsize $\pm$ 0.01 & 36.67 & \scriptsize $\pm$ 0.00 & 0.05 & \scriptsize $\pm$ 0.00 & 74.70 & \scriptsize $\pm$ 0.00 & 0.05 & \scriptsize $\pm$ 0.00 & 62.33 & \scriptsize $\pm$ 0.00 & 0.08 & \scriptsize $\pm$ 0.01 & 55.95 & \scriptsize $\pm$ 0.00 & 0.08 & \scriptsize $\pm$ 0.00 \\
    \midrule
    OFO & 73.23 & \scriptsize $\pm$ 0.01 & 0.20 & \scriptsize $\pm$ 0.01 & 58.93 & \scriptsize $\pm$ 0.22 & 0.18 & \scriptsize $\pm$ 0.01 & & $\times$ & & $\times$ & & $\times$ & & $\times$ & & $\times$ & & $\times$ & & $\times$ & & $\times$ & & $\times$ & & $\times$ \\
    Greedy & & $\times$ & & $\times$ & 59.83 & \scriptsize $\pm$ 0.11 & 0.84 & \scriptsize $\pm$ 0.03 & 54.19 & \scriptsize $\pm$ 0.07 & 0.81 & \scriptsize $\pm$ 0.03 & 52.67 & \scriptsize $\pm$ 0.13 & 0.76 & \scriptsize $\pm$ 0.03 & \textit{88.21}&\textit{\scriptsize $\pm$ 1.24} & 0.76 & \scriptsize $\pm$ 0.02 & 89.74 & \scriptsize $\pm$ 0.09 & 0.90 & \scriptsize $\pm$ 0.05 & 89.73 & \scriptsize $\pm$ 0.10 & 0.96 & \scriptsize $\pm$ 0.03 \\
    Greedy$(\hat \eta)$ & & $\times$ & & $\times$ & \textbf{59.88}&\textbf{\scriptsize $\pm$ 0.15} & 0.84 & \scriptsize $\pm$ 0.04 & \textbf{54.39}&\textbf{\scriptsize $\pm$ 0.15} & 0.83 & \scriptsize $\pm$ 0.04 & 52.67 & \scriptsize $\pm$ 0.13 & 0.75 & \scriptsize $\pm$ 0.02 & \textit{88.21}&\textit{\scriptsize $\pm$ 1.24} & 0.75 & \scriptsize $\pm$ 0.03 & \textbf{89.92}&\textbf{\scriptsize $\pm$ 0.05} & 0.90 & \scriptsize $\pm$ 0.04 & \textbf{89.85}&\textbf{\scriptsize $\pm$ 0.05} & 0.94 & \scriptsize $\pm$ 0.04 \\
    Online-FW & \textbf{73.68}&\textbf{\scriptsize $\pm$ 0.04} & 3.76 & \scriptsize $\pm$ 0.29 & 59.53 & \scriptsize $\pm$ 0.14 & 7.60 & \scriptsize $\pm$ 0.50 & 54.18 & \scriptsize $\pm$ 0.06 & 8.72 & \scriptsize $\pm$ 0.39 & 52.68 & \scriptsize $\pm$ 0.19 & 3.95 & \scriptsize $\pm$ 0.15 & 58.76 & \scriptsize $\pm$ 1.64 & 8.20 & \scriptsize $\pm$ 1.25 & 89.69 & \scriptsize $\pm$ 0.05 & 3.88 & \scriptsize $\pm$ 0.26 & 89.75 & \scriptsize $\pm$ 0.08 & 7.68 & \scriptsize $\pm$ 0.41 \\
    Online-FW$(\hat \eta)$ & 73.22 & \scriptsize $\pm$ 0.01 & 7.26 & \scriptsize $\pm$ 0.11 & 59.78 & \scriptsize $\pm$ 0.11 & 9.68 & \scriptsize $\pm$ 0.13 & 54.34 & \scriptsize $\pm$ 0.15 & 34.28 & \scriptsize $\pm$ 0.40 & \textbf{53.87}&\textbf{\scriptsize $\pm$ 0.15} & 5.40 & \scriptsize $\pm$ 0.04 & 51.37 & \scriptsize $\pm$ 0.89 & 23.53 & \scriptsize $\pm$ 0.87 & \textit{89.91}&\textit{\scriptsize $\pm$ 0.02} & 1.71 & \scriptsize $\pm$ 0.02 & \textit{89.84}&\textit{\scriptsize $\pm$ 0.04} & 20.59 & \scriptsize $\pm$ 1.95 \\
    \midrule
    \OMMA{} & \textit{73.29}&\textit{\scriptsize $\pm$ 0.02} & 1.73 & \scriptsize $\pm$ 0.07 & \textit{59.85}&\textit{\scriptsize $\pm$ 0.07} & 1.84 & \scriptsize $\pm$ 0.11 & 54.15 & \scriptsize $\pm$ 0.11 & 1.81 & \scriptsize $\pm$ 0.07 & 52.67 & \scriptsize $\pm$ 0.13 & 1.45 & \scriptsize $\pm$ 0.05 & \textbf{88.41}&\textbf{\scriptsize $\pm$ 0.73} & 1.44 & \scriptsize $\pm$ 0.04 & 89.74 & \scriptsize $\pm$ 0.09 & 1.88 & \scriptsize $\pm$ 0.17 & 89.73 & \scriptsize $\pm$ 0.10 & 1.97 & \scriptsize $\pm$ 0.11 \\
    \OMMAeta{} & 73.22 & \scriptsize $\pm$ 0.01 & 1.75 & \scriptsize $\pm$ 0.09 & 59.84 & \scriptsize $\pm$ 0.12 & 1.73 & \scriptsize $\pm$ 0.08 & \textit{54.37}&\textit{\scriptsize $\pm$ 0.13} & 1.73 & \scriptsize $\pm$ 0.07 & \textit{53.85}&\textit{\scriptsize $\pm$ 0.15} & 1.47 & \scriptsize $\pm$ 0.05 & 84.99 & \scriptsize $\pm$ 1.31 & 1.46 & \scriptsize $\pm$ 0.04 & \textbf{89.92}&\textbf{\scriptsize $\pm$ 0.05} & 1.90 & \scriptsize $\pm$ 0.18 & \textbf{89.85}&\textbf{\scriptsize $\pm$ 0.05} & 2.08 & \scriptsize $\pm$ 0.25 \\
    \midrule
    Offline-FW & 73.31 & \scriptsize $\pm$ 0.01 & & $-$ & 55.82 & \scriptsize $\pm$ 0.29 & & $-$ & 50.05 & \scriptsize $\pm$ 0.01 & & $-$ & 52.02 & \scriptsize $\pm$ 0.00 & & $-$ & 61.11 & \scriptsize $\pm$ 3.07 & & $-$ & 89.18 & \scriptsize $\pm$ 0.00 & & $-$ & 89.00 & \scriptsize $\pm$ 0.01 & & $-$ \\
\midrule
& \multicolumn{28}{c}{\datasettable{Mediamill} ($m = 101, n = 12914$)}  \\ 
 \midrule
    Top-$k$ / $\hat \eta \!>\!0.5$ & 52.45 & \scriptsize $\pm$ 0.00 & 0.23 & \scriptsize $\pm$ 0.02 & 4.06 & \scriptsize $\pm$ 0.00 & 0.23 & \scriptsize $\pm$ 0.01 & 4.43 & \scriptsize $\pm$ 0.00 & 0.13 & \scriptsize $\pm$ 0.00 & 4.35 & \scriptsize $\pm$ 0.00 & 0.13 & \scriptsize $\pm$ 0.00 & 7.42 & \scriptsize $\pm$ 0.00 & 0.13 & \scriptsize $\pm$ 0.00 & 4.62 & \scriptsize $\pm$ 0.00 & 0.23 & \scriptsize $\pm$ 0.02 & 3.65 & \scriptsize $\pm$ 0.00 & 0.23 & \scriptsize $\pm$ 0.01 \\
    \midrule
    OFO & \textit{56.99}&\textit{\scriptsize $\pm$ 0.01} & 0.57 & \scriptsize $\pm$ 0.01 & 12.36 & \scriptsize $\pm$ 0.03 & 1.06 & \scriptsize $\pm$ 0.03 & & $\times$ & & $\times$ & & $\times$ & & $\times$ & & $\times$ & & $\times$ & & $\times$ & & $\times$ & & $\times$ & & $\times$ \\
    Greedy & & $\times$ & & $\times$ & 12.43 & \scriptsize $\pm$ 0.07 & 2.85 & \scriptsize $\pm$ 0.27 & 10.29 & \scriptsize $\pm$ 0.07 & 2.18 & \scriptsize $\pm$ 0.17 & \textit{9.41}&\textit{\scriptsize $\pm$ 0.32} & 1.99 & \scriptsize $\pm$ 0.10 & \textit{22.54}&\textit{\scriptsize $\pm$ 1.35} & 1.98 & \scriptsize $\pm$ 0.11 & \textit{65.68}&\textit{\scriptsize $\pm$ 0.07} & 3.08 & \scriptsize $\pm$ 0.31 & \textbf{66.48}&\textbf{\scriptsize $\pm$ 0.41} & 3.14 & \scriptsize $\pm$ 0.22 \\
    Greedy$(\hat \eta)$ & & $\times$ & & $\times$ & \textit{14.33}&\textit{\scriptsize $\pm$ 0.04} & 2.32 & \scriptsize $\pm$ 0.10 & \textit{11.88}&\textit{\scriptsize $\pm$ 0.14} & 2.12 & \scriptsize $\pm$ 0.11 & \textit{9.41}&\textit{\scriptsize $\pm$ 0.32} & 1.99 & \scriptsize $\pm$ 0.11 & \textit{22.54}&\textit{\scriptsize $\pm$ 1.35} & 1.96 & \scriptsize $\pm$ 0.08 & \textbf{65.93}&\textbf{\scriptsize $\pm$ 0.17} & 3.05 & \scriptsize $\pm$ 0.26 & 65.37 & \scriptsize $\pm$ 0.16 & 3.14 & \scriptsize $\pm$ 0.21 \\
    Online-FW & 56.98 & \scriptsize $\pm$ 0.01 & 8.72 & \scriptsize $\pm$ 0.50 & 12.22 & \scriptsize $\pm$ 0.10 & 11.70 & \scriptsize $\pm$ 0.58 & 10.18 & \scriptsize $\pm$ 0.14 & 9.69 & \scriptsize $\pm$ 1.50 & 9.16 & \scriptsize $\pm$ 0.70 & 3.56 & \scriptsize $\pm$ 0.29 & 12.70 & \scriptsize $\pm$ 0.84 & 9.82 & \scriptsize $\pm$ 0.54 & 64.73 & \scriptsize $\pm$ 0.50 & 6.92 & \scriptsize $\pm$ 0.93 & \textit{65.68}&\textit{\scriptsize $\pm$ 0.45} & 20.44 & \scriptsize $\pm$ 1.45 \\
    Online-FW$(\hat \eta)$ & \textit{56.99}&\textit{\scriptsize $\pm$ 0.00} & 11.25 & \scriptsize $\pm$ 0.17 & \textit{14.33}&\textit{\scriptsize $\pm$ 0.02} & 16.15 & \scriptsize $\pm$ 1.44 & \textbf{11.97}&\textbf{\scriptsize $\pm$ 0.10} & 47.57 & \scriptsize $\pm$ 0.45 & \textbf{16.43}&\textbf{\scriptsize $\pm$ 0.08} & 8.13 & \scriptsize $\pm$ 0.05 & 17.35 & \scriptsize $\pm$ 0.68 & 27.09 & \scriptsize $\pm$ 1.23 & 65.37 & \scriptsize $\pm$ 0.22 & 2.48 & \scriptsize $\pm$ 0.02 & 65.56 & \scriptsize $\pm$ 0.07 & 49.09 & \scriptsize $\pm$ 0.26 \\
    \midrule
    \OMMA{} & \textbf{57.00}&\textbf{\scriptsize $\pm$ 0.01} & 4.56 & \scriptsize $\pm$ 0.33 & 12.39 & \scriptsize $\pm$ 0.07 & 4.51 & \scriptsize $\pm$ 0.05 & 10.25 & \scriptsize $\pm$ 0.10 & 4.56 & \scriptsize $\pm$ 0.23 & 9.39 & \scriptsize $\pm$ 0.32 & 3.83 & \scriptsize $\pm$ 0.19 & \textbf{22.85}&\textbf{\scriptsize $\pm$ 1.11} & 3.76 & \scriptsize $\pm$ 0.16 & 65.67 & \scriptsize $\pm$ 0.11 & 4.77 & \scriptsize $\pm$ 0.36 & \textbf{66.48}&\textbf{\scriptsize $\pm$ 0.40} & 4.95 & \scriptsize $\pm$ 0.23 \\
    \OMMAeta{} & \textit{56.99}&\textit{\scriptsize $\pm$ 0.01} & 4.52 & \scriptsize $\pm$ 0.25 & \textbf{14.34}&\textbf{\scriptsize $\pm$ 0.03} & 4.75 & \scriptsize $\pm$ 0.63 & 11.87 & \scriptsize $\pm$ 0.11 & 4.45 & \scriptsize $\pm$ 0.22 & \textbf{16.43}&\textbf{\scriptsize $\pm$ 0.07} & 3.82 & \scriptsize $\pm$ 0.16 & 18.31 & \scriptsize $\pm$ 0.49 & 3.81 & \scriptsize $\pm$ 0.19 & \textbf{65.93}&\textbf{\scriptsize $\pm$ 0.17} & 4.88 & \scriptsize $\pm$ 0.37 & 65.35 & \scriptsize $\pm$ 0.16 & 5.08 & \scriptsize $\pm$ 0.35 \\
    \midrule
    Offline-FW & 56.99 & \scriptsize $\pm$ 0.01 & & $-$ & 13.15 & \scriptsize $\pm$ 0.14 & & $-$ & 10.39 & \scriptsize $\pm$ 0.12 & & $-$ & 16.51 & \scriptsize $\pm$ 0.00 & & $-$ & 14.38 & \scriptsize $\pm$ 0.99 & & $-$ & 66.08 & \scriptsize $\pm$ 0.20 & & $-$ & 65.80 & \scriptsize $\pm$ 0.27 & & $-$ \\
\midrule
& \multicolumn{28}{c}{\datasettable{Flickr} ($m = 195, n = 24154$)}  \\ 
 \midrule
    Top-$k$ / $\hat \eta \!>\!0.5$ & 29.46 & \scriptsize $\pm$ 0.00 & 0.39 & \scriptsize $\pm$ 0.03 & 18.27 & \scriptsize $\pm$ 0.00 & 0.38 & \scriptsize $\pm$ 0.01 & 26.39 & \scriptsize $\pm$ 0.00 & 0.23 & \scriptsize $\pm$ 0.04 & 38.96 & \scriptsize $\pm$ 0.00 & 0.22 & \scriptsize $\pm$ 0.04 & 21.38 & \scriptsize $\pm$ 0.00 & 0.19 & \scriptsize $\pm$ 0.01 & 27.08 & \scriptsize $\pm$ 0.00 & 0.39 & \scriptsize $\pm$ 0.04 & 20.02 & \scriptsize $\pm$ 0.00 & 0.39 & \scriptsize $\pm$ 0.03 \\
    \midrule
    OFO & \textbf{41.05}&\textbf{\scriptsize $\pm$ 0.02} & 0.92 & \scriptsize $\pm$ 0.03 & 30.46 & \scriptsize $\pm$ 0.13 & 0.95 & \scriptsize $\pm$ 0.04 & & $\times$ & & $\times$ & & $\times$ & & $\times$ & & $\times$ & & $\times$ & & $\times$ & & $\times$ & & $\times$ & & $\times$ \\
    Greedy & & $\times$ & & $\times$ & 30.90 & \scriptsize $\pm$ 0.10 & 3.62 & \scriptsize $\pm$ 0.22 & 30.42 & \scriptsize $\pm$ 0.03 & 3.30 & \scriptsize $\pm$ 0.16 & \textbf{46.41}&\textbf{\scriptsize $\pm$ 0.15} & 3.03 & \scriptsize $\pm$ 0.11 & \textit{57.55}&\textit{\scriptsize $\pm$ 1.29} & 3.05 & \scriptsize $\pm$ 0.10 & 83.39 & \scriptsize $\pm$ 0.07 & 5.11 & \scriptsize $\pm$ 0.37 & \textbf{83.37}&\textbf{\scriptsize $\pm$ 0.03} & 5.20 & \scriptsize $\pm$ 0.31 \\
    Greedy$(\hat \eta)$ & & $\times$ & & $\times$ & \textit{31.16}&\textit{\scriptsize $\pm$ 0.06} & 3.45 & \scriptsize $\pm$ 0.20 & \textbf{30.63}&\textbf{\scriptsize $\pm$ 0.11} & 3.33 & \scriptsize $\pm$ 0.13 & \textbf{46.41}&\textbf{\scriptsize $\pm$ 0.15} & 3.06 & \scriptsize $\pm$ 0.12 & \textit{57.55}&\textit{\scriptsize $\pm$ 1.29} & 3.04 & \scriptsize $\pm$ 0.10 & \textbf{83.42}&\textbf{\scriptsize $\pm$ 0.03} & 5.05 & \scriptsize $\pm$ 0.37 & 83.24 & \scriptsize $\pm$ 0.03 & 5.12 & \scriptsize $\pm$ 0.29 \\
    Online-FW & \textbf{41.05}&\textbf{\scriptsize $\pm$ 0.01} & 10.82 & \scriptsize $\pm$ 1.13 & 30.60 & \scriptsize $\pm$ 0.06 & 17.19 & \scriptsize $\pm$ 0.46 & 29.58 & \scriptsize $\pm$ 0.08 & 14.36 & \scriptsize $\pm$ 0.81 & \textit{46.38}&\textit{\scriptsize $\pm$ 0.16} & 7.31 & \scriptsize $\pm$ 0.15 & 28.66 & \scriptsize $\pm$ 0.93 & 10.06 & \scriptsize $\pm$ 1.25 & 83.37 & \scriptsize $\pm$ 0.03 & 12.04 & \scriptsize $\pm$ 0.67 & \textit{83.28}&\textit{\scriptsize $\pm$ 0.05} & 17.76 & \scriptsize $\pm$ 1.18 \\
    Online-FW$(\hat \eta)$ & \textit{41.02}&\textit{\scriptsize $\pm$ 0.02} & 14.66 & \scriptsize $\pm$ 0.18 & \textbf{31.17}&\textbf{\scriptsize $\pm$ 0.06} & 18.17 & \scriptsize $\pm$ 0.24 & 29.65 & \scriptsize $\pm$ 0.21 & 41.54 & \scriptsize $\pm$ 3.51 & 46.28 & \scriptsize $\pm$ 0.08 & 10.30 & \scriptsize $\pm$ 0.09 & 25.73 & \scriptsize $\pm$ 0.82 & 31.20 & \scriptsize $\pm$ 0.83 & 83.37 & \scriptsize $\pm$ 0.02 & 9.35 & \scriptsize $\pm$ 1.03 & 83.20 & \scriptsize $\pm$ 0.03 & 37.08 & \scriptsize $\pm$ 0.84 \\
    \midrule
    \OMMA{} & 41.01 & \scriptsize $\pm$ 0.02 & 6.94 & \scriptsize $\pm$ 0.35 & 30.90 & \scriptsize $\pm$ 0.13 & 7.14 & \scriptsize $\pm$ 0.44 & 30.39 & \scriptsize $\pm$ 0.05 & 7.06 & \scriptsize $\pm$ 0.35 & \textbf{46.41}&\textbf{\scriptsize $\pm$ 0.15} & 5.83 & \scriptsize $\pm$ 0.19 & \textbf{58.35}&\textbf{\scriptsize $\pm$ 0.34} & 5.86 & \scriptsize $\pm$ 0.17 & 83.39 & \scriptsize $\pm$ 0.07 & 7.57 & \scriptsize $\pm$ 0.67 & \textbf{83.37}&\textbf{\scriptsize $\pm$ 0.03} & 7.91 & \scriptsize $\pm$ 0.44 \\
    \OMMAeta{} & \textit{41.02}&\textit{\scriptsize $\pm$ 0.03} & 7.05 & \scriptsize $\pm$ 0.43 & 31.15 & \scriptsize $\pm$ 0.08 & 6.97 & \scriptsize $\pm$ 0.38 & \textit{30.55}&\textit{\scriptsize $\pm$ 0.09} & 7.09 & \scriptsize $\pm$ 0.44 & 46.33 & \scriptsize $\pm$ 0.05 & 5.83 & \scriptsize $\pm$ 0.18 & 55.56 & \scriptsize $\pm$ 0.43 & 5.90 & \scriptsize $\pm$ 0.19 & \textit{83.41}&\textit{\scriptsize $\pm$ 0.03} & 7.51 & \scriptsize $\pm$ 0.58 & 83.23 & \scriptsize $\pm$ 0.04 & 8.08 & \scriptsize $\pm$ 0.61 \\
    \midrule
    Offline-FW & 40.96 & \scriptsize $\pm$ 0.00 & & $-$ & 30.80 & \scriptsize $\pm$ 0.04 & & $-$ & 30.26 & \scriptsize $\pm$ 0.09 & & $-$ & 46.45 & \scriptsize $\pm$ 0.00 & & $-$ & 39.27 & \scriptsize $\pm$ 1.14 & & $-$ & 83.44 & \scriptsize $\pm$ 0.01 & & $-$ & 83.22 & \scriptsize $\pm$ 0.01 & & $-$ \\
\midrule
& \multicolumn{28}{c}{\datasettable{RCV1X} ($m = 2456, n = 155962$)}  \\ 
 \midrule
    Top-$k$ / $\hat \eta \!>\!0.5$ & 68.57 & \scriptsize $\pm$ 0.00 & 12.72 & \scriptsize $\pm$ 2.65 & 11.29 & \scriptsize $\pm$ 0.00 & 13.48 & \scriptsize $\pm$ 4.35 & 5.34 & \scriptsize $\pm$ 0.00 & 1.44 & \scriptsize $\pm$ 0.02 & 4.59 & \scriptsize $\pm$ 0.00 & 1.45 & \scriptsize $\pm$ 0.03 & 13.24 & \scriptsize $\pm$ 0.00 & 1.44 & \scriptsize $\pm$ 0.04 & 16.01 & \scriptsize $\pm$ 0.00 & 12.01 & \scriptsize $\pm$ 2.14 & \textit{12.36}&\textit{\scriptsize $\pm$ 0.00} & 15.99 & \scriptsize $\pm$ 6.34 \\
    \midrule
    OFO & \textbf{69.83}&\textbf{\scriptsize $\pm$ 0.00} & 27.57 & \scriptsize $\pm$ 13.52 & 20.26 & \scriptsize $\pm$ 0.15 & 76.05 & \scriptsize $\pm$ 39.24 & & $\times$ & & $\times$ & & $\times$ & & $\times$ & & $\times$ & & $\times$ & & $\times$ & & $\times$ & & $\times$ & & $\times$ \\
    Greedy & & $\times$ & & $\times$ & \textbf{20.80}&\textbf{\scriptsize $\pm$ 0.16} & 91.71 & \scriptsize $\pm$ 34.37 & 16.01 & \scriptsize $\pm$ 0.12 & 27.04 & \scriptsize $\pm$ 0.43 & 21.20 & \scriptsize $\pm$ 0.08 & 25.58 & \scriptsize $\pm$ 0.38 & \textbf{30.91}&\textbf{\scriptsize $\pm$ 0.37} & 25.37 & \scriptsize $\pm$ 0.39 & \textbf{69.07}&\textbf{\scriptsize $\pm$ 0.00} & 183.85 & \scriptsize $\pm$ 69.54 & \textbf{67.04}&\textbf{\scriptsize $\pm$ 0.00} & 234.79 & \scriptsize $\pm$ 106.38 \\
    Greedy$(\hat \eta)$ & & $\times$ & & $\times$ & 20.50 & \scriptsize $\pm$ 0.09 & 57.47 & \scriptsize $\pm$ 19.59 & \textbf{16.09}&\textbf{\scriptsize $\pm$ 0.09} & 27.00 & \scriptsize $\pm$ 0.42 & 21.20 & \scriptsize $\pm$ 0.08 & 25.56 & \scriptsize $\pm$ 0.42 & \textbf{30.91}&\textbf{\scriptsize $\pm$ 0.37} & 25.20 & \scriptsize $\pm$ 0.56 & \textit{69.06}&\textit{\scriptsize $\pm$ 0.00} & 192.01 & \scriptsize $\pm$ 126.11 & \textbf{67.04}&\textbf{\scriptsize $\pm$ 0.00} & 233.71 & \scriptsize $\pm$ 102.62 \\
    Online-FW & \textbf{69.83}&\textbf{\scriptsize $\pm$ 0.00} & 30.37 & \scriptsize $\pm$ 1.42 & 19.82 & \scriptsize $\pm$ 0.16 & 36.40 & \scriptsize $\pm$ 0.99 & 15.33 & \scriptsize $\pm$ 0.09 & 52.14 & \scriptsize $\pm$ 1.35 & 21.09 & \scriptsize $\pm$ 0.15 & 36.73 & \scriptsize $\pm$ 0.62 & 19.88 & \scriptsize $\pm$ 0.21 & 62.42 & \scriptsize $\pm$ 6.57 & \textbf{69.07}&\textbf{\scriptsize $\pm$ 0.00} & 30.96 & \scriptsize $\pm$ 0.65 & \textbf{67.04}&\textbf{\scriptsize $\pm$ 0.00} & 40.11 & \scriptsize $\pm$ 1.61 \\
    Online-FW$(\hat \eta)$ & \textit{69.79}&\textit{\scriptsize $\pm$ 0.01} & 46.18 & \scriptsize $\pm$ 0.47 & 20.40 & \scriptsize $\pm$ 0.08 & 63.85 & \scriptsize $\pm$ 0.56 & 15.55 & \scriptsize $\pm$ 0.10 & 220.60 & \scriptsize $\pm$ 0.55 & \textit{22.21}&\textit{\scriptsize $\pm$ 0.08} & 47.24 & \scriptsize $\pm$ 0.49 & 22.17 & \scriptsize $\pm$ 0.34 & 205.95 & \scriptsize $\pm$ 0.77 & \textit{69.06}&\textit{\scriptsize $\pm$ 0.00} & 48.66 & \scriptsize $\pm$ 0.87 & \textbf{67.04}&\textbf{\scriptsize $\pm$ 0.00} & 68.64 & \scriptsize $\pm$ 1.10 \\
    \midrule
    \OMMA{} & 69.77 & \scriptsize $\pm$ 0.00 & 58.51 & \scriptsize $\pm$ 6.52 & 20.57 & \scriptsize $\pm$ 0.16 & 56.42 & \scriptsize $\pm$ 4.68 & 15.87 & \scriptsize $\pm$ 0.11 & 66.80 & \scriptsize $\pm$ 13.58 & 21.08 & \scriptsize $\pm$ 0.10 & 56.39 & \scriptsize $\pm$ 0.94 & \textit{30.39}&\textit{\scriptsize $\pm$ 0.21} & 53.22 & \scriptsize $\pm$ 8.16 & \textbf{69.07}&\textbf{\scriptsize $\pm$ 0.00} & 63.82 & \scriptsize $\pm$ 10.34 & \textbf{67.04}&\textbf{\scriptsize $\pm$ 0.00} & 64.56 & \scriptsize $\pm$ 5.64 \\
    \OMMAeta{} & 69.71 & \scriptsize $\pm$ 0.00 & 59.59 & \scriptsize $\pm$ 6.46 & \textit{20.71}&\textit{\scriptsize $\pm$ 0.11} & 60.31 & \scriptsize $\pm$ 9.50 & \textit{16.07}&\textit{\scriptsize $\pm$ 0.08} & 62.89 & \scriptsize $\pm$ 15.22 & \textbf{22.23}&\textbf{\scriptsize $\pm$ 0.09} & 56.86 & \scriptsize $\pm$ 0.86 & 30.38 & \scriptsize $\pm$ 0.24 & 49.32 & \scriptsize $\pm$ 4.36 & \textit{69.06}&\textit{\scriptsize $\pm$ 0.00} & 65.63 & \scriptsize $\pm$ 9.65 & \textbf{67.04}&\textbf{\scriptsize $\pm$ 0.00} & 69.29 & \scriptsize $\pm$ 8.55 \\
    \midrule
    Offline-FW & 69.84 & \scriptsize $\pm$ 0.00 & & $-$ & 15.08 & \scriptsize $\pm$ 0.05 & & $-$ & 11.74 & \scriptsize $\pm$ 0.05 & & $-$ & 22.22 & \scriptsize $\pm$ 0.00 & & $-$ & 20.98 & \scriptsize $\pm$ 0.71 & & $-$ & 69.48 & \scriptsize $\pm$ 0.00 & & $-$ & 67.50 & \scriptsize $\pm$ 0.00 & & $-$ \\
\midrule
& \multicolumn{28}{c}{\datasettable{AmazonCat} ($m = 13330, n = 306784$)}  \\ 
 \midrule
    Top-$k$ / $\hat \eta \!>\!0.5$ & 67.77 & \scriptsize $\pm$ 0.00 & 52.41 & \scriptsize $\pm$ 14.16 & 28.76 & \scriptsize $\pm$ 0.00 & 60.25 & \scriptsize $\pm$ 23.57 & 14.98 & \scriptsize $\pm$ 0.00 & 3.03 & \scriptsize $\pm$ 0.08 & 11.18 & \scriptsize $\pm$ 0.00 & 2.83 & \scriptsize $\pm$ 0.05 & 30.98 & \scriptsize $\pm$ 0.00 & 2.83 & \scriptsize $\pm$ 0.04 & \textit{33.93}&\textit{\scriptsize $\pm$ 0.00} & 77.78 & \scriptsize $\pm$ 35.26 & \textit{30.03}&\textit{\scriptsize $\pm$ 0.00} & 39.14 & \scriptsize $\pm$ 6.60 \\
    \midrule
    OFO & \textit{70.38}&\textit{\scriptsize $\pm$ 0.00} & 88.22 & \scriptsize $\pm$ 33.36 & 39.60 & \scriptsize $\pm$ 0.06 & 154.54 & \scriptsize $\pm$ 62.87 & & $\times$ & & $\times$ & & $\times$ & & $\times$ & & $\times$ & & $\times$ & & $\times$ & & $\times$ & & $\times$ & & $\times$ \\
    Greedy & & $\times$ & & $\times$ & 44.20 & \scriptsize $\pm$ 0.06 & 175.96 & \scriptsize $\pm$ 48.04 & 43.64 & \scriptsize $\pm$ 0.05 & 52.42 & \scriptsize $\pm$ 0.87 & 57.81 & \scriptsize $\pm$ 0.01 & 49.01 & \scriptsize $\pm$ 1.22 & \textbf{54.00}&\textbf{\scriptsize $\pm$ 0.07} & 48.90 & \scriptsize $\pm$ 1.10 & \textbf{80.86}&\textbf{\scriptsize $\pm$ 0.00} & 415.92 & \scriptsize $\pm$ 182.65 & \textbf{80.04}&\textbf{\scriptsize $\pm$ 0.00} & 449.12 & \scriptsize $\pm$ 231.23 \\
    Greedy$(\hat \eta)$ & & $\times$ & & $\times$ & 44.08 & \scriptsize $\pm$ 0.04 & 127.17 & \scriptsize $\pm$ 28.98 & \textbf{46.11}&\textbf{\scriptsize $\pm$ 0.06} & 52.50 & \scriptsize $\pm$ 0.95 & 57.81 & \scriptsize $\pm$ 0.01 & 50.08 & \scriptsize $\pm$ 0.62 & \textbf{54.00}&\textbf{\scriptsize $\pm$ 0.07} & 48.58 & \scriptsize $\pm$ 0.72 & \textbf{80.86}&\textbf{\scriptsize $\pm$ 0.00} & 465.89 & \scriptsize $\pm$ 144.03 & \textbf{80.04}&\textbf{\scriptsize $\pm$ 0.00} & 597.36 & \scriptsize $\pm$ 209.20 \\
    Online-FW & \textbf{70.61}&\textbf{\scriptsize $\pm$ 0.01} & 44.76 & \scriptsize $\pm$ 1.36 & 42.42 & \scriptsize $\pm$ 0.03 & 60.33 & \scriptsize $\pm$ 1.26 & 40.20 & \scriptsize $\pm$ 0.04 & 69.75 & \scriptsize $\pm$ 0.46 & 57.74 & \scriptsize $\pm$ 0.03 & 73.35 & \scriptsize $\pm$ 0.52 & 40.10 & \scriptsize $\pm$ 0.18 & 79.41 & \scriptsize $\pm$ 0.97 & \textbf{80.86}&\textbf{\scriptsize $\pm$ 0.00} & 63.70 & \scriptsize $\pm$ 0.94 & \textbf{80.04}&\textbf{\scriptsize $\pm$ 0.00} & 83.86 & \scriptsize $\pm$ 1.95 \\
    Online-FW$(\hat \eta)$ & 70.34 & \scriptsize $\pm$ 0.00 & 88.21 & \scriptsize $\pm$ 1.10 & \textit{47.32}&\textit{\scriptsize $\pm$ 0.06} & 111.40 & \scriptsize $\pm$ 1.59 & 44.63 & \scriptsize $\pm$ 0.05 & 425.18 & \scriptsize $\pm$ 1.53 & \textit{58.87}&\textit{\scriptsize $\pm$ 0.04} & 89.94 & \scriptsize $\pm$ 0.38 & 49.68 & \scriptsize $\pm$ 0.14 & 416.26 & \scriptsize $\pm$ 1.83 & \textbf{80.86}&\textbf{\scriptsize $\pm$ 0.00} & 97.05 & \scriptsize $\pm$ 1.06 & \textbf{80.04}&\textbf{\scriptsize $\pm$ 0.00} & 129.23 & \scriptsize $\pm$ 1.45 \\
    \midrule
    \OMMA{} & 69.90 & \scriptsize $\pm$ 0.00 & 107.61 & \scriptsize $\pm$ 5.61 & 43.04 & \scriptsize $\pm$ 0.05 & 106.79 & \scriptsize $\pm$ 5.02 & 42.14 & \scriptsize $\pm$ 0.05 & 114.83 & \scriptsize $\pm$ 11.24 & 57.80 & \scriptsize $\pm$ 0.03 & 91.40 & \scriptsize $\pm$ 1.46 & \textit{52.04}&\textit{\scriptsize $\pm$ 0.15} & 92.92 & \scriptsize $\pm$ 3.23 & \textbf{80.86}&\textbf{\scriptsize $\pm$ 0.00} & 129.35 & \scriptsize $\pm$ 16.42 & \textbf{80.04}&\textbf{\scriptsize $\pm$ 0.00} & 120.60 & \scriptsize $\pm$ 5.64 \\
    \OMMAeta{} & 70.02 & \scriptsize $\pm$ 0.00 & 106.94 & \scriptsize $\pm$ 4.27 & \textbf{47.67}&\textbf{\scriptsize $\pm$ 0.03} & 109.30 & \scriptsize $\pm$ 6.82 & \textit{45.06}&\textit{\scriptsize $\pm$ 0.05} & 111.92 & \scriptsize $\pm$ 9.07 & \textbf{58.89}&\textbf{\scriptsize $\pm$ 0.04} & 91.80 & \scriptsize $\pm$ 0.94 & 51.52 & \scriptsize $\pm$ 0.13 & 93.68 & \scriptsize $\pm$ 3.58 & \textbf{80.86}&\textbf{\scriptsize $\pm$ 0.00} & 118.54 & \scriptsize $\pm$ 6.60 & \textbf{80.04}&\textbf{\scriptsize $\pm$ 0.00} & 121.44 & \scriptsize $\pm$ 5.28 \\
    \midrule
    Offline-FW & 70.43 & \scriptsize $\pm$ 0.01 & & $-$ & 38.20 & \scriptsize $\pm$ 0.09 & & $-$ & 39.38 & \scriptsize $\pm$ 0.06 & & $-$ & 62.39 & \scriptsize $\pm$ 0.00 & & $-$ & 42.01 & \scriptsize $\pm$ 0.61 & & $-$ & 83.23 & \scriptsize $\pm$ 0.00 & & $-$ & 82.58 & \scriptsize $\pm$ 0.00 & & $-$ \\
\bottomrule
\end{tabular}
}
\end{table*}

\begin{table*}[ht]
    \caption{Predictive performance and running times of the different online algorithms on \emph{multi-class} problems, averaged over 5 runs and reported with standard deviations ($\pm$)s. The best result on each metric is in \textbf{bold}, the second best is in \textit{italic}. We additionally report basic statistics of the benchmarks: number of classes/labels and instances in the test sequence. $\times$ -- means that the algorithm does not support the optimization of that metric. The inference time of Offline-FW was not measured.}
    \label{tab:extended-results-multi-class}
    \vspace{8pt}
    \small
    \centering

\resizebox{\linewidth}{!}{
\setlength\tabcolsep{3 pt}
\begin{tabular}{l|r@{}lr@{}l|r@{}lr@{}l|r@{}lr@{}l|r@{}lr@{}l|r@{}lr@{}l|r@{}lr@{}l|r@{}lr@{}l}
\toprule
    Method & \multicolumn{4}{c|}{Macro F1} & \multicolumn{4}{c|}{Macro F1$@3$} & \multicolumn{4}{c|}{Macro Recall$@3$} & \multicolumn{4}{c|}{Macro Precision$@3$} & \multicolumn{4}{c|}{Multi-class G-Mean} & \multicolumn{4}{c|}{Multi-class H-Mean}& \multicolumn{4}{c}{Multi-class Q-Mean} \\
    & \multicolumn{2}{c}{(\%)} & \multicolumn{2}{c|}{time (s)} & \multicolumn{2}{c}{(\%)} & \multicolumn{2}{c|}{time (s)} & \multicolumn{2}{c}{(\%)} & \multicolumn{2}{c|}{time (s)} & \multicolumn{2}{c}{(\%)} & \multicolumn{2}{c|}{time (s)} & \multicolumn{2}{c}{(\%)} & \multicolumn{2}{c|}{time (s)} & \multicolumn{2}{c}{(\%)} & \multicolumn{2}{c|}{time (s)} & \multicolumn{2}{c}{(\%)} & \multicolumn{2}{c}{time (s)} \\
\midrule
& \multicolumn{28}{c}{\datasettable{News20} ($m = 20, n = 7532$)}  \\ 
 \midrule
    Top-$k$ & \textbf{83.37}&\textbf{\scriptsize $\pm$ 0.00} & 0.06 & \scriptsize $\pm$ 0.00 & 49.23 & \scriptsize $\pm$ 0.00 & 0.08 & \scriptsize $\pm$ 0.00 & 94.78 & \scriptsize $\pm$ 0.00 & 0.06 & \scriptsize $\pm$ 0.00 & 33.95 & \scriptsize $\pm$ 0.00 & 0.07 & \scriptsize $\pm$ 0.01 & 82.51 & \scriptsize $\pm$ 0.00 & 0.07 & \scriptsize $\pm$ 0.02 & 81.66 & \scriptsize $\pm$ 0.00 & 0.09 & \scriptsize $\pm$ 0.01 & 80.12 & \scriptsize $\pm$ 0.00 & 0.07 & \scriptsize $\pm$ 0.02 \\
    \midrule
    Greedy & 82.44 & \scriptsize $\pm$ 0.44 & 1.30 & \scriptsize $\pm$ 0.04 & 72.69 & \scriptsize $\pm$ 0.71 & 1.27 & \scriptsize $\pm$ 0.03 & \textit{94.89}&\textit{\scriptsize $\pm$ 0.06} & 1.28 & \scriptsize $\pm$ 0.05 & \textbf{83.89}&\textbf{\scriptsize $\pm$ 0.38} & 1.22 & \scriptsize $\pm$ 0.07 & & $\times$ & & $\times$ & & $\times$ & & $\times$ & & $\times$ & & $\times$ \\
    Greedy$(\hat \eta)$ & \textit{83.08}&\textit{\scriptsize $\pm$ 0.09} & 1.29 & \scriptsize $\pm$ 0.03 & \textit{73.37}&\textit{\scriptsize $\pm$ 0.60} & 1.27 & \scriptsize $\pm$ 0.03 & \textit{94.89}&\textit{\scriptsize $\pm$ 0.06} & 1.32 & \scriptsize $\pm$ 0.05 & \textbf{83.89}&\textbf{\scriptsize $\pm$ 0.38} & 1.17 & \scriptsize $\pm$ 0.01 & & $\times$ & & $\times$ & & $\times$ & & $\times$ & & $\times$ & & $\times$ \\
    Online-FW & 82.88 & \scriptsize $\pm$ 0.05 & 6.67 & \scriptsize $\pm$ 0.31 & 54.10 & \scriptsize $\pm$ 1.09 & 8.18 & \scriptsize $\pm$ 0.24 & 94.88 & \scriptsize $\pm$ 0.05 & 6.76 & \scriptsize $\pm$ 0.24 & 19.90 & \scriptsize $\pm$ 2.95 & 6.66 & \scriptsize $\pm$ 0.46 & 82.68 & \scriptsize $\pm$ 0.15 & 1.93 & \scriptsize $\pm$ 0.42 & \textit{82.35}&\textit{\scriptsize $\pm$ 0.04} & 8.93 & \scriptsize $\pm$ 1.65 & 80.95 & \scriptsize $\pm$ 0.06 & 8.17 & \scriptsize $\pm$ 0.34 \\
    Online-FW$(\hat \eta)$ & 82.55 & \scriptsize $\pm$ 0.27 & 19.46 & \scriptsize $\pm$ 1.66 & 56.06 & \scriptsize $\pm$ 3.04 & 13.04 & \scriptsize $\pm$ 0.26 & \textbf{95.01}&\textbf{\scriptsize $\pm$ 0.02} & 6.85 & \scriptsize $\pm$ 0.18 & 15.58 & \scriptsize $\pm$ 4.48 & 11.69 & \scriptsize $\pm$ 0.74 & \textbf{82.78}&\textbf{\scriptsize $\pm$ 0.05} & 1.92 & \scriptsize $\pm$ 0.43 & \textbf{82.57}&\textbf{\scriptsize $\pm$ 0.10} & 21.99 & \scriptsize $\pm$ 5.32 & \textit{80.97}&\textit{\scriptsize $\pm$ 0.10} & 22.87 & \scriptsize $\pm$ 1.17 \\
    \midrule
    OMMA & 82.11 & \scriptsize $\pm$ 0.35 & 3.51 & \scriptsize $\pm$ 0.08 & 72.84 & \scriptsize $\pm$ 0.61 & 3.53 & \scriptsize $\pm$ 0.07 & \textit{94.89}&\textit{\scriptsize $\pm$ 0.05} & 2.83 & \scriptsize $\pm$ 0.14 & \textit{83.88}&\textit{\scriptsize $\pm$ 0.51} & 3.19 & \scriptsize $\pm$ 0.58 & 82.72 & \scriptsize $\pm$ 0.15 & 3.15 & \scriptsize $\pm$ 0.66 & 82.08 & \scriptsize $\pm$ 0.14 & 3.57 & \scriptsize $\pm$ 0.60 & 80.95 & \scriptsize $\pm$ 0.08 & 3.81 & \scriptsize $\pm$ 0.40 \\
    OMMA$(\hat \eta)$ & 83.07 & \scriptsize $\pm$ 0.11 & 3.57 & \scriptsize $\pm$ 0.07 & \textbf{73.41}&\textbf{\scriptsize $\pm$ 1.06} & 3.53 & \scriptsize $\pm$ 0.03 & \textbf{95.01}&\textbf{\scriptsize $\pm$ 0.03} & 3.11 & \scriptsize $\pm$ 0.21 & 83.30 & \scriptsize $\pm$ 0.43 & 3.16 & \scriptsize $\pm$ 0.56 & \textit{82.77}&\textit{\scriptsize $\pm$ 0.06} & 3.44 & \scriptsize $\pm$ 0.81 & 82.18 & \scriptsize $\pm$ 0.10 & 3.45 & \scriptsize $\pm$ 0.56 & \textbf{81.00}&\textbf{\scriptsize $\pm$ 0.09} & 3.56 & \scriptsize $\pm$ 0.07 \\
    \midrule
    Offline-FW & 83.31 & \scriptsize $\pm$ 0.02 & & $-$ & 72.37 & \scriptsize $\pm$ 0.28 & & $-$ & 94.93 & \scriptsize $\pm$ 0.00 & & $-$ & 80.87 & \scriptsize $\pm$ 1.48 & & $-$ & 82.61 & \scriptsize $\pm$ 0.11 & & $-$ & 82.07 & \scriptsize $\pm$ 0.01 & & $-$ & 78.52 & \scriptsize $\pm$ 0.59 & & $-$ \\
\midrule
& \multicolumn{28}{c}{\datasettable{Ledgar-LexGlue} ($m = 100, n = 10000$)}  \\ 
 \midrule
    Top-$k$ & 79.06 & \scriptsize $\pm$ 0.00 & 0.09 & \scriptsize $\pm$ 0.00 & 51.80 & \scriptsize $\pm$ 0.00 & 0.12 & \scriptsize $\pm$ 0.00 & 92.04 & \scriptsize $\pm$ 0.00 & 0.10 & \scriptsize $\pm$ 0.00 & 38.88 & \scriptsize $\pm$ 0.00 & 0.10 & \scriptsize $\pm$ 0.01 & 0.00 & \scriptsize $\pm$ 0.00 & 0.10 & \scriptsize $\pm$ 0.02 & 0.00 & \scriptsize $\pm$ 0.00 & 0.12 & \scriptsize $\pm$ 0.03 & 69.21 & \scriptsize $\pm$ 0.00 & 0.09 & \scriptsize $\pm$ 0.00 \\
    \midrule
    Greedy & 79.30 & \scriptsize $\pm$ 0.03 & 1.71 & \scriptsize $\pm$ 0.05 & 78.08 & \scriptsize $\pm$ 0.29 & 1.81 & \scriptsize $\pm$ 0.03 & 93.26 & \scriptsize $\pm$ 0.23 & 1.69 & \scriptsize $\pm$ 0.06 & \textit{91.17}&\textit{\scriptsize $\pm$ 0.58} & 1.70 & \scriptsize $\pm$ 0.08 & & $\times$ & & $\times$ & & $\times$ & & $\times$ & & $\times$ & & $\times$ \\
    Greedy$(\hat \eta)$ & \textbf{79.35}&\textbf{\scriptsize $\pm$ 0.24} & 1.72 & \scriptsize $\pm$ 0.02 & \textit{78.21}&\textit{\scriptsize $\pm$ 0.29} & 1.75 & \scriptsize $\pm$ 0.03 & 93.26 & \scriptsize $\pm$ 0.23 & 1.64 & \scriptsize $\pm$ 0.05 & \textit{91.17}&\textit{\scriptsize $\pm$ 0.58} & 1.69 & \scriptsize $\pm$ 0.24 & & $\times$ & & $\times$ & & $\times$ & & $\times$ & & $\times$ & & $\times$ \\
    Online-FW & 79.22 & \scriptsize $\pm$ 0.10 & 7.21 & \scriptsize $\pm$ 1.60 & 70.94 & \scriptsize $\pm$ 0.80 & 7.32 & \scriptsize $\pm$ 0.21 & 93.30 & \scriptsize $\pm$ 0.28 & 6.31 & \scriptsize $\pm$ 0.09 & 54.52 & \scriptsize $\pm$ 0.74 & 8.04 & \scriptsize $\pm$ 1.97 & 62.31 & \scriptsize $\pm$ 31.15 & 3.06 & \scriptsize $\pm$ 0.47 & 75.81&\scriptsize $\pm$ 0.67 & 8.73 & \scriptsize $\pm$ 1.86 & \textbf{74.59}&\textbf{\scriptsize $\pm$ 0.31} & 5.49 & \scriptsize $\pm$ 0.67 \\
    Online-FW$(\hat \eta)$ & 79.22 & \scriptsize $\pm$ 0.09 & 14.62 & \scriptsize $\pm$ 3.47 & 73.85 & \scriptsize $\pm$ 0.28 & 13.48 & \scriptsize $\pm$ 1.04 & \textit{93.38}&\textit{\scriptsize $\pm$ 0.10} & 8.56 & \scriptsize $\pm$ 0.63 & 49.63 & \scriptsize $\pm$ 2.36 & 9.97 & \scriptsize $\pm$ 0.19 & \textit{78.02}&\textit{\scriptsize $\pm$ 0.33} & 3.14 & \scriptsize $\pm$ 0.64 & \textbf{77.58}&\textbf{\scriptsize $\pm$ 0.31} & 37.99 & \scriptsize $\pm$ 4.17 & 74.50 & \scriptsize $\pm$ 0.32 & 41.84 & \scriptsize $\pm$ 3.15 \\
    \midrule
    OMMA & 79.28 & \scriptsize $\pm$ 0.04 & 4.50 & \scriptsize $\pm$ 0.17 & 78.10 & \scriptsize $\pm$ 0.27 & 5.05 & \scriptsize $\pm$ 0.08 & 93.26 & \scriptsize $\pm$ 0.23 & 4.12 & \scriptsize $\pm$ 0.29 & \textbf{91.41}&\textbf{\scriptsize $\pm$ 0.75} & 3.84 & \scriptsize $\pm$ 0.27 & 77.48 & \scriptsize $\pm$ 0.61 & 3.96 & \scriptsize $\pm$ 0.75 & 74.62 & \scriptsize $\pm$ 2.49 & 5.20 & \scriptsize $\pm$ 1.14 & \textbf{74.59}&\textbf{\scriptsize $\pm$ 0.25} & 4.76 & \scriptsize $\pm$ 0.10 \\
    OMMA$(\hat \eta)$ & \textit{79.34}&\textit{\scriptsize $\pm$ 0.26} & 4.70 & \scriptsize $\pm$ 0.09 & \textbf{78.22}&\textbf{\scriptsize $\pm$ 0.30} & 5.11 & \scriptsize $\pm$ 0.11 & \textbf{93.39}&\textbf{\scriptsize $\pm$ 0.09} & 4.02 & \scriptsize $\pm$ 0.27 & 90.10 & \scriptsize $\pm$ 0.33 & 4.01 & \scriptsize $\pm$ 0.45 & \textbf{78.03}&\textbf{\scriptsize $\pm$ 0.32} & 4.01 & \scriptsize $\pm$ 0.69 & \textit{76.08}&\textit{\scriptsize $\pm$ 0.37} & 4.54 & \scriptsize $\pm$ 0.09 & \textit{74.53}&\textit{\scriptsize $\pm$ 0.24} & 4.84 & \scriptsize $\pm$ 0.05 \\
    \midrule
    Offline-FW & 78.76 & \scriptsize $\pm$ 0.13 & & $-$ & 77.82 & \scriptsize $\pm$ 0.51 & & $-$ & 93.46 & \scriptsize $\pm$ 0.00 & & $-$ & 82.78 & \scriptsize $\pm$ 1.16 & & $-$ & 77.33 & \scriptsize $\pm$ 0.36 & & $-$ & 75.12 & \scriptsize $\pm$ 0.07 & & $-$ & 74.25 & \scriptsize $\pm$ 0.13 & & $-$ \\
\midrule
& \multicolumn{28}{c}{\datasettable{Caltech-256} ($m = 256, n = 14890$)}  \\ 
 \midrule
    Top-$k$ & 79.45 & \scriptsize $\pm$ 0.00 & 0.16 & \scriptsize $\pm$ 0.00 & 46.82 & \scriptsize $\pm$ 0.00 & 0.18 & \scriptsize $\pm$ 0.00 & 89.85 & \scriptsize $\pm$ 0.00 & 0.18 & \scriptsize $\pm$ 0.00 & 32.58 & \scriptsize $\pm$ 0.00 & 0.20 & \scriptsize $\pm$ 0.01 & 77.32 & \scriptsize $\pm$ 0.00 & 0.17 & \scriptsize $\pm$ 0.01 & 75.69 & \scriptsize $\pm$ 0.00 & 0.18 & \scriptsize $\pm$ 0.00 & 74.53 & \scriptsize $\pm$ 0.00 & 0.16 & \scriptsize $\pm$ 0.00 \\
    \midrule
    Greedy & 79.59 & \scriptsize $\pm$ 0.07 & 2.60 & \scriptsize $\pm$ 0.04 & 79.02 & \scriptsize $\pm$ 0.10 & 2.69 & \scriptsize $\pm$ 0.04 & 90.20 & \scriptsize $\pm$ 0.07 & 2.72 & \scriptsize $\pm$ 0.07 & 96.61 & \scriptsize $\pm$ 0.32 & 2.93 & \scriptsize $\pm$ 0.29 & & $\times$ & & $\times$ & & $\times$ & & $\times$ & & $\times$ & & $\times$ \\
    Greedy$(\hat \eta)$ & \textbf{79.70}&\textbf{\scriptsize $\pm$ 0.06} & 2.71 & \scriptsize $\pm$ 0.05 & \textbf{79.14}&\textbf{\scriptsize $\pm$ 0.07} & 2.74 & \scriptsize $\pm$ 0.07 & 90.20 & \scriptsize $\pm$ 0.07 & 2.79 & \scriptsize $\pm$ 0.15 & 96.61 & \scriptsize $\pm$ 0.32 & 2.88 & \scriptsize $\pm$ 0.39 & & $\times$ & & $\times$ & & $\times$ & & $\times$ & & $\times$ & & $\times$ \\
    Online-FW & 79.15 & \scriptsize $\pm$ 0.16 & 10.07 & \scriptsize $\pm$ 0.42 & 70.29 & \scriptsize $\pm$ 0.77 & 12.46 & \scriptsize $\pm$ 1.39 & 90.20 & \scriptsize $\pm$ 0.05 & 57.84 & \scriptsize $\pm$ 2.57 & 63.22 & \scriptsize $\pm$ 0.61 & 12.08 & \scriptsize $\pm$ 1.01 & 78.31 & \scriptsize $\pm$ 0.14 & 3.38 & \scriptsize $\pm$ 0.04 & \textit{77.99}&\textit{\scriptsize $\pm$ 0.12} & 16.06 & \scriptsize $\pm$ 1.98 & 76.00 & \scriptsize $\pm$ 0.11 & 15.86 & \scriptsize $\pm$ 2.15 \\
    Online-FW$(\hat \eta)$ & 79.29 & \scriptsize $\pm$ 0.08 & 52.60 & \scriptsize $\pm$ 2.07 & 72.97 & \scriptsize $\pm$ 0.85 & 28.32 & \scriptsize $\pm$ 2.72 & \textit{90.34}&\textit{\scriptsize $\pm$ 0.03} & 21.59 & \scriptsize $\pm$ 1.07 & 65.42 & \scriptsize $\pm$ 1.30 & 22.67 & \scriptsize $\pm$ 1.24 & \textbf{78.41}&\textbf{\scriptsize $\pm$ 0.12} & 3.94 & \scriptsize $\pm$ 0.03 & \textbf{78.08}&\textbf{\scriptsize $\pm$ 0.12} & 53.02 & \scriptsize $\pm$ 0.23 & \textit{76.07}&\textit{\scriptsize $\pm$ 0.12} & 52.89 & \scriptsize $\pm$ 0.17 \\
    \midrule
    OMMA & 79.54 & \scriptsize $\pm$ 0.08 & 7.01 & \scriptsize $\pm$ 0.10 & 78.96 & \scriptsize $\pm$ 0.13 & 7.28 & \scriptsize $\pm$ 0.16 & 90.20 & \scriptsize $\pm$ 0.07 & 5.71 & \scriptsize $\pm$ 0.10 & \textit{96.77}&\textit{\scriptsize $\pm$ 0.36} & 7.13 & \scriptsize $\pm$ 0.46 & 78.33 & \scriptsize $\pm$ 0.13 & 5.61 & \scriptsize $\pm$ 0.12 & 77.12 & \scriptsize $\pm$ 0.17 & 6.42 & \scriptsize $\pm$ 0.20 & \textbf{76.15}&\textbf{\scriptsize $\pm$ 0.07} & 6.99 & \scriptsize $\pm$ 0.17 \\
    OMMA$(\hat \eta)$ & \textit{79.66}&\textit{\scriptsize $\pm$ 0.04} & 7.11 & \scriptsize $\pm$ 0.09 & \textit{79.11}&\textit{\scriptsize $\pm$ 0.06} & 7.23 & \scriptsize $\pm$ 0.08 & \textbf{90.35}&\textbf{\scriptsize $\pm$ 0.03} & 5.88 & \scriptsize $\pm$ 0.18 & \textbf{96.96}&\textbf{\scriptsize $\pm$ 0.12} & 7.16 & \scriptsize $\pm$ 0.54 & \textit{78.36}&\textit{\scriptsize $\pm$ 0.11} & 5.68 & \scriptsize $\pm$ 0.13 & 77.01 & \scriptsize $\pm$ 0.10 & 6.35 & \scriptsize $\pm$ 0.06 & 75.99 & \scriptsize $\pm$ 0.09 & 7.13 & \scriptsize $\pm$ 0.14 \\
    \midrule
    Offline-FW & 78.40 & \scriptsize $\pm$ 0.00 & & $-$ & 77.74 & \scriptsize $\pm$ 0.07 & & $-$ & 90.28 & \scriptsize $\pm$ 0.00 & & $-$ & 91.34 & \scriptsize $\pm$ 1.24 & & $-$ & 77.88 & \scriptsize $\pm$ 0.09 & & $-$ & 76.59 & \scriptsize $\pm$ 0.00 & & $-$ & 75.72 & \scriptsize $\pm$ 0.00 & & $-$ \\
\bottomrule
\end{tabular}
}
\end{table*}

\begin{figure*}[ht]
    \centering
    \scriptsize
    
    \vspace{5pt}\datasettable{YouTube}\vspace{5pt}
    
    \includegraphics[width=0.135\textwidth]{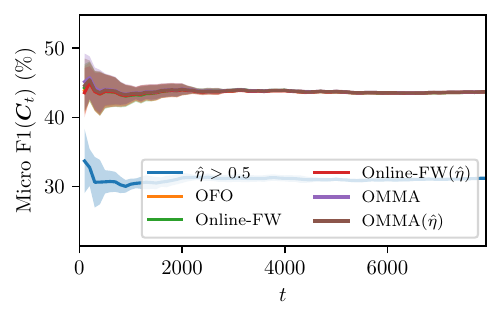}
    \includegraphics[width=0.135\textwidth]{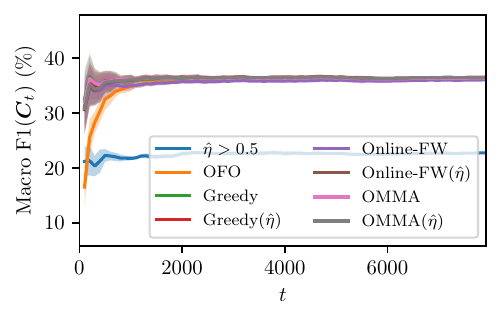}
    \includegraphics[width=0.135\textwidth]{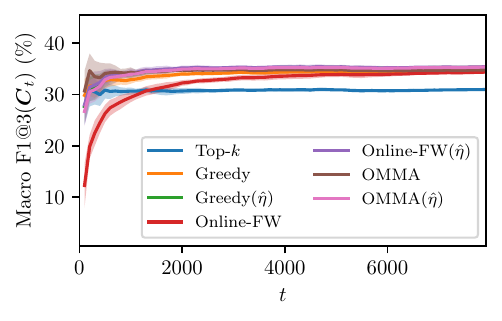}
    \includegraphics[width=0.135\textwidth]{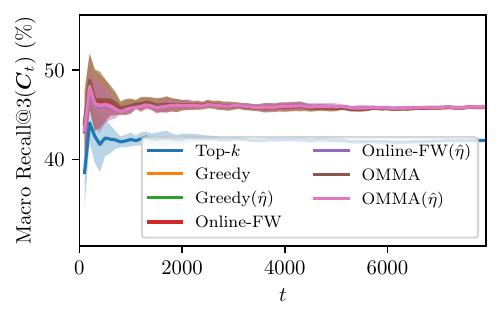}
    \includegraphics[width=0.135\textwidth]{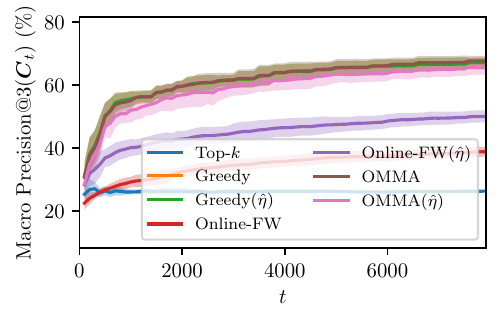}
    \includegraphics[width=0.135\textwidth]{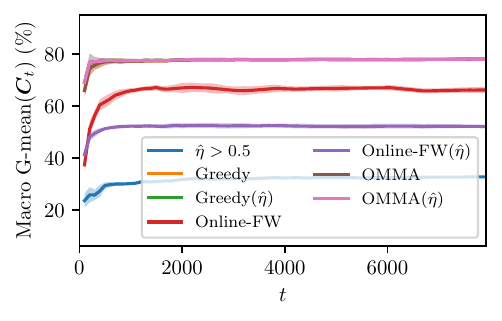}
    \includegraphics[width=0.135\textwidth]{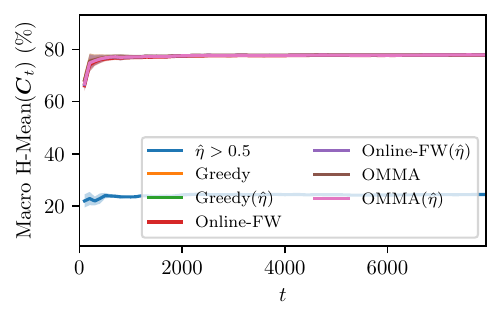}

    \vspace{5pt}\datasettable{Eurlex-LexGlue}\vspace{5pt}
    
    \includegraphics[width=0.135\textwidth]{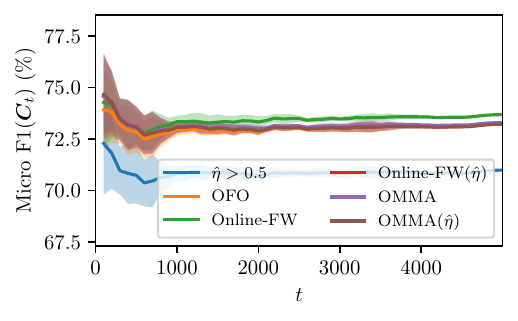}
    \includegraphics[width=0.135\textwidth]{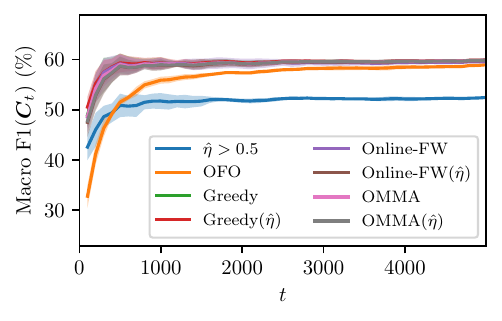}
    \includegraphics[width=0.135\textwidth]{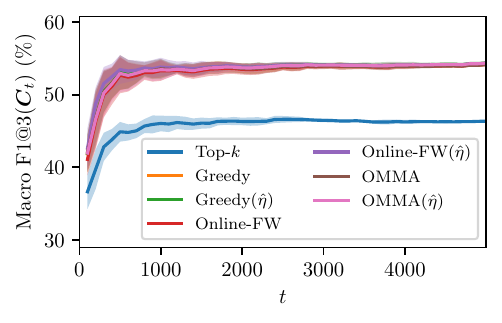}
    \includegraphics[width=0.135\textwidth]{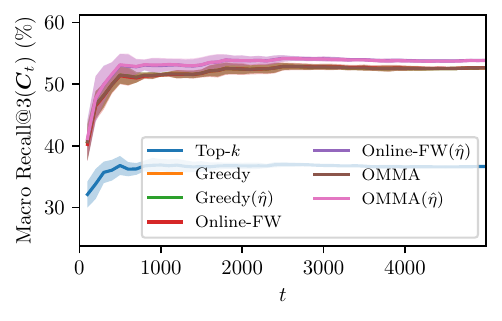}
    \includegraphics[width=0.135\textwidth]{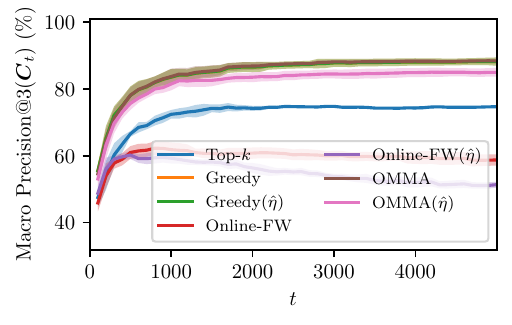}
    \includegraphics[width=0.135\textwidth]{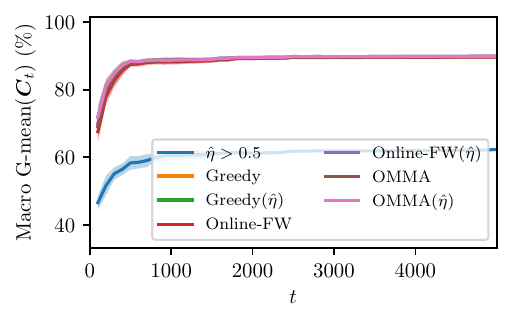}
    \includegraphics[width=0.135\textwidth]{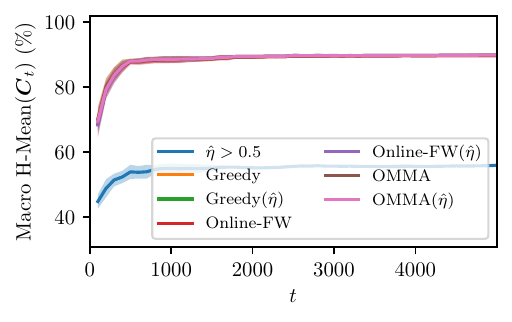}

    \vspace{5pt}\datasettable{Mediamill}\vspace{5pt}

    \includegraphics[width=0.135\textwidth]{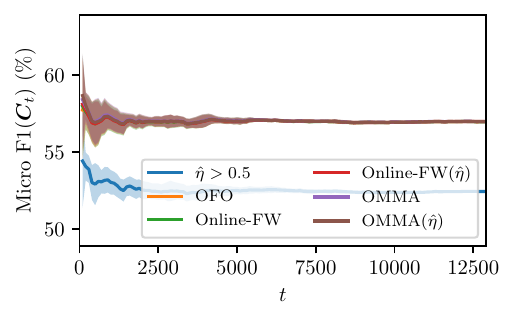}
    \includegraphics[width=0.135\textwidth]{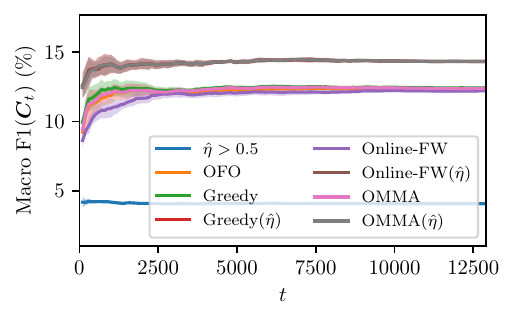}
    \includegraphics[width=0.135\textwidth]{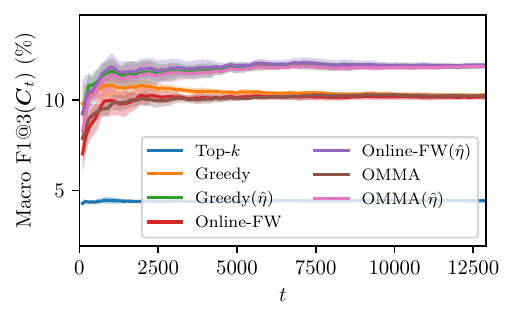}
    \includegraphics[width=0.135\textwidth]{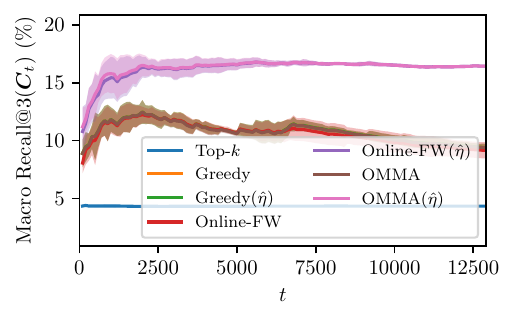}
    \includegraphics[width=0.135\textwidth]{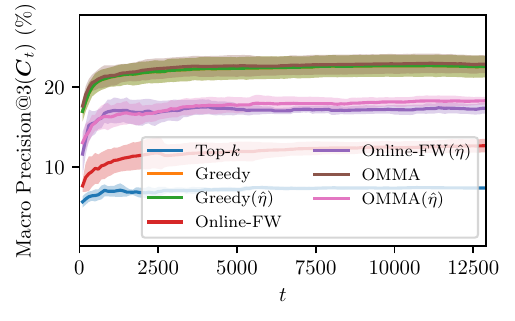}
    \includegraphics[width=0.135\textwidth]{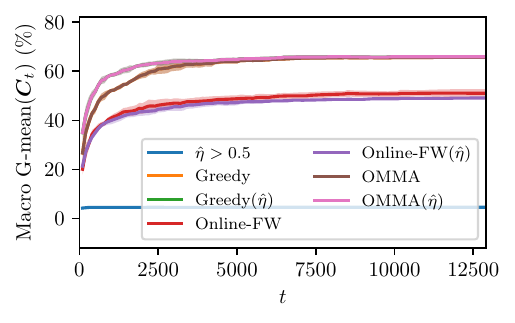}
    \includegraphics[width=0.135\textwidth]{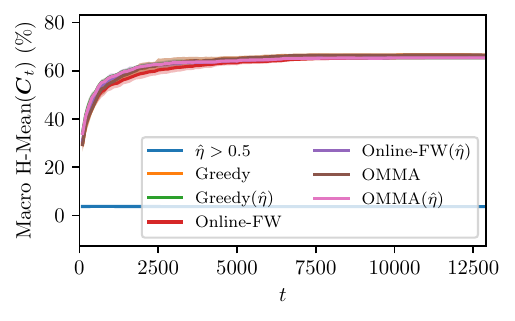}

    \vspace{5pt}\datasettable{Flickr}\vspace{5pt}
    
    \includegraphics[width=0.135\textwidth]{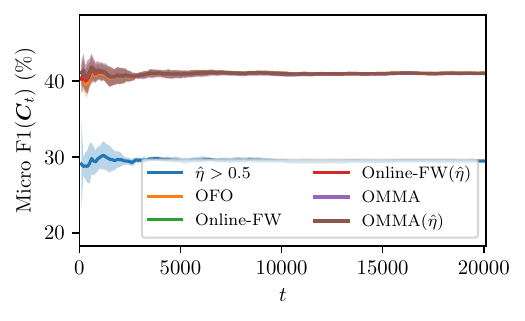}
    \includegraphics[width=0.135\textwidth]{figs/plots-v2/flicker_deepwalk_plt_macro_f1_k=0_pred_utility_history.pdf}
    \includegraphics[width=0.135\textwidth]{figs/plots-v2/flicker_deepwalk_plt_macro_f1_k=3_pred_utility_history.pdf}
    \includegraphics[width=0.135\textwidth]{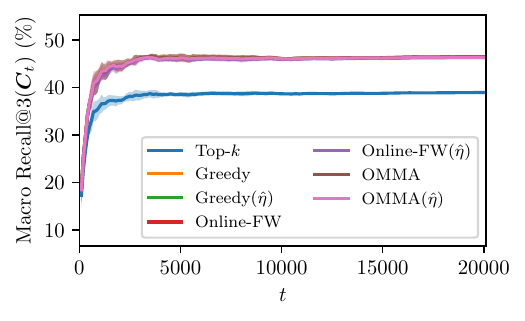}
    \includegraphics[width=0.135\textwidth]{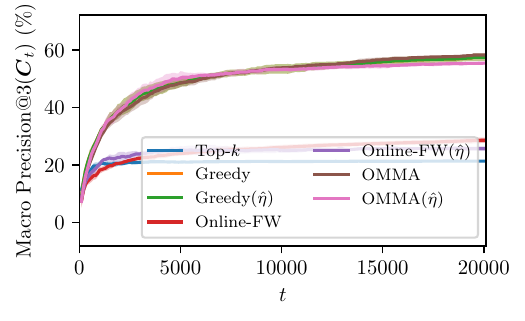}
    \includegraphics[width=0.135\textwidth]{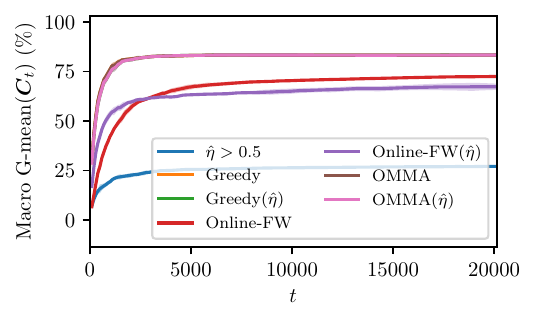}
    \includegraphics[width=0.135\textwidth]{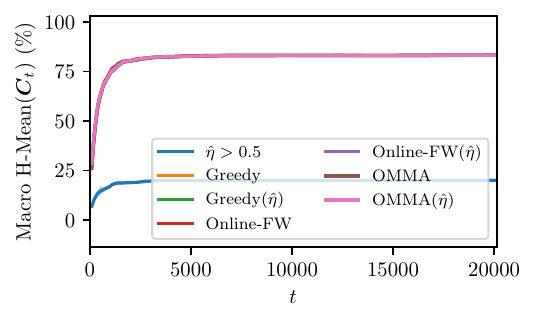}

    \vspace{5pt}\datasettable{RCV1X}\vspace{5pt}
    
    \includegraphics[width=0.135\textwidth]{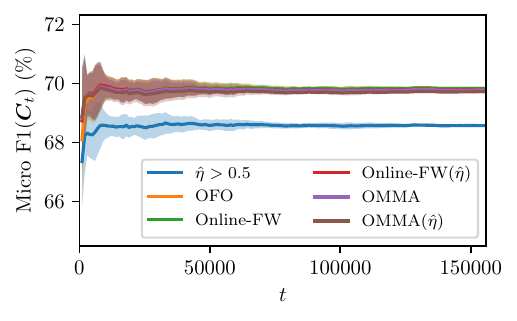}
    \includegraphics[width=0.135\textwidth]{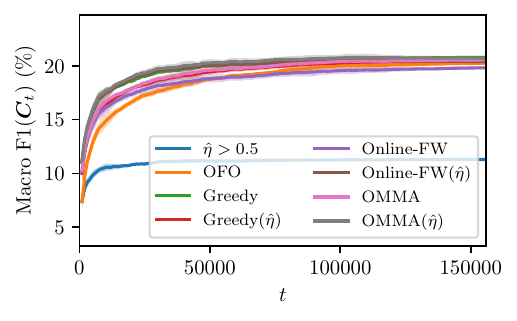}
    \includegraphics[width=0.135\textwidth]{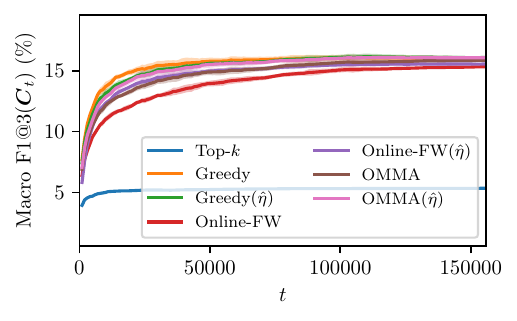}
    \includegraphics[width=0.135\textwidth]{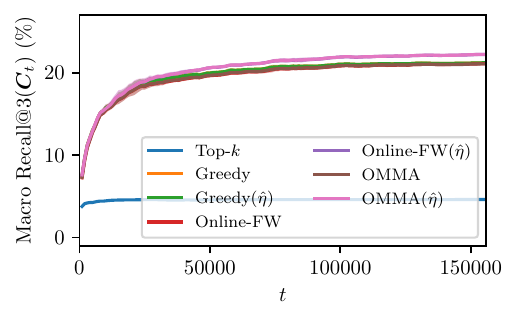}
    \includegraphics[width=0.135\textwidth]{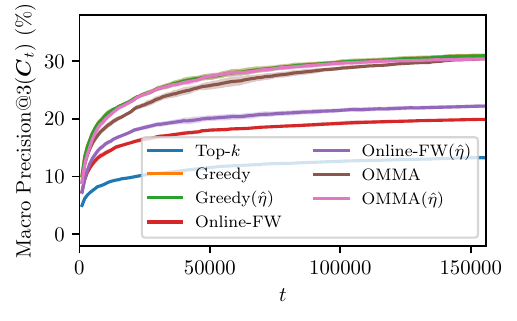}
    \includegraphics[width=0.135\textwidth]{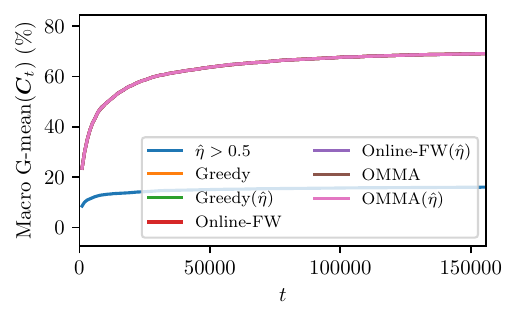}
    \includegraphics[width=0.135\textwidth]{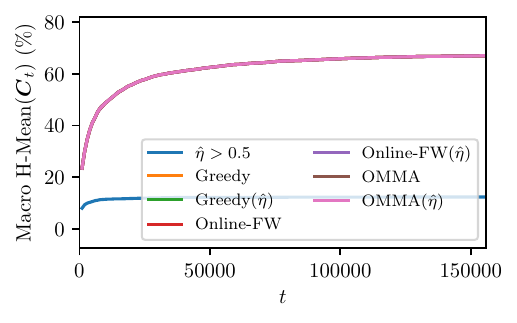}

    \vspace{5pt}\datasettable{AmazonCat}\vspace{5pt}
    
    \includegraphics[width=0.135\textwidth]{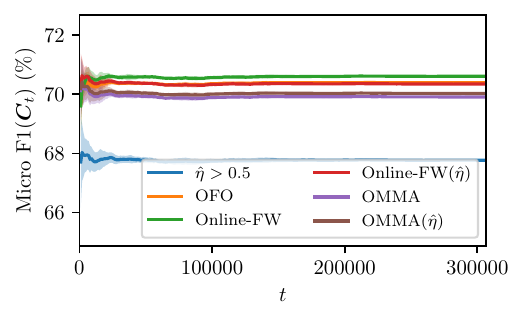}
    \includegraphics[width=0.135\textwidth]{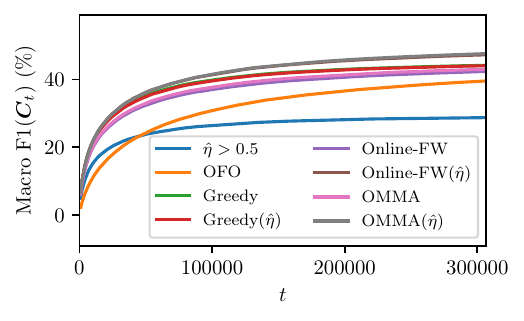}
    \includegraphics[width=0.135\textwidth]{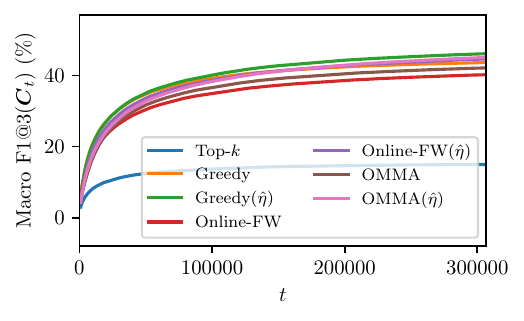}
    \includegraphics[width=0.135\textwidth]{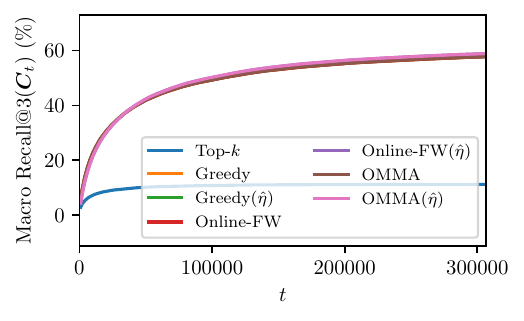}
    \includegraphics[width=0.135\textwidth]{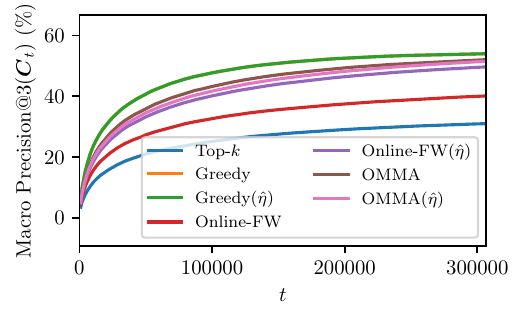}
    \includegraphics[width=0.135\textwidth]{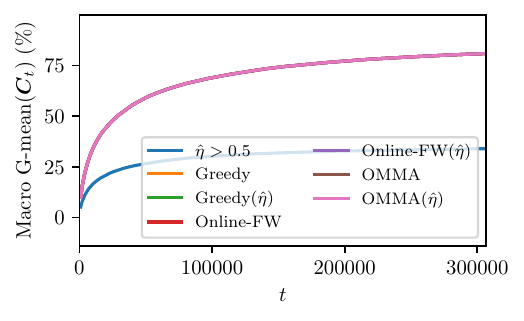}
    \includegraphics[width=0.135\textwidth]{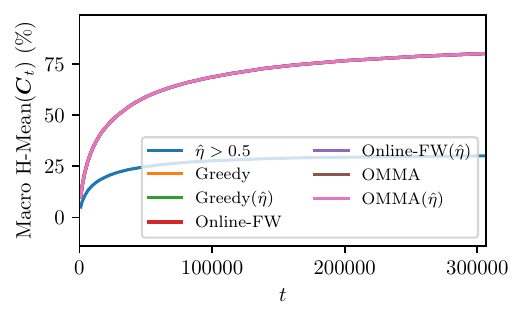}

    \vspace{5pt}\datasettable{News20}\vspace{5pt}
    
    \includegraphics[width=0.135\textwidth]{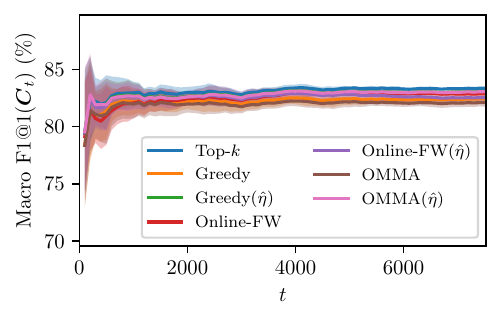}
    \includegraphics[width=0.135\textwidth]{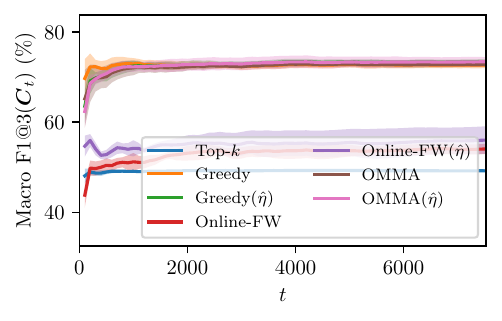}
    \includegraphics[width=0.135\textwidth]{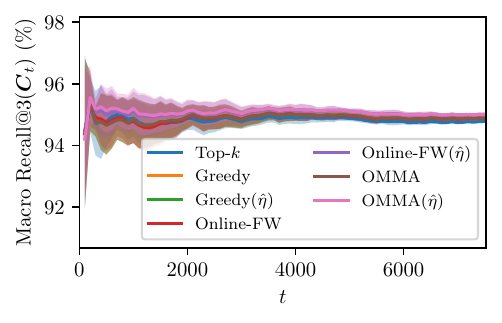}
    \includegraphics[width=0.135\textwidth]{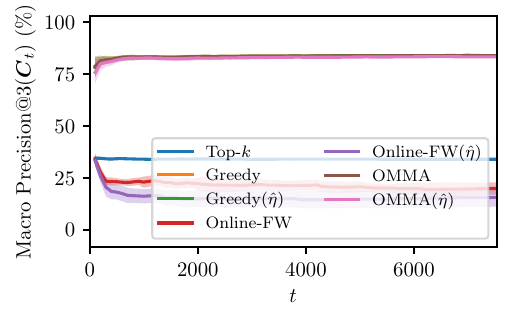}
    \includegraphics[width=0.135\textwidth]{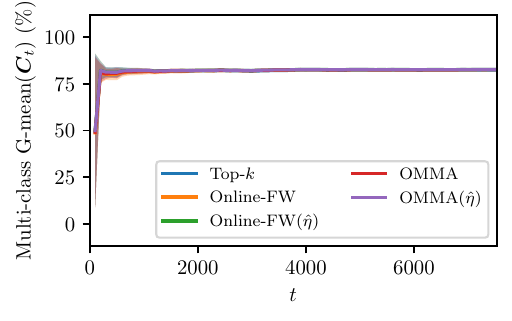}
    \includegraphics[width=0.135\textwidth]{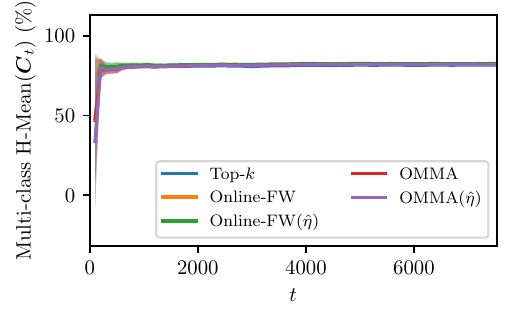}
    \includegraphics[width=0.135\textwidth]{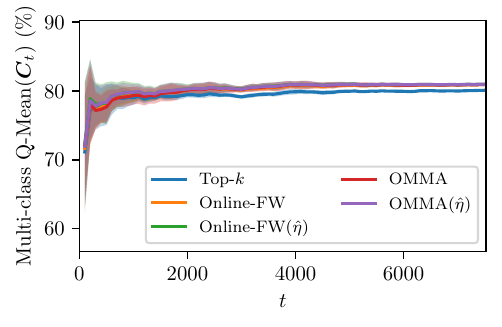}

    \vspace{5pt}\datasettable{Ledgar-LexGlue}\vspace{5pt}
    
    \includegraphics[width=0.135\textwidth]{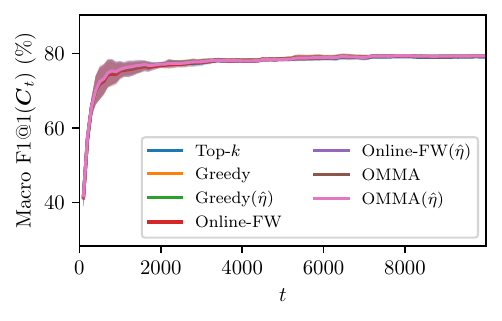}
    \includegraphics[width=0.135\textwidth]{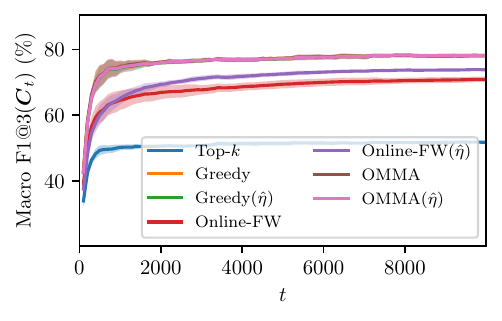}
    \includegraphics[width=0.135\textwidth]{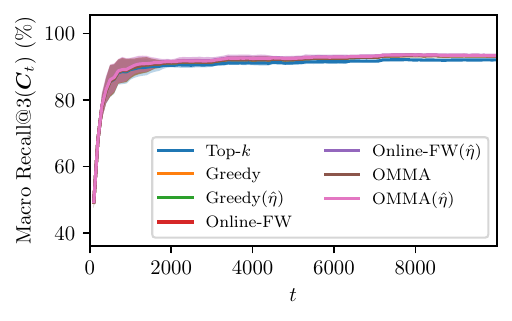}
    \includegraphics[width=0.135\textwidth]{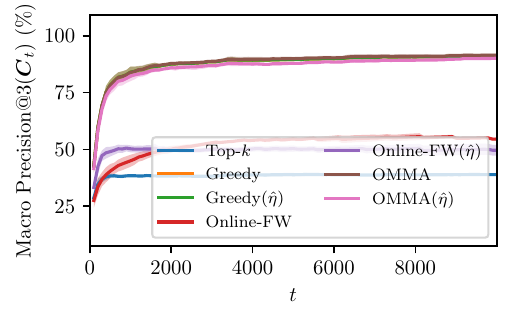}
    \includegraphics[width=0.135\textwidth]{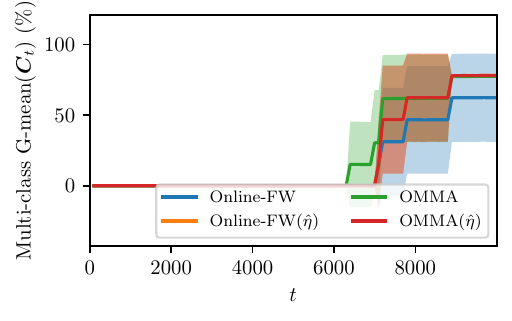}
    \includegraphics[width=0.135\textwidth]{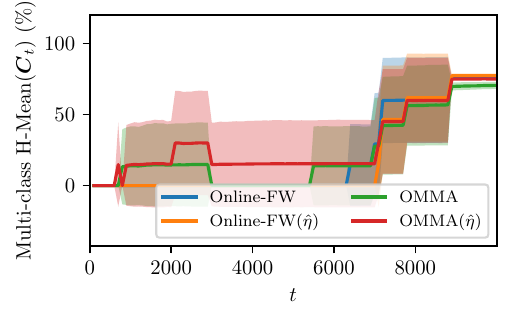}
    \includegraphics[width=0.135\textwidth]{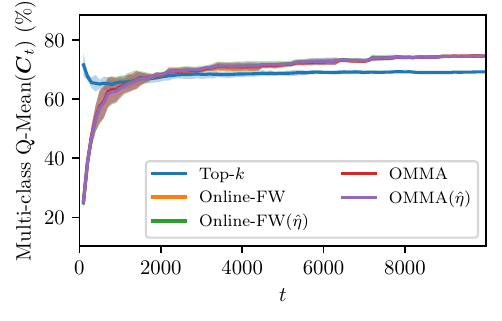}

    \vspace{5pt}\datasettable{Caltech-256}\vspace{5pt}
    
    \includegraphics[width=0.135\textwidth]{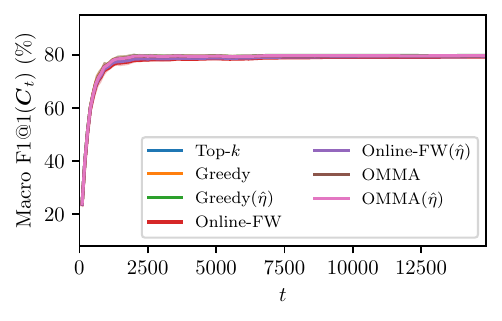}
    \includegraphics[width=0.135\textwidth]{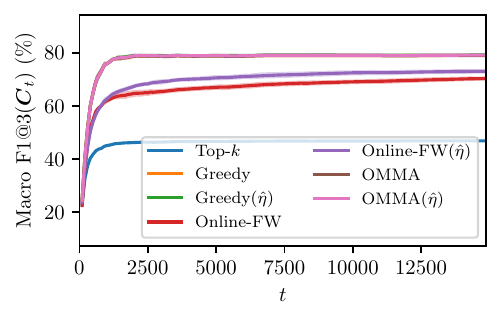}
    \includegraphics[width=0.135\textwidth]{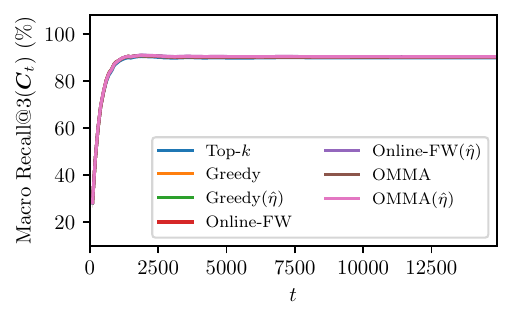}
    \includegraphics[width=0.135\textwidth]{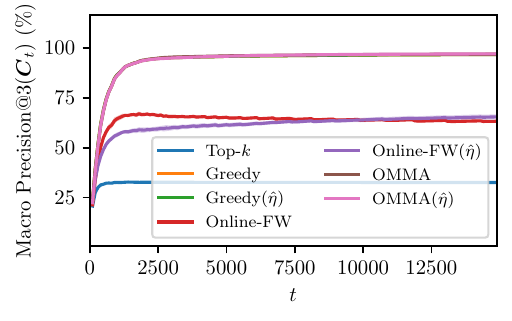}
    \includegraphics[width=0.135\textwidth]{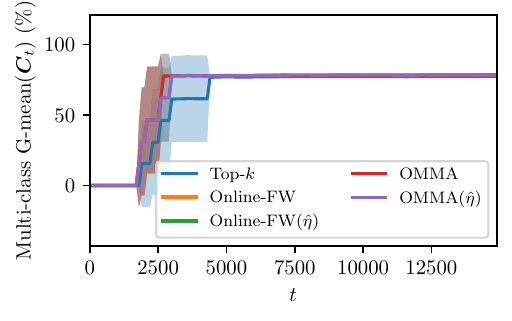}
    \includegraphics[width=0.135\textwidth]{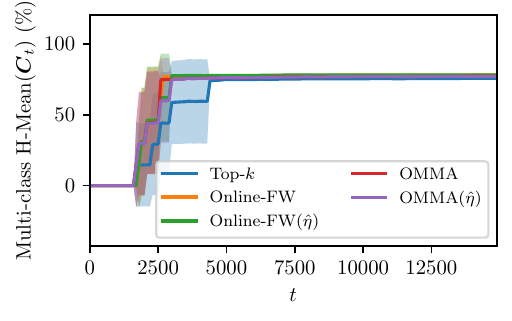}
    \includegraphics[width=0.135\textwidth]{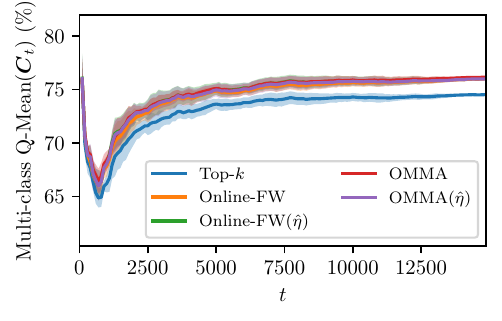}

    \caption{Running comparison of performance for the online algorithms. Averaged over 5 runs, the opaque fill indicate the standard deviation at given iteration $t$.}
    \label{fig:all-plots}
\end{figure*}

\begin{figure*}[ht]
    \centering
    \scriptsize
    
    \vspace{5pt}\datasettable{YouTube}\vspace{5pt}
    
    \includegraphics[width=0.135\textwidth]{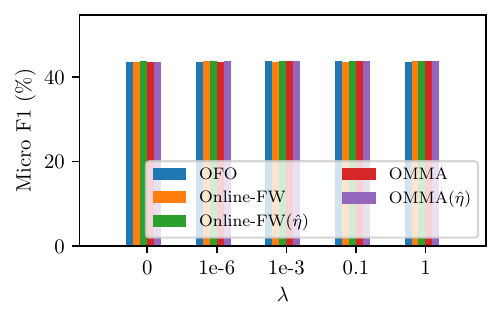}
    \includegraphics[width=0.135\textwidth]{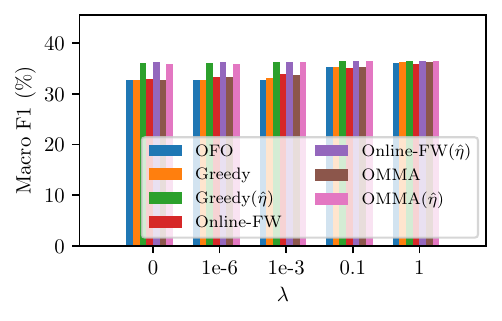}
    \includegraphics[width=0.135\textwidth]{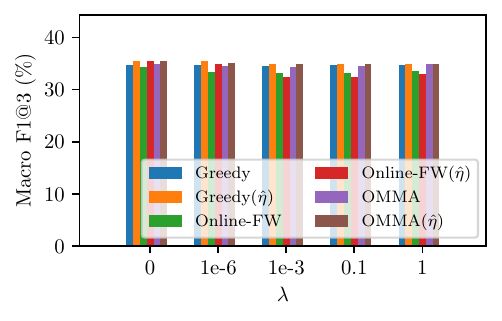}
    \includegraphics[width=0.135\textwidth]{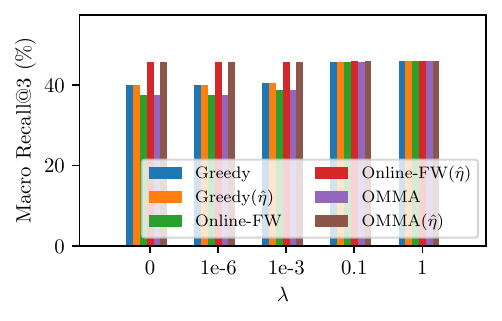}
    \includegraphics[width=0.135\textwidth]{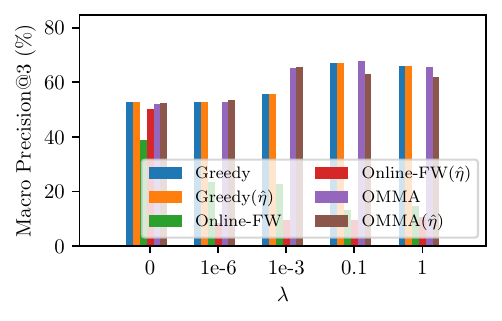}
    \includegraphics[width=0.135\textwidth]{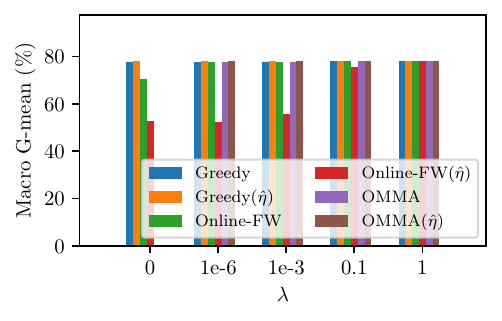}
    \includegraphics[width=0.135\textwidth]{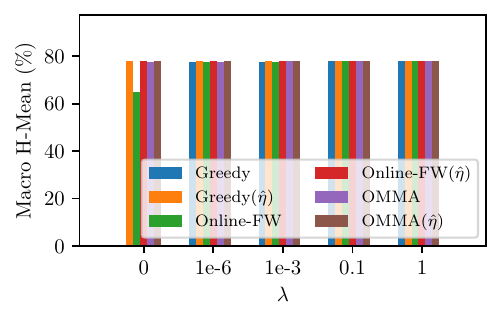}

    \vspace{5pt}\datasettable{Eurlex-LexGlue}\vspace{5pt}
    
    \includegraphics[width=0.135\textwidth]{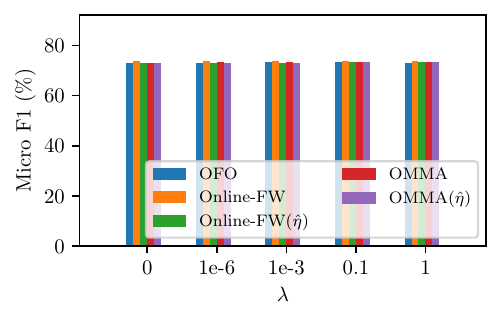}
    \includegraphics[width=0.135\textwidth]{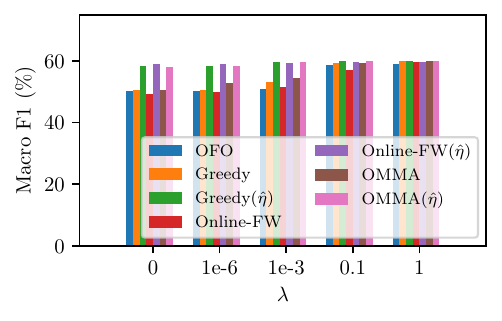}
    \includegraphics[width=0.135\textwidth]{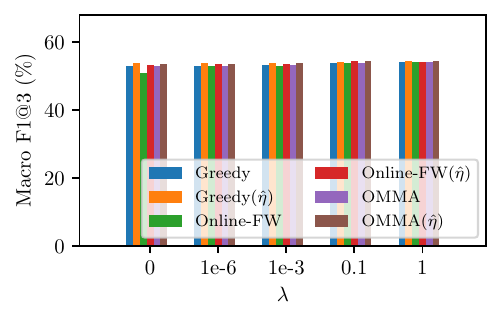}
    \includegraphics[width=0.135\textwidth]{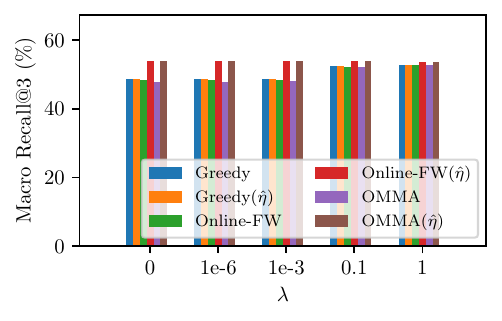}
    \includegraphics[width=0.135\textwidth]{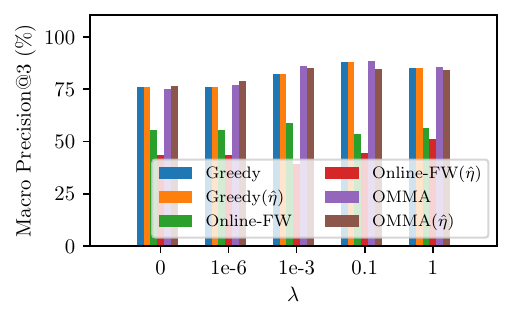}
    \includegraphics[width=0.135\textwidth]{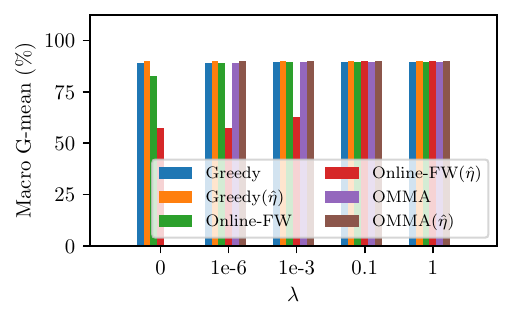}
    \includegraphics[width=0.135\textwidth]{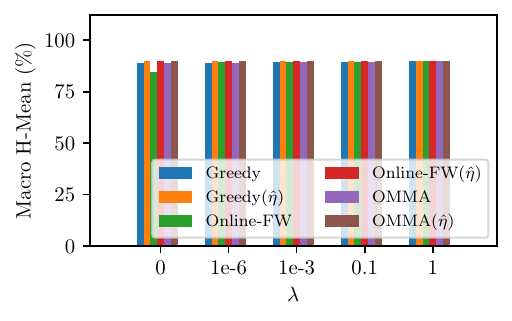}

    \vspace{5pt}\datasettable{Mediamill}\vspace{5pt}

    \includegraphics[width=0.135\textwidth]{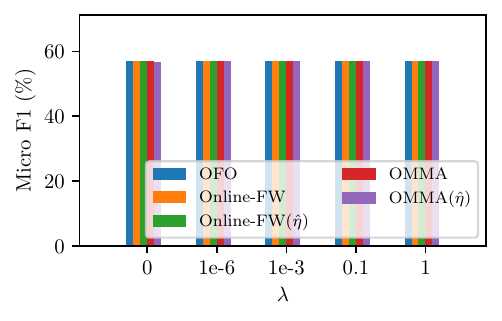}
    \includegraphics[width=0.135\textwidth]{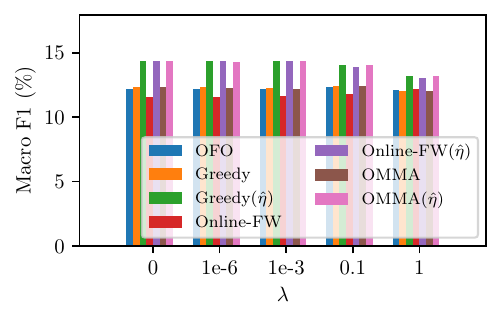}
    \includegraphics[width=0.135\textwidth]{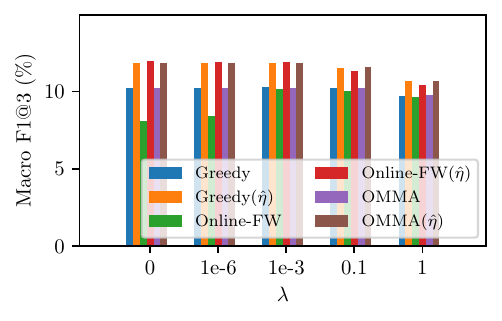}
    \includegraphics[width=0.135\textwidth]{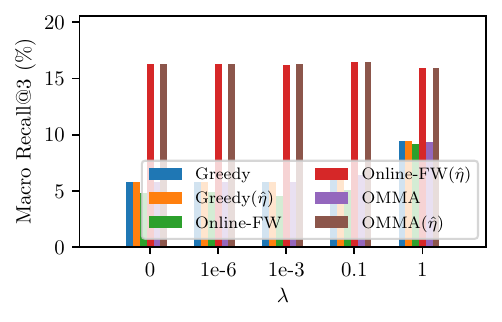}
    \includegraphics[width=0.135\textwidth]{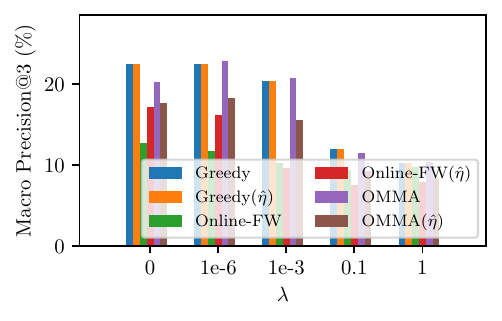}
    \includegraphics[width=0.135\textwidth]{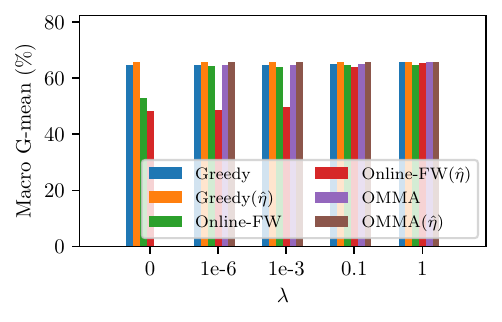}
    \includegraphics[width=0.135\textwidth]{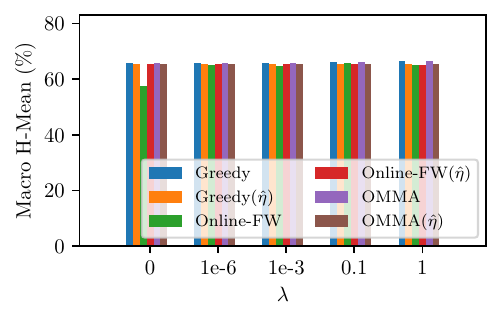}

    \vspace{5pt}\datasettable{Flickr}\vspace{5pt}
    
    \includegraphics[width=0.135\textwidth]{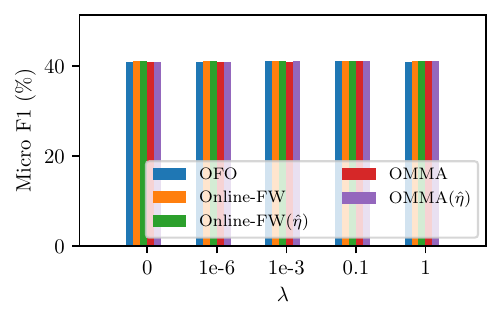}
    \includegraphics[width=0.135\textwidth]{figs/plots-reg/flicker_deepwalk_plt_macro_f1_k=0_pred_utility_history.pdf}
    \includegraphics[width=0.135\textwidth]{figs/plots-reg/flicker_deepwalk_plt_macro_f1_k=3_pred_utility_history.pdf}
    \includegraphics[width=0.135\textwidth]{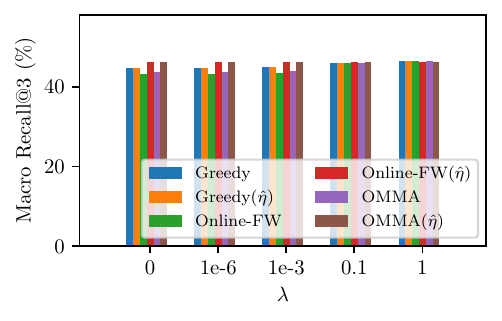}
    \includegraphics[width=0.135\textwidth]{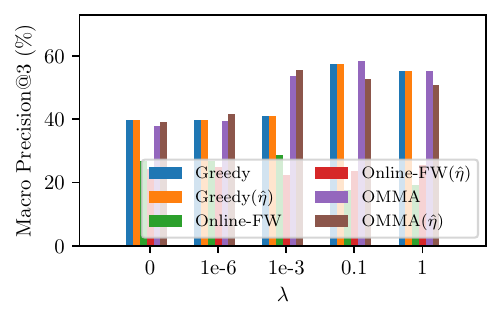}
    \includegraphics[width=0.135\textwidth]{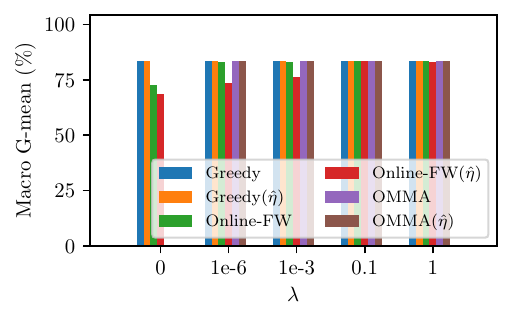}
    \includegraphics[width=0.135\textwidth]{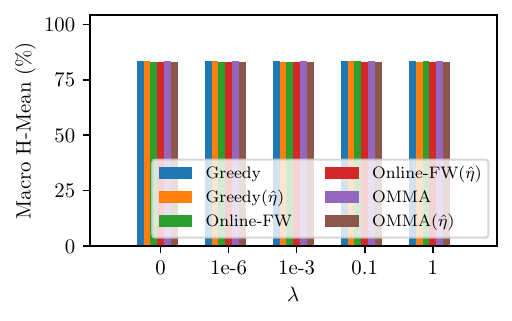}

    \vspace{5pt}\datasettable{News20}\vspace{5pt}
    
    \includegraphics[width=0.135\textwidth]{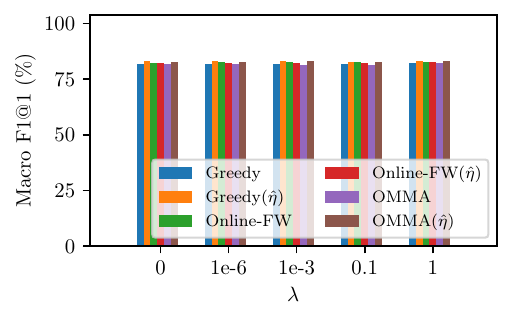}
    \includegraphics[width=0.135\textwidth]{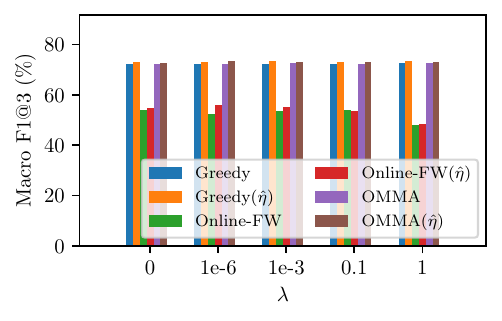}
    \includegraphics[width=0.135\textwidth]{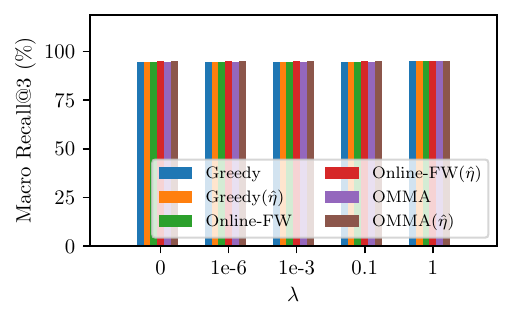}
    \includegraphics[width=0.135\textwidth]{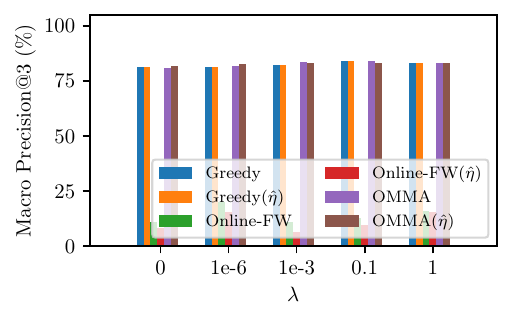}
    \includegraphics[width=0.135\textwidth]{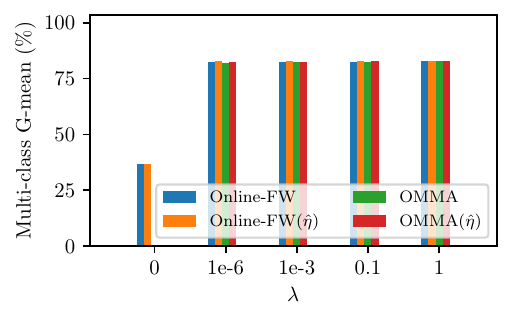}
    \includegraphics[width=0.135\textwidth]{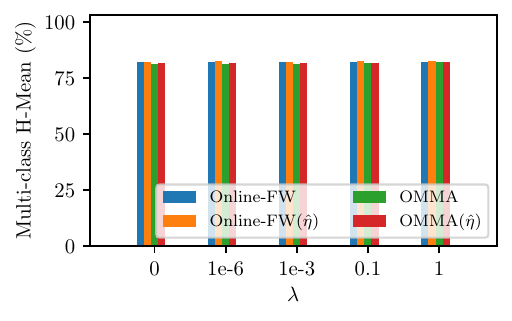}
    \includegraphics[width=0.135\textwidth]{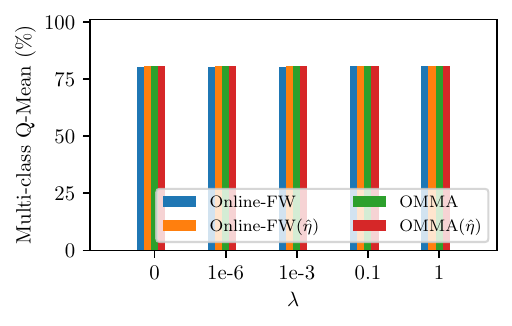}

    \vspace{5pt}\datasettable{Ledgar-LexGlue}\vspace{5pt}
    
    \includegraphics[width=0.135\textwidth]{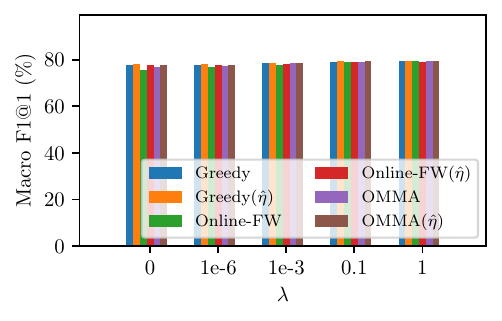}
    \includegraphics[width=0.135\textwidth]{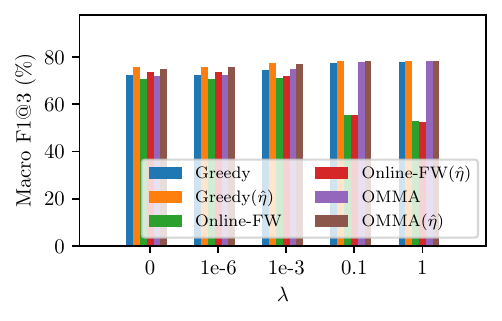}
    \includegraphics[width=0.135\textwidth]{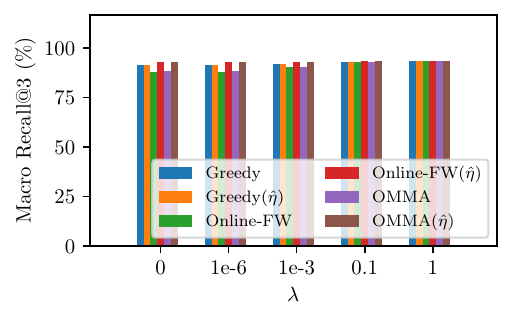}
    \includegraphics[width=0.135\textwidth]{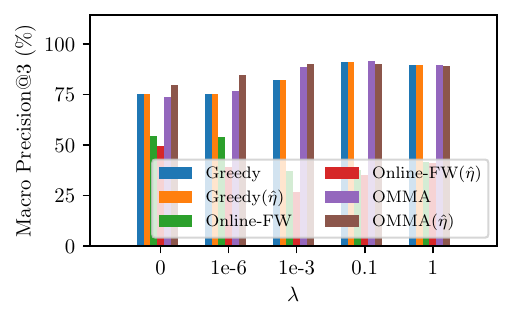}
    \includegraphics[width=0.135\textwidth]{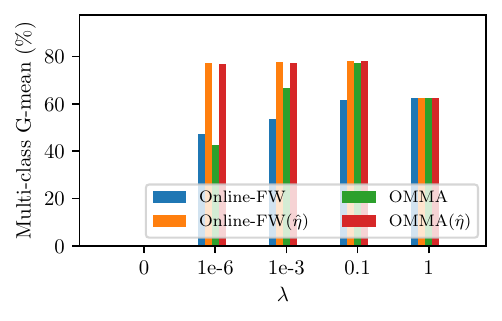}
    \includegraphics[width=0.135\textwidth]{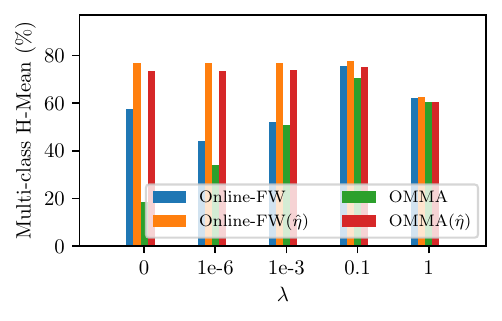}
    \includegraphics[width=0.135\textwidth]{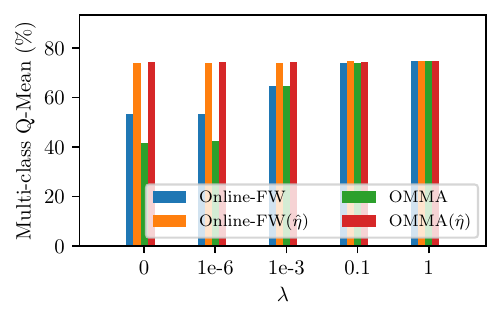}

    \vspace{5pt}\datasettable{Caltech-256}\vspace{5pt}
    
    \includegraphics[width=0.135\textwidth]{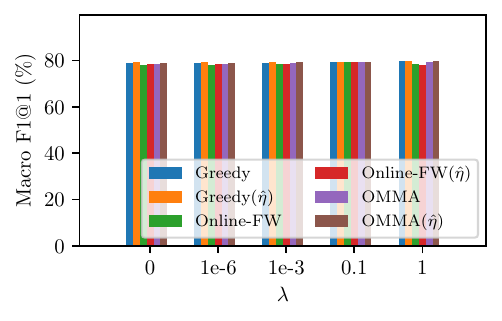}
    \includegraphics[width=0.135\textwidth]{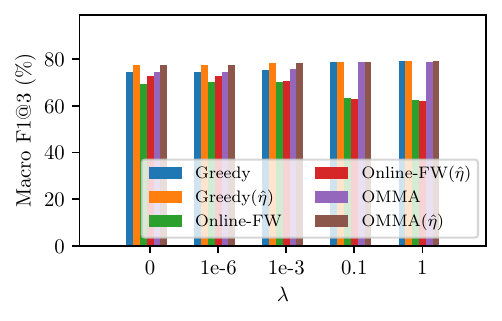}
    \includegraphics[width=0.135\textwidth]{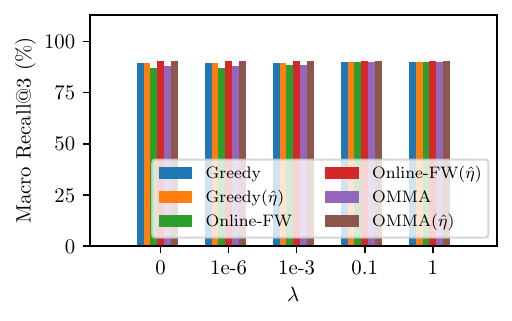}
    \includegraphics[width=0.135\textwidth]{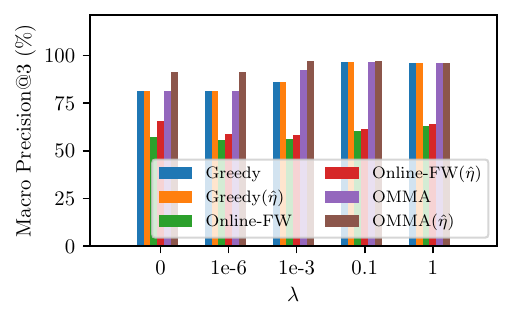}
    \includegraphics[width=0.135\textwidth]{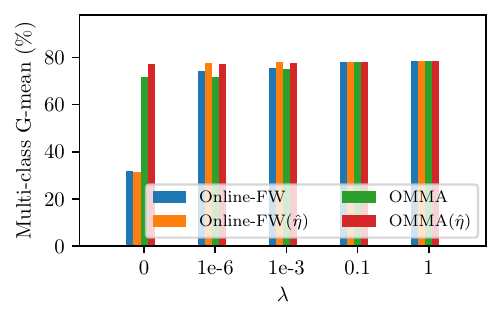}
    \includegraphics[width=0.135\textwidth]{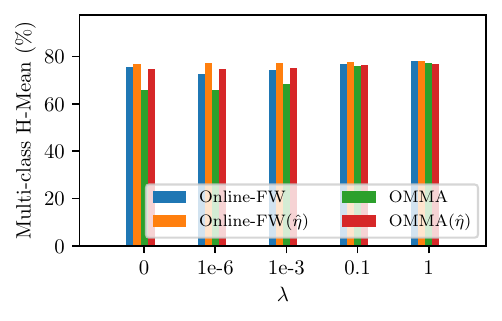}
    \includegraphics[width=0.135\textwidth]{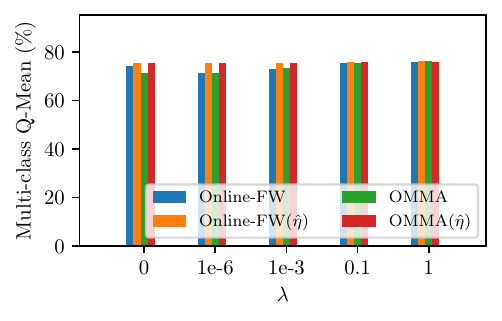}

    \caption{Comparison of performance of the online algorithms with different values of $\lambda$. Averaged over 5 runs.}
    \label{fig:all-reg-plots}
\end{figure*}

\section{Results with the online CPE}
\label{app:online-cpe-results}

In the main paper and the previous section, we presented the results using fixed CPE that was trained on a separate, much larger training set. 
This approach of being agnostic to the CPE model is commonly used in the literature on learning complex performance metrics~\citep{Koyejo_et_al_2014, Busa-Fekete_et_al_2015, Dembczynski_et_al_2017, Liu_et_al_2018, Narasimhan_et_al_2022, Schultheis_etal_ICLR2024}. 
However, it is possible to train CPE along with using online optimization methods. 

In this section, we additionally present the results with such online CPE conducted on four datasets (YouTube, Eurlex-LexGlue, Mediamill, and Flickr). 
In this experiment, the CPE starts with randomly initialized parameters and updates them after each observed instance in the test sequence. 
In~\cref{tab:extended-results-online-cpe}, we report the final results on the sequence, and in~\cref{fig:all-online-cpe-plots}, we present running performance for the same set of datasets and metrics. 
As expected, the average performance over the sequence is lower than in the case of fixed CPE, but the proposed OMMA performance is, in most cases, on par or better than the other algorithms.

\begin{table*}[ht]
    \caption{Predictive performance and running times of the different online algorithms on \emph{multi-label} problems with \emph{incrementally updated CPE}, averaged over 5 runs and reported with standard deviations ($\pm$)s. The best result on each metric is in \textbf{bold}, the second best is in \textit{italic}. We additionally report basic statistics of the benchmarks: number of classes/labels and instances in the test sequence. $\times$ -- means that the algorithm does not support the optimization of that metric. The inference time of Offline-FW was not measured.}
    \label{tab:extended-results-online-cpe}
    \vspace{8pt}
    \small
    \centering

\resizebox{\linewidth}{!}{
\setlength\tabcolsep{3 pt}
\begin{tabular}{l|r@{}lr@{}l|r@{}lr@{}l|r@{}lr@{}l|r@{}lr@{}l|r@{}lr@{}l|r@{}lr@{}l|r@{}lr@{}l}
\toprule
    Method & \multicolumn{4}{c|}{Micro F1} & \multicolumn{4}{c|}{Macro F1} & \multicolumn{4}{c|}{Macro F1$@3$} & \multicolumn{4}{c|}{Macro Recall$@3$} & \multicolumn{4}{c|}{Macro Precision$@3$} & \multicolumn{4}{c|}{Macro G-Mean}& \multicolumn{4}{c}{Macro H-Mean} \\
    & \multicolumn{2}{c}{(\%)} & \multicolumn{2}{c|}{time (s)} & \multicolumn{2}{c}{(\%)} & \multicolumn{2}{c|}{time (s)} & \multicolumn{2}{c}{(\%)} & \multicolumn{2}{c|}{time (s)} & \multicolumn{2}{c}{(\%)} & \multicolumn{2}{c|}{time (s)} & \multicolumn{2}{c}{(\%)} & \multicolumn{2}{c|}{time (s)} & \multicolumn{2}{c}{(\%)} & \multicolumn{2}{c|}{time (s)} & \multicolumn{2}{c}{(\%)} & \multicolumn{2}{c}{time (s)} \\
\midrule
& \multicolumn{28}{c}{\datasettable{YouTube} ($m = 46, n = 31703$)}  \\ 
 \midrule
    Top-$k$ / $\hat \eta \!>\!0.5$ & 26.07 & \scriptsize $\pm$ 0.11 & 2.84 & \scriptsize $\pm$ 0.07 & 18.26 & \scriptsize $\pm$ 0.06 & 2.73 & \scriptsize $\pm$ 0.04 & 27.73 & \scriptsize $\pm$ 0.03 & 0.29 & \scriptsize $\pm$ 0.00 & 38.11 & \scriptsize $\pm$ 0.04 & 0.28 & \scriptsize $\pm$ 0.00 & 23.08 & \scriptsize $\pm$ 0.12 & 0.28 & \scriptsize $\pm$ 0.00 & 27.83 & \scriptsize $\pm$ 0.18 & 2.80 & \scriptsize $\pm$ 0.02 & 19.85 & \scriptsize $\pm$ 0.07 & 2.81 & \scriptsize $\pm$ 0.09 \\
    \midrule
    OFO & 32.52 & \scriptsize $\pm$ 0.31 & 4.38 & \scriptsize $\pm$ 0.03 & 28.26 & \scriptsize $\pm$ 0.40 & 4.15 & \scriptsize $\pm$ 0.03 & & $\times$ & & $\times$ & & $\times$ & & $\times$ & & $\times$ & & $\times$ & & $\times$ & & $\times$ & & $\times$ & & $\times$ \\
    Greedy & & $\times$ & & $\times$ & 25.21 & \scriptsize $\pm$ 0.17 & 7.75 & \scriptsize $\pm$ 0.12 & \textit{31.53}&\textit{\scriptsize $\pm$ 0.10} & 5.27 & \scriptsize $\pm$ 0.06 & 36.76 & \scriptsize $\pm$ 0.11 & 5.08 & \scriptsize $\pm$ 0.06 & \textit{46.91}&\textit{\scriptsize $\pm$ 3.36} & 4.79 & \scriptsize $\pm$ 0.02 & \textbf{76.39}&\textbf{\scriptsize $\pm$ 0.01} & 7.79 & \scriptsize $\pm$ 0.06 & \textbf{76.41}&\textbf{\scriptsize $\pm$ 0.03} & 7.99 & \scriptsize $\pm$ 0.03 \\
    Greedy$(\hat \eta)$ & & $\times$ & & $\times$ & \textit{31.88}&\textit{\scriptsize $\pm$ 0.04} & 7.88 & \scriptsize $\pm$ 0.18 & 31.17 & \scriptsize $\pm$ 0.06 & 5.44 & \scriptsize $\pm$ 0.03 & 36.76 & \scriptsize $\pm$ 0.11 & 5.06 & \scriptsize $\pm$ 0.05 & \textit{46.91}&\textit{\scriptsize $\pm$ 3.36} & 4.79 & \scriptsize $\pm$ 0.04 & \textbf{76.39}&\textbf{\scriptsize $\pm$ 0.02} & 7.82 & \scriptsize $\pm$ 0.08 & \textit{76.31}&\textit{\scriptsize $\pm$ 0.02} & 8.00 & \scriptsize $\pm$ 0.06 \\
    Online-FW & 33.95 & \scriptsize $\pm$ 0.08 & 22.50 & \scriptsize $\pm$ 1.65 & 26.33 & \scriptsize $\pm$ 0.22 & 21.96 & \scriptsize $\pm$ 0.71 & 31.13 & \scriptsize $\pm$ 0.10 & 20.70 & \scriptsize $\pm$ 2.49 & 36.74 & \scriptsize $\pm$ 0.10 & 10.65 & \scriptsize $\pm$ 0.19 & 27.97 & \scriptsize $\pm$ 0.89 & 6.08 & \scriptsize $\pm$ 0.19 & \textit{76.03}&\textit{\scriptsize $\pm$ 0.08} & 14.64 & \scriptsize $\pm$ 0.60 & 75.92 & \scriptsize $\pm$ 0.04 & 21.97 & \scriptsize $\pm$ 2.08 \\
    Online-FW$(\hat \eta)$ & \textbf{39.46}&\textbf{\scriptsize $\pm$ 0.02} & 19.60 & \scriptsize $\pm$ 0.24 & 31.78 & \scriptsize $\pm$ 0.02 & 20.52 & \scriptsize $\pm$ 0.14 & 30.09 & \scriptsize $\pm$ 0.09 & 67.48 & \scriptsize $\pm$ 0.57 & \textit{42.03}&\textit{\scriptsize $\pm$ 0.06} & 12.75 & \scriptsize $\pm$ 0.14 & 9.36 & \scriptsize $\pm$ 0.71 & 50.73 & \scriptsize $\pm$ 2.62 & 75.53 & \scriptsize $\pm$ 0.07 & 11.52 & \scriptsize $\pm$ 2.78 & 75.15 & \scriptsize $\pm$ 0.10 & 55.67 & \scriptsize $\pm$ 10.27 \\
    \midrule
    \OMMA{} & 32.02 & \scriptsize $\pm$ 0.14 & 20.90 & \scriptsize $\pm$ 0.36 & 25.21 & \scriptsize $\pm$ 0.17 & 18.06 & \scriptsize $\pm$ 0.20 & \textbf{31.54}&\textbf{\scriptsize $\pm$ 0.12} & 14.09 & \scriptsize $\pm$ 0.13 & 36.75 & \scriptsize $\pm$ 0.11 & 11.86 & \scriptsize $\pm$ 0.21 & \textbf{47.72}&\textbf{\scriptsize $\pm$ 3.13} & 11.73 & \scriptsize $\pm$ 0.12 & \textbf{76.39}&\textbf{\scriptsize $\pm$ 0.01} & 19.08 & \scriptsize $\pm$ 0.14 & \textbf{76.41}&\textbf{\scriptsize $\pm$ 0.02} & 20.31 & \scriptsize $\pm$ 0.35 \\
    \OMMAeta{} & \textit{39.34}&\textit{\scriptsize $\pm$ 0.02} & 21.06 & \scriptsize $\pm$ 0.34 & \textbf{31.89}&\textbf{\scriptsize $\pm$ 0.05} & 17.90 & \scriptsize $\pm$ 0.34 & 31.13 & \scriptsize $\pm$ 0.12 & 14.06 & \scriptsize $\pm$ 0.11 & \textbf{42.05}&\textbf{\scriptsize $\pm$ 0.05} & 11.75 & \scriptsize $\pm$ 0.25 & 43.87 & \scriptsize $\pm$ 1.10 & 11.64 & \scriptsize $\pm$ 0.39 & \textbf{76.39}&\textbf{\scriptsize $\pm$ 0.02} & 18.95 & \scriptsize $\pm$ 0.14 & \textit{76.31}&\textit{\scriptsize $\pm$ 0.03} & 20.17 & \scriptsize $\pm$ 0.06 \\
\midrule
& \multicolumn{28}{c}{\datasettable{Eurlex-LexGlue} ($m = 100, n = 65000$)}  \\ 
 \midrule
    Top-$k$ / $\hat \eta \!>\!0.5$ & 26.50 & \scriptsize $\pm$ 0.03 & 10.76 & \scriptsize $\pm$ 0.06 & 6.54 & \scriptsize $\pm$ 0.02 & 10.83 & \scriptsize $\pm$ 0.00 & 12.69 & \scriptsize $\pm$ 0.03 & 0.60 & \scriptsize $\pm$ 0.01 & 10.90 & \scriptsize $\pm$ 0.02 & 0.58 & \scriptsize $\pm$ 0.01 & 31.48 & \scriptsize $\pm$ 1.17 & 0.60 & \scriptsize $\pm$ 0.01 & 9.31 & \scriptsize $\pm$ 0.02 & 10.73 & \scriptsize $\pm$ 0.02 & 6.62 & \scriptsize $\pm$ 0.02 & 10.75 & \scriptsize $\pm$ 0.09 \\
    \midrule
    OFO & 47.54 & \scriptsize $\pm$ 0.01 & 16.24 & \scriptsize $\pm$ 0.08 & 22.67 & \scriptsize $\pm$ 0.14 & 15.83 & \scriptsize $\pm$ 0.17 & & $\times$ & & $\times$ & & $\times$ & & $\times$ & & $\times$ & & $\times$ & & $\times$ & & $\times$ & & $\times$ & & $\times$ \\
    Greedy & & $\times$ & & $\times$ & 21.53 & \scriptsize $\pm$ 0.06 & 21.14 & \scriptsize $\pm$ 0.12 & 22.61 & \scriptsize $\pm$ 0.07 & 10.71 & \scriptsize $\pm$ 0.02 & 25.02 & \scriptsize $\pm$ 0.11 & 10.18 & \scriptsize $\pm$ 0.35 & 35.17 & \scriptsize $\pm$ 0.24 & 10.09 & \scriptsize $\pm$ 0.05 & \textit{71.71}&\textit{\scriptsize $\pm$ 0.01} & 21.27 & \scriptsize $\pm$ 0.11 & \textit{72.49}&\textit{\scriptsize $\pm$ 0.05} & 21.60 & \scriptsize $\pm$ 0.31 \\
    Greedy$(\hat \eta)$ & & $\times$ & & $\times$ & \textit{28.61}&\textit{\scriptsize $\pm$ 0.03} & 20.73 & \scriptsize $\pm$ 0.15 & \textit{29.86}&\textit{\scriptsize $\pm$ 0.04} & 10.80 & \scriptsize $\pm$ 0.07 & 25.02 & \scriptsize $\pm$ 0.11 & 10.07 & \scriptsize $\pm$ 0.15 & 35.17 & \scriptsize $\pm$ 0.24 & 10.20 & \scriptsize $\pm$ 0.09 & 70.56 & \scriptsize $\pm$ 0.02 & 21.27 & \scriptsize $\pm$ 0.05 & 69.62 & \scriptsize $\pm$ 0.04 & 21.44 & \scriptsize $\pm$ 0.08 \\
    Online-FW & 49.10 & \scriptsize $\pm$ 0.12 & 35.78 & \scriptsize $\pm$ 0.37 & 21.10 & \scriptsize $\pm$ 0.19 & 36.84 & \scriptsize $\pm$ 0.90 & 25.73 & \scriptsize $\pm$ 0.21 & 29.24 & \scriptsize $\pm$ 3.35 & 24.56 & \scriptsize $\pm$ 0.09 & 15.94 & \scriptsize $\pm$ 0.32 & 21.78 & \scriptsize $\pm$ 1.41 & 12.41 & \scriptsize $\pm$ 0.29 & 70.89 & \scriptsize $\pm$ 0.29 & 40.18 & \scriptsize $\pm$ 0.64 & 71.00 & \scriptsize $\pm$ 0.13 & 54.91 & \scriptsize $\pm$ 2.13 \\
    Online-FW$(\hat \eta)$ & \textbf{51.78}&\textbf{\scriptsize $\pm$ 0.03} & 38.66 & \scriptsize $\pm$ 0.41 & 26.70 & \scriptsize $\pm$ 0.13 & 40.81 & \scriptsize $\pm$ 0.10 & 28.70 & \scriptsize $\pm$ 0.04 & 113.14 & \scriptsize $\pm$ 0.51 & \textit{26.17}&\textit{\scriptsize $\pm$ 0.04} & 21.37 & \scriptsize $\pm$ 0.17 & 25.42 & \scriptsize $\pm$ 1.42 & 95.89 & \scriptsize $\pm$ 0.51 & 70.67 & \scriptsize $\pm$ 0.08 & 23.76 & \scriptsize $\pm$ 0.76 & 70.06 & \scriptsize $\pm$ 0.10 & 121.24 & \scriptsize $\pm$ 0.35 \\
    \midrule
    \OMMA{} & 47.09 & \scriptsize $\pm$ 0.05 & 49.15 & \scriptsize $\pm$ 0.37 & 21.53 & \scriptsize $\pm$ 0.06 & 43.19 & \scriptsize $\pm$ 0.11 & 22.58 & \scriptsize $\pm$ 0.03 & 28.61 & \scriptsize $\pm$ 0.05 & 25.01 & \scriptsize $\pm$ 0.12 & 22.67 & \scriptsize $\pm$ 0.16 & \textit{63.71}&\textit{\scriptsize $\pm$ 2.42} & 23.14 & \scriptsize $\pm$ 0.17 & \textbf{71.73}&\textbf{\scriptsize $\pm$ 0.01} & 46.14 & \scriptsize $\pm$ 0.56 & \textbf{72.51}&\textbf{\scriptsize $\pm$ 0.05} & 48.10 & \scriptsize $\pm$ 0.68 \\
    \OMMAeta{} & \textit{51.60}&\textit{\scriptsize $\pm$ 0.03} & 48.38 & \scriptsize $\pm$ 0.29 & \textbf{28.63}&\textbf{\scriptsize $\pm$ 0.02} & 42.67 & \scriptsize $\pm$ 0.74 & \textbf{29.87}&\textbf{\scriptsize $\pm$ 0.04} & 29.60 & \scriptsize $\pm$ 0.54 & \textbf{26.25}&\textbf{\scriptsize $\pm$ 0.04} & 22.98 & \scriptsize $\pm$ 0.04 & \textbf{64.96}&\textbf{\scriptsize $\pm$ 1.49} & 23.43 & \scriptsize $\pm$ 0.25 & 70.58 & \scriptsize $\pm$ 0.02 & 46.16 & \scriptsize $\pm$ 0.12 & 69.64 & \scriptsize $\pm$ 0.04 & 48.40 & \scriptsize $\pm$ 0.79 \\
\midrule
& \multicolumn{28}{c}{\datasettable{Mediamill} ($m = 101, n = 43907$)}  \\ 
 \midrule
    Top-$k$ / $\hat \eta \!>\!0.5$ & 48.38 & \scriptsize $\pm$ 0.04 & 5.21 & \scriptsize $\pm$ 0.06 & 3.09 & \scriptsize $\pm$ 0.01 & 5.21 & \scriptsize $\pm$ 0.05 & 3.65 & \scriptsize $\pm$ 0.02 & 0.39 & \scriptsize $\pm$ 0.01 & 3.72 & \scriptsize $\pm$ 0.01 & 0.38 & \scriptsize $\pm$ 0.02 & 6.87 & \scriptsize $\pm$ 0.20 & 0.39 & \scriptsize $\pm$ 0.00 & 3.95 & \scriptsize $\pm$ 0.06 & 4.91 & \scriptsize $\pm$ 0.03 & 2.45 & \scriptsize $\pm$ 0.01 & 5.15 & \scriptsize $\pm$ 0.05 \\
    \midrule
    OFO & 52.00 & \scriptsize $\pm$ 0.02 & 8.22 & \scriptsize $\pm$ 0.04 & \textbf{8.14}&\textbf{\scriptsize $\pm$ 0.03} & 8.63 & \scriptsize $\pm$ 0.16 & & $\times$ & & $\times$ & & $\times$ & & $\times$ & & $\times$ & & $\times$ & & $\times$ & & $\times$ & & $\times$ & & $\times$ \\
    Greedy & & $\times$ & & $\times$ & \textit{8.12}&\textit{\scriptsize $\pm$ 0.02} & 12.39 & \scriptsize $\pm$ 0.02 & 4.71 & \scriptsize $\pm$ 0.04 & 7.00 & \scriptsize $\pm$ 0.09 & 4.22 & \scriptsize $\pm$ 0.07 & 6.68 & \scriptsize $\pm$ 0.10 & \textbf{11.60}&\textbf{\scriptsize $\pm$ 1.88} & 6.56 & \scriptsize $\pm$ 0.11 & \textit{53.55}&\textit{\scriptsize $\pm$ 0.11} & 12.11 & \scriptsize $\pm$ 0.10 & 53.72 & \scriptsize $\pm$ 0.11 & 11.96 & \scriptsize $\pm$ 0.03 \\
    Greedy$(\hat \eta)$ & & $\times$ & & $\times$ & 7.79 & \scriptsize $\pm$ 0.01 & 12.17 & \scriptsize $\pm$ 0.01 & \textbf{5.67}&\textbf{\scriptsize $\pm$ 0.03} & 7.10 & \scriptsize $\pm$ 0.15 & 4.22 & \scriptsize $\pm$ 0.07 & 6.59 & \scriptsize $\pm$ 0.06 & \textbf{11.60}&\textbf{\scriptsize $\pm$ 1.88} & 6.62 & \scriptsize $\pm$ 0.15 & 47.59 & \scriptsize $\pm$ 0.15 & 11.85 & \scriptsize $\pm$ 0.05 & 43.82 & \scriptsize $\pm$ 0.21 & 11.90 & \scriptsize $\pm$ 0.04 \\
    Online-FW & 51.81 & \scriptsize $\pm$ 0.03 & 27.08 & \scriptsize $\pm$ 0.91 & 8.02 & \scriptsize $\pm$ 0.00 & 25.08 & \scriptsize $\pm$ 1.11 & 4.82 & \scriptsize $\pm$ 0.08 & 14.94 & \scriptsize $\pm$ 2.17 & 4.08 & \scriptsize $\pm$ 0.05 & 10.43 & \scriptsize $\pm$ 0.41 & 6.19 & \scriptsize $\pm$ 0.15 & 7.47 & \scriptsize $\pm$ 0.10 & \textbf{53.63}&\textbf{\scriptsize $\pm$ 0.01} & 19.08 & \scriptsize $\pm$ 0.09 & \textbf{53.77}&\textbf{\scriptsize $\pm$ 0.07} & 21.75 & \scriptsize $\pm$ 0.10 \\
    Online-FW$(\hat \eta)$ & \textbf{52.16}&\textbf{\scriptsize $\pm$ 0.01} & 27.50 & \scriptsize $\pm$ 0.56 & 7.74 & \scriptsize $\pm$ 0.03 & 32.25 & \scriptsize $\pm$ 0.21 & \textbf{5.67}&\textbf{\scriptsize $\pm$ 0.03} & 84.47 & \scriptsize $\pm$ 2.23 & \textit{4.27}&\textit{\scriptsize $\pm$ 0.08} & 16.26 & \scriptsize $\pm$ 0.22 & 6.24 & \scriptsize $\pm$ 0.36 & 75.56 & \scriptsize $\pm$ 0.87 & 49.23 & \scriptsize $\pm$ 0.40 & 12.18 & \scriptsize $\pm$ 0.29 & 41.81 & \scriptsize $\pm$ 0.36 & 72.49 & \scriptsize $\pm$ 0.98 \\
    \midrule
    \OMMA{} & 51.71 & \scriptsize $\pm$ 0.02 & 30.72 & \scriptsize $\pm$ 0.33 & \textit{8.12}&\textit{\scriptsize $\pm$ 0.02} & 26.51 & \scriptsize $\pm$ 0.13 & 4.71 & \scriptsize $\pm$ 0.04 & 18.89 & \scriptsize $\pm$ 0.14 & 4.22 & \scriptsize $\pm$ 0.07 & 15.20 & \scriptsize $\pm$ 0.27 & \textit{10.30}&\textit{\scriptsize $\pm$ 1.06} & 15.56 & \scriptsize $\pm$ 0.13 & 53.54 & \scriptsize $\pm$ 0.12 & 27.21 & \scriptsize $\pm$ 0.25 & \textit{53.73}&\textit{\scriptsize $\pm$ 0.11} & 29.98 & \scriptsize $\pm$ 0.24 \\
    \OMMAeta{} & \textit{52.12}&\textit{\scriptsize $\pm$ 0.01} & 30.63 & \scriptsize $\pm$ 0.48 & 7.80 & \scriptsize $\pm$ 0.01 & 26.04 & \scriptsize $\pm$ 0.17 & \textit{5.66}&\textit{\scriptsize $\pm$ 0.03} & 18.61 & \scriptsize $\pm$ 0.12 & \textbf{4.30}&\textbf{\scriptsize $\pm$ 0.06} & 15.20 & \scriptsize $\pm$ 0.11 & 8.77 & \scriptsize $\pm$ 0.54 & 15.20 & \scriptsize $\pm$ 0.15 & 47.57 & \scriptsize $\pm$ 0.14 & 28.14 & \scriptsize $\pm$ 0.38 & 43.81 & \scriptsize $\pm$ 0.22 & 28.18 & \scriptsize $\pm$ 0.05 \\
\midrule
& \multicolumn{28}{c}{\datasettable{Flickr} ($m = 195, n = 80513$)}  \\ 
 \midrule
    Top-$k$ / $\hat \eta \!>\!0.5$ & 14.42 & \scriptsize $\pm$ 0.03 & 16.48 & \scriptsize $\pm$ 0.12 & 4.25 & \scriptsize $\pm$ 0.05 & 16.50 & \scriptsize $\pm$ 0.27 & 14.73 & \scriptsize $\pm$ 0.14 & 0.85 & \scriptsize $\pm$ 0.02 & 19.71 & \scriptsize $\pm$ 0.07 & 0.86 & \scriptsize $\pm$ 0.01 & 17.10 & \scriptsize $\pm$ 0.69 & 0.88 & \scriptsize $\pm$ 0.01 & 7.85 & \scriptsize $\pm$ 0.15 & 16.48 & \scriptsize $\pm$ 0.17 & 4.52 & \scriptsize $\pm$ 0.05 & 16.41 & \scriptsize $\pm$ 0.11 \\
    \midrule
    OFO & \textbf{32.28}&\textbf{\scriptsize $\pm$ 0.05} & 24.31 & \scriptsize $\pm$ 0.78 & 14.22 & \scriptsize $\pm$ 0.12 & 24.56 & \scriptsize $\pm$ 0.68 & & $\times$ & & $\times$ & & $\times$ & & $\times$ & & $\times$ & & $\times$ & & $\times$ & & $\times$ & & $\times$ & & $\times$ \\
    Greedy & & $\times$ & & $\times$ & 5.74 & \scriptsize $\pm$ 0.03 & 30.57 & \scriptsize $\pm$ 0.38 & 15.23 & \scriptsize $\pm$ 0.03 & 14.14 & \scriptsize $\pm$ 0.20 & 21.98 & \scriptsize $\pm$ 0.16 & 13.53 & \scriptsize $\pm$ 0.21 & \textbf{31.81}&\textbf{\scriptsize $\pm$ 1.19} & 13.17 & \scriptsize $\pm$ 0.18 & 62.53 & \scriptsize $\pm$ 0.04 & 30.87 & \scriptsize $\pm$ 0.65 & 65.78 & \scriptsize $\pm$ 0.02 & 31.65 & \scriptsize $\pm$ 0.57 \\
    Greedy$(\hat \eta)$ & & $\times$ & & $\times$ & \textit{15.13}&\textit{\scriptsize $\pm$ 0.16} & 30.93 & \scriptsize $\pm$ 0.33 & 16.05 & \scriptsize $\pm$ 0.03 & 14.09 & \scriptsize $\pm$ 0.09 & 21.98 & \scriptsize $\pm$ 0.16 & 13.35 & \scriptsize $\pm$ 0.12 & \textbf{31.81}&\textbf{\scriptsize $\pm$ 1.19} & 13.14 & \scriptsize $\pm$ 0.19 & \textbf{69.00}&\textbf{\scriptsize $\pm$ 0.05} & 31.57 & \scriptsize $\pm$ 1.05 & \textit{68.05}&\textit{\scriptsize $\pm$ 0.03} & 32.38 & \scriptsize $\pm$ 0.18 \\
    Online-FW & 14.05 & \scriptsize $\pm$ 0.37 & 58.37 & \scriptsize $\pm$ 1.02 & 5.79 & \scriptsize $\pm$ 0.04 & 70.80 & \scriptsize $\pm$ 2.28 & 14.77 & \scriptsize $\pm$ 0.12 & 29.77 & \scriptsize $\pm$ 3.92 & 21.69 & \scriptsize $\pm$ 0.15 & 20.50 & \scriptsize $\pm$ 0.18 & 12.51 & \scriptsize $\pm$ 0.70 & 15.81 & \scriptsize $\pm$ 0.36 & 66.00 & \scriptsize $\pm$ 0.17 & 50.70 & \scriptsize $\pm$ 0.48 & 66.77 & \scriptsize $\pm$ 0.08 & 45.53 & \scriptsize $\pm$ 0.75 \\
    Online-FW$(\hat \eta)$ & \textit{31.58}&\textit{\scriptsize $\pm$ 0.07} & 58.87 & \scriptsize $\pm$ 0.38 & 13.97 & \scriptsize $\pm$ 0.07 & 61.33 & \scriptsize $\pm$ 0.28 & \textbf{16.09}&\textbf{\scriptsize $\pm$ 0.04} & 166.71 & \scriptsize $\pm$ 3.33 & \textit{27.15}&\textit{\scriptsize $\pm$ 0.02} & 34.83 & \scriptsize $\pm$ 0.55 & 14.57 & \scriptsize $\pm$ 0.77 & 140.62 & \scriptsize $\pm$ 0.84 & \textit{67.71}&\textit{\scriptsize $\pm$ 0.10} & 34.11 & \scriptsize $\pm$ 0.28 & \textbf{68.48}&\textbf{\scriptsize $\pm$ 0.08} & 160.12 & \scriptsize $\pm$ 0.75 \\
    \midrule
    \OMMA{} & 10.51 & \scriptsize $\pm$ 0.45 & 64.39 & \scriptsize $\pm$ 0.63 & 5.74 & \scriptsize $\pm$ 0.03 & 58.36 & \scriptsize $\pm$ 1.42 & 15.24 & \scriptsize $\pm$ 0.01 & 36.03 & \scriptsize $\pm$ 0.40 & 21.95 & \scriptsize $\pm$ 0.15 & 28.58 & \scriptsize $\pm$ 0.16 & 28.88 & \scriptsize $\pm$ 1.35 & 29.26 & \scriptsize $\pm$ 0.20 & 62.53 & \scriptsize $\pm$ 0.04 & 63.16 & \scriptsize $\pm$ 1.74 & 65.77 & \scriptsize $\pm$ 0.03 & 67.12 & \scriptsize $\pm$ 0.66 \\
    \OMMAeta{} & 31.33 & \scriptsize $\pm$ 0.08 & 66.84 & \scriptsize $\pm$ 0.36 & \textbf{15.15}&\textbf{\scriptsize $\pm$ 0.17} & 59.32 & \scriptsize $\pm$ 1.07 & \textit{16.06}&\textit{\scriptsize $\pm$ 0.03} & 36.00 & \scriptsize $\pm$ 0.32 & \textbf{27.18}&\textbf{\scriptsize $\pm$ 0.02} & 29.11 & \scriptsize $\pm$ 0.18 & \textit{29.52}&\textit{\scriptsize $\pm$ 0.24} & 28.93 & \scriptsize $\pm$ 0.54 & \textbf{69.00}&\textbf{\scriptsize $\pm$ 0.05} & 61.45 & \scriptsize $\pm$ 1.41 & \textit{68.05}&\textit{\scriptsize $\pm$ 0.03} & 65.96 & \scriptsize $\pm$ 0.39 \\
\bottomrule
\end{tabular}
}
\end{table*}

\begin{figure*}[ht]
    \centering
    \scriptsize
    
    \vspace{5pt}\datasettable{YouTube}\vspace{5pt}
    
    \includegraphics[width=0.135\textwidth]{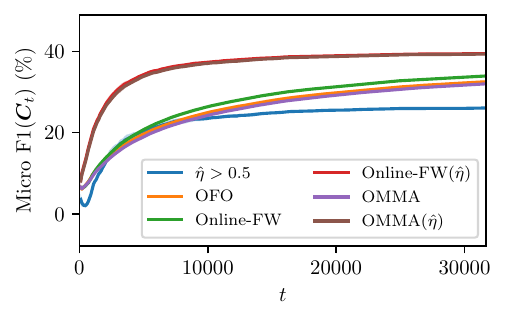}
    \includegraphics[width=0.135\textwidth]{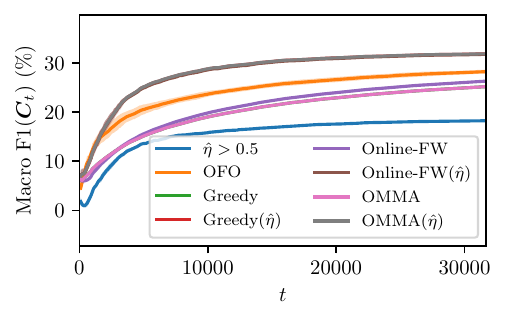}
    \includegraphics[width=0.135\textwidth]{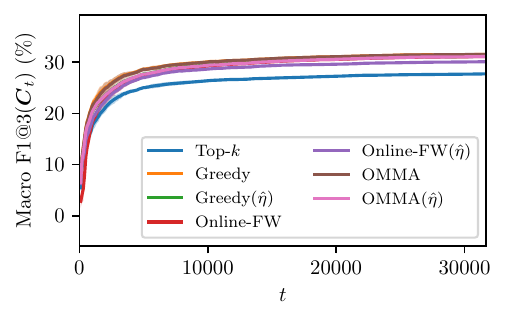}
    \includegraphics[width=0.135\textwidth]{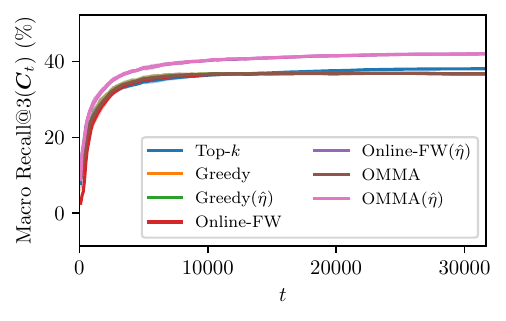}
    \includegraphics[width=0.135\textwidth]{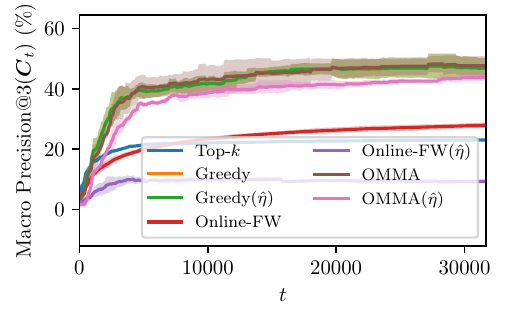}
    \includegraphics[width=0.135\textwidth]{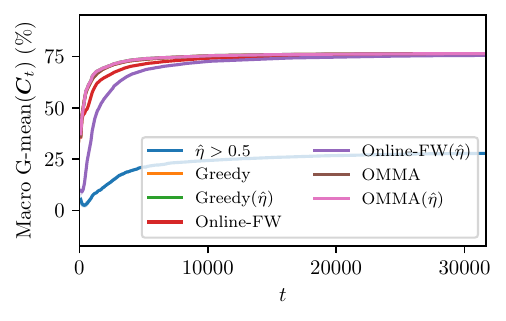}
    \includegraphics[width=0.135\textwidth]{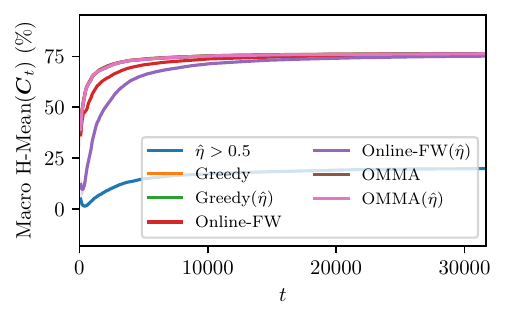}

    \vspace{5pt}\datasettable{Eurlex-LexGlue}\vspace{5pt}
    
    \includegraphics[width=0.135\textwidth]{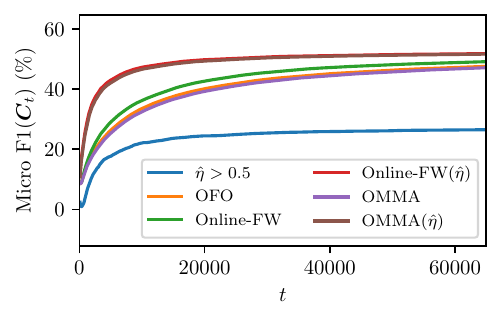}
    \includegraphics[width=0.135\textwidth]{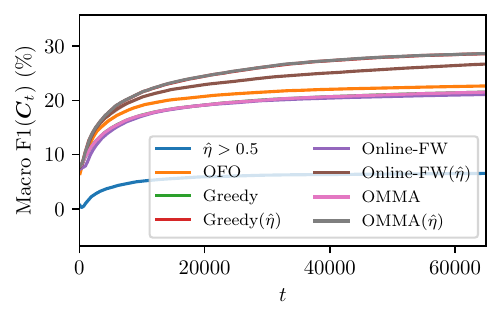}
    \includegraphics[width=0.135\textwidth]{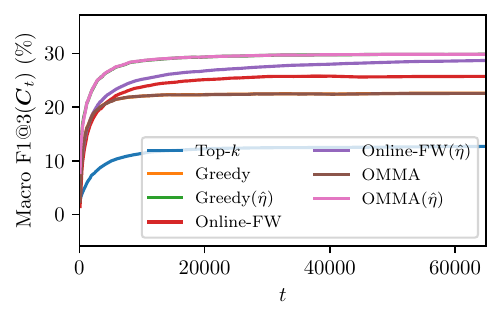}
    \includegraphics[width=0.135\textwidth]{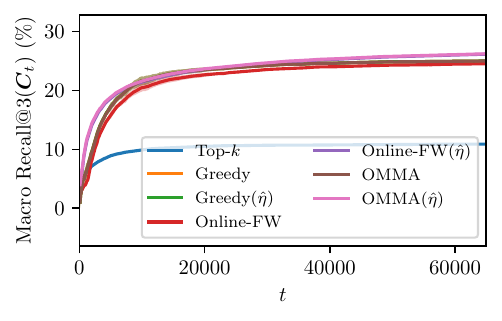}
    \includegraphics[width=0.135\textwidth]{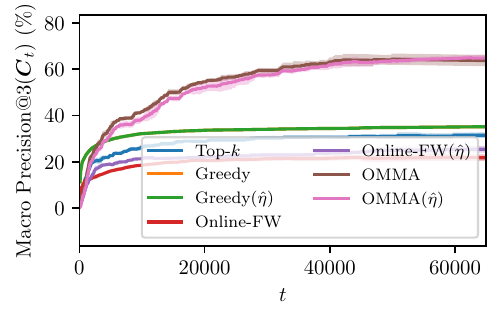}
    \includegraphics[width=0.135\textwidth]{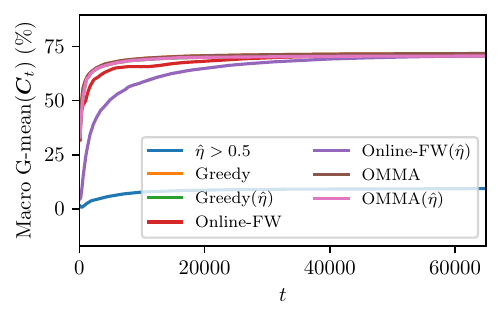}
    \includegraphics[width=0.135\textwidth]{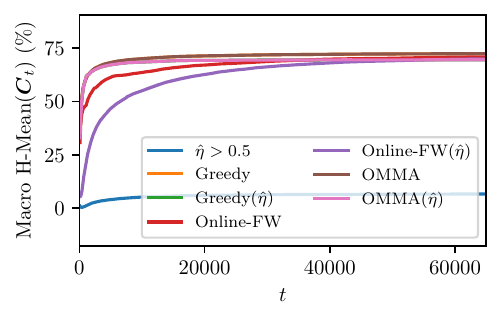}

    \vspace{5pt}\datasettable{Mediamill}\vspace{5pt}

    \includegraphics[width=0.135\textwidth]{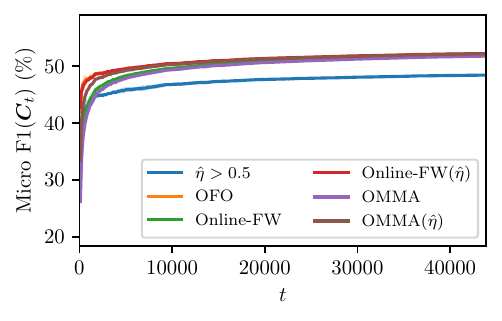}
    \includegraphics[width=0.135\textwidth]{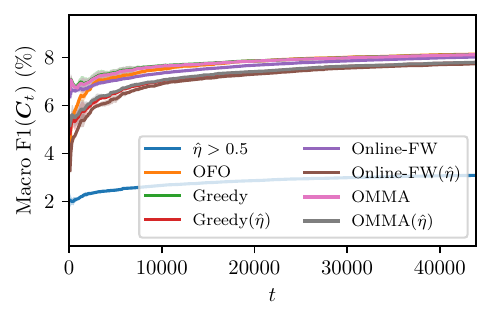}
    \includegraphics[width=0.135\textwidth]{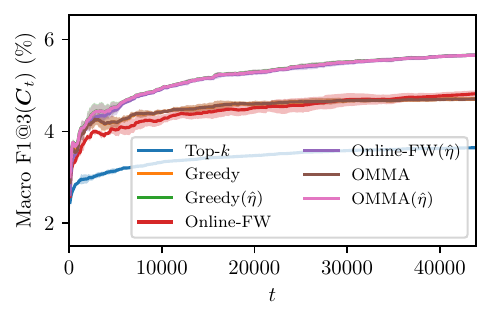}
    \includegraphics[width=0.135\textwidth]{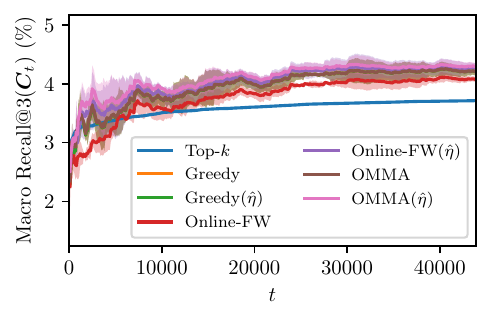}
    \includegraphics[width=0.135\textwidth]{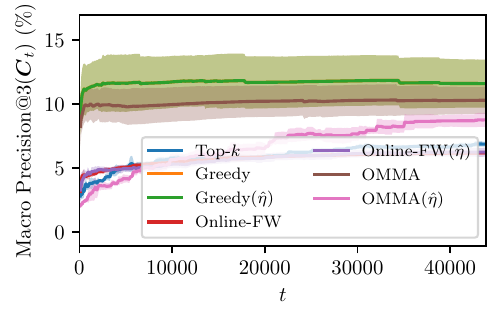}
    \includegraphics[width=0.135\textwidth]{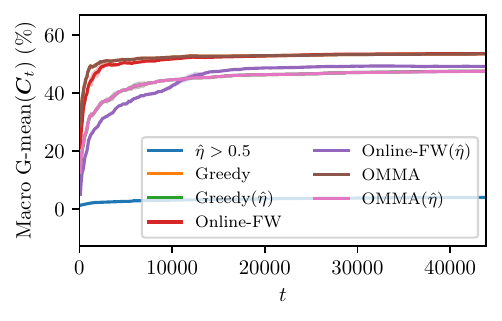}
    \includegraphics[width=0.135\textwidth]{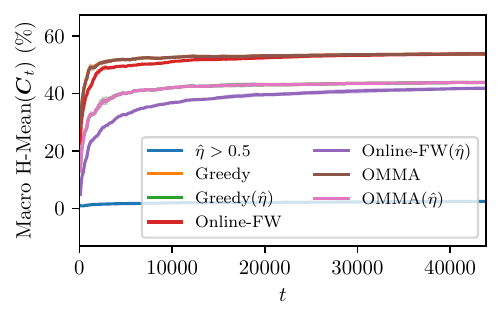}

    \vspace{5pt}\datasettable{Flickr}\vspace{5pt}
    
    \includegraphics[width=0.135\textwidth]{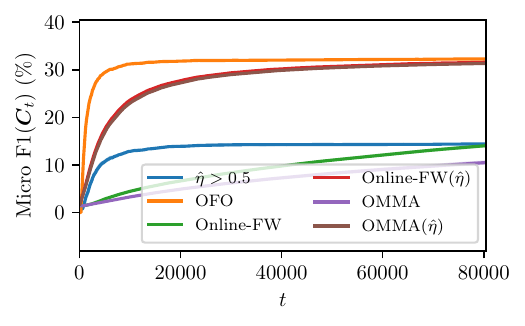}
    \includegraphics[width=0.135\textwidth]{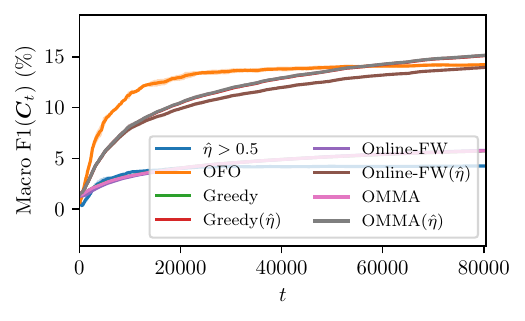}
    \includegraphics[width=0.135\textwidth]{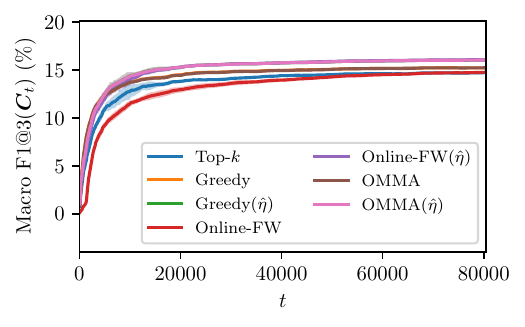}
    \includegraphics[width=0.135\textwidth]{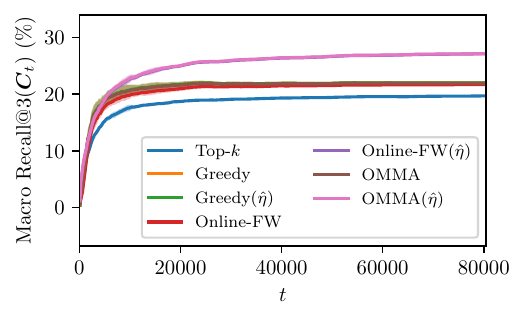}
    \includegraphics[width=0.135\textwidth]{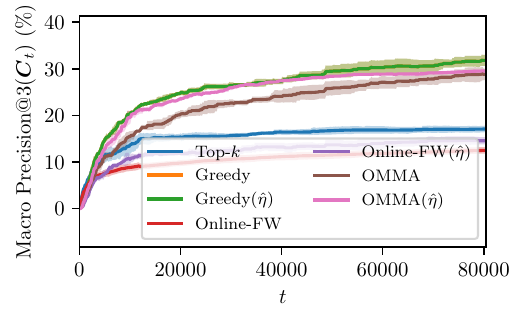}
    \includegraphics[width=0.135\textwidth]{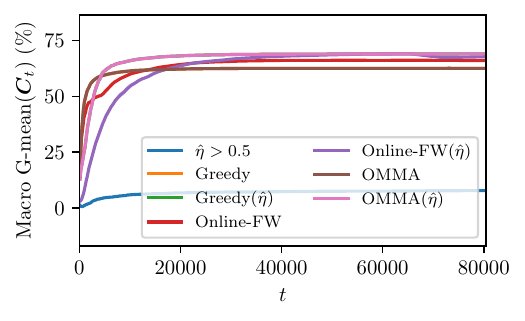}
    \includegraphics[width=0.135\textwidth]{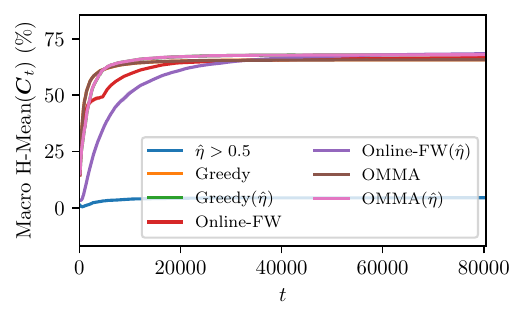}

    \caption{Running comparison of performance for the online algorithms with \emph{incrementally updated CPE}. Averaged over 5 runs, the opaque fill indicate the standard deviation at given iteration $t$.}
    \label{fig:all-online-cpe-plots}
\end{figure*}

\end{document}